\documentclass[11pt]{article}

\usepackage[
  shownumpages,               % comment to hide page numbers
  bgcolor={245,245,245},      % box bg color: RGB or xcolor name (e.g. gray, blue, etc.)
  braincolor={0,0,0},    % accent color: RGB or xcolor name (e.g. gray, blue, etc.)
  linkcolor={30,30,255},  % defaults to braincolor
  citecolor={30,30,255},  % defaults to braincolor
  urlcolor={30,30,255},            % defaults to blue
  citingstyle=authoryear,     % authoryear: [Ivanov et al., 2023] | numbers: [1] | super: ¹
  bibliostyle=plainnat,       % plainnat, abbrvnat, unsrtnat, ...
  bibfile=refs          % .bib file name (without .bib)
]{arxiv_submition/arxiv/brainlab}

\usepackage[T1]{fontenc}    % use 8-bit T1 fonts
\usepackage{hyperref}       % hyperlinks
\usepackage{url}            % simple URL typesetting
\usepackage{booktabs}       % professional-quality tables
\usepackage{amsfonts}       % blackboard math symbols
\usepackage{microtype}      % microtypography
\usepackage{xcolor}         % colors

\usepackage{amsmath}
\usepackage{amssymb}
\usepackage{mathtools}
\usepackage{amsthm}
\usepackage{wrapfig}
\usepackage{nicefrac}
\allowdisplaybreaks
\newtheorem{proposition}[theorem]{Proposition}
\newtheorem{lemma}[theorem]{Lemma}
\newtheorem{notation}{Notation}

\setbrainmeta{
    title={Aligning Distributionally Robust Optimization with Practical Deep Learning Needs},
    authors={
    Dmitrii Feoktistov\textsuperscript{1,2},
    Igor Ignashin\textsuperscript{3},
    Andrey Veprikov\textsuperscript{3},
    Nikita Borovko\textsuperscript{1},
    Alexander Bogdanov\textsuperscript{3},
    Savelii Chezhegov\textsuperscript{3},
    Aleksandr Beznosikov\textsuperscript{3,4}
  },
  affiliations={
    \textsuperscript{1}Lomonosov Moscow State University \\
    \textsuperscript{2}Yandex Research \\
    \textsuperscript{3}Moscow Center for Advanced Studies \\
    \textsuperscript{4}Innopolis University
  },
  abstract={
    Deep learning (DL) models often struggle with real-world data heterogeneity, such as class imbalance or varied data sources, as standard training methods treat all samples equally. Distributionally Robust Optimization (DRO) offers a principled approach by optimizing for a worst-case data distribution. However, a significant gap exists between DRO and current DL practices. DRO methods often lack adaptive parameter updates (like Adam), struggle with the non-convexity of neural networks, and are difficult to integrate with group-based weighting in standard mini-batch training pipelines. This paper aims to bridge this gap by introducing ALSO -- Adaptive Loss Scaling Optimizer -- a novel optimizer that integrates an adaptive, Adam-like update for the model parameters with an efficient, principled mechanism for learning worst-case data weights. Crucially, it supports stochastic updates for both model parameters and data weights, making it fully compatible with group-based weighting and standard Deep Learning training pipelines. We prove the convergence of our proposed algorithm for non-convex objectives, which is the typical case for DL models. Empirical evaluation across diverse Deep Learning tasks characterized by different types of data heterogeneity demonstrates that ALSO outperforms both traditional DL approaches and existing DRO methods.
  },
}

\usepackage{amsfonts}
\usepackage{amsmath}
\usepackage{amssymb}
\usepackage{dsfont}

\usepackage{xcolor}
\usepackage{colortbl}
\usepackage{color}
\usepackage{graphicx}
\usepackage{multirow}
\usepackage{pifont}

\usepackage{mathtools}

\usepackage{verbatim}
\usepackage{xspace} %

\newcommand{\dotprod}[2]{\left\langle #1,#2 \right\rangle}

\definecolor{PineGreen}{HTML}{008B72}

  \providecommand{\R}{\mathbb{R}} %

  \DeclareMathOperator{\E}{{\mathbb E}}

  \DeclareMathOperator*{\argmin}{arg\,min}
  \DeclareMathOperator*{\argmax}{arg\,max}

  \providecommand{\cZ}{\mathcal{Z}}

  \usepackage{bm}

\newcommand{\circledOne}{\text{\ding{172}}}
\newcommand{\circledTwo}{\text{\ding{173}}}
\newcommand{\circledThree}{\text{\ding{174}}}
\newcommand{\circledFour}{\text{\ding{175}}}

\usepackage{url}

\def\<#1,#2>{\langle #1,#2\rangle}

\providecommand{\mycomment}[3]{\todo[inline, caption={},size=footnotesize,color=#1!20]{\textbf{#2: }#3}}%
\newcommand\commenter[2]%
{%
  \expandafter\newcommand\csname #1\endcsname[1]{\mycomment{#2}{#1}{##1}}
}

\commenter{Vladimir}{red} % Солодкин Владимир 
\commenter{andrew}{blue} % Веприков Андрей

\begin{document}
\begin{mainpart}

\section{Introduction}
\label{sec:intro}
    Deep Learning (DL) has long been centered around the  empirical risk minimization (ERM) problem:
    \begin{equation}
    \label{eq:emp_risk}
        \min_{\theta \in \mathbb{R}^d} \left\{\frac{1}{n} \sum_{i=1}^n f_i(\theta)  + \frac{\tau}{2}\|\theta\|_2^2 \right\},
    \end{equation}
    where $\theta$ are the parameters of the DL model, $f_i(\theta)$ is the loss function on the $i$-th element $(\mathbf{x}_i, \mathbf{y}_i) \in \mathbf{X} \times \mathbf{Y}$ of the training data, $n$ is the number of training samples and $\frac{\tau}{2} \|\theta\|^2_2$ is a regularization term used to avoid overfitting \citep{ying2019overview}. The standard ERM framework, and the optimizers designed for it like SGD and Adam \citep{kingma2014adam}, implicitly assume that all training samples are of equal importance. However, this assumption rarely holds in real-world applications, which are often characterized by significant data heterogeneity. Datasets may suffer from severe class imbalance or be composed of data from different sources with varying distributions. In these common scenarios, treating all samples equally can lead to models with suboptimal performance and poor generalization.

    A principled approach to address this challenge is Distributionally Robust Optimization (DRO) \citep{delage2010distributionally, lin2022distributionally, wiesemann2024distributionally}. Instead of minimizing loss over a fixed, uniform data distribution, DRO seeks to optimize the model performance against a ''worst-case`` distribution. While DRO is a broad area \citep{wiesemann2024distributionally}, one of the common formulations of this idea leads to the following minimax problem:
    \begin{equation}
    \label{eq:dro_objective}
        \min_{\theta \in \mathbb{R}^d}  \left\{L_{DRO}(\theta) := \max_{\pi \in \Delta_{n-1} \cap U } \sum_{i=1}^n \pi_i f_i(\theta)  + \frac{\tau}{2}\|\theta\|_2^2 \right\},
    \end{equation}
    where $U$ is an uncertainty set, i.e. constraint on $\pi$. For example, one can use KL-divergence ball to prevent significant deviations from some prior distribution $\hat{\pi}$: $U = \left\{\pi \in \Delta_{n-1}: \text{KL} \left[\mathbf{\pi} \| \hat{\mathbf{\pi}} \right] \leq r\right\} $. Although DRO has successful applications in specific DL fields \citep{lotidis2023wasserstein, kallus2022doubly, liu2022distributionally, blanchet2020semi}, we identify several challenges in applying existing methods for general DL:

    \begin{itemize}
        \item \textbf{Lack of adaptive $\theta$ update}. Most general DRO methods either use simple SGD updates \citep{carmon2022distributionally} or apply Variance Reduction (VR) techniques \citep{mehta2024drago, mehta2023stochastic, levy2020large, qi2021online}, while the most successful DL optimizers are adaptive \citep{kingma2014adam, choi2019empirical}. 
        \item \textbf{Gap Between Theory and Practice}. Despite the success of the existing DRO methods in the convex domain (e.g. logistic regression)  \citep{mehta2024drago, mehta2023stochastic, levy2020large}, neural networks are inherently non-convex, presenting additional challenges. Several attempts have been made to develop DRO methods specifically for Deep Learning, but they are either heuristic \citep{liu2021just, sagawa2019distributionally}, or pose instability and overfitting risks \citep{qi2021online}.
        \item \textbf{Challenges with Batching and Grouping.} For practical needs one often wants to assign weights to samples based on their specific properties such as class \citep{he2009learning, lin2017deep} or worker identification \citep{mohri2019agnostic}, rather than assigning unique weights to individual objects. The problem \eqref{eq:dro_objective} can deal with this requirement if one uses $i$ as group id, not object, i.e. $f_i$ is the total loss of the group (and $n$ is number of groups). However, most DRO methods assume that $f_i$ is deterministic, which is impractical for group-based weighting in cases the presence of large groups, since calculating the full $f_i$ requires a pass over the entire group. Additionally, the requirement of full $f_i$ computation complicates algorithm integration into the standard DL training pipelines with batching.
    \end{itemize}

    This paper aims to bridge this critical gap by introducing \texttt{ALSO} -- \texttt{A}daptive \texttt{L}oss \texttt{S}caling \texttt{O}ptimizer -- an optimizer designed to align DRO with the needs of practical Deep Learning. \texttt{ALSO} is designed from the ground up to be practical: it employs an adaptive, Adam-like step for the model parameters, deals with stochastic updates for both $\theta$ and $\pi$ (allowing standard batching during training with group-based $\pi$), and effectively solves the distribution-finding subproblem for the group weights. We provide a rigorous theoretical analysis, proving \texttt{ALSO}'s convergence for non-convex objectives typical in Deep Learning.

    The key contributions of this work are:
    \begin{itemize}
        \item \textbf{Deep Learning DRO optimizer.} We present \texttt{ALSO} -- a novel algorithm designed to solve the problem \eqref{eq:adv_pi_problem} in Deep Learning contexts (see Algorithm \ref{algorithm:also_optimistic}).
        \item  \textbf{Theory.} We establish a convergence of \texttt{ALSO} in the stochastic, non-convex, $L$-smooth case.
        \item \textbf{Experiments}. We experimentally demonstrate that \texttt{ALSO} outperforms classical DL approaches and DRO algorithms in a diverse set of DL tasks characterized by data heterogeneity: learning from unbalanced data, tabular DL, robust training against adversarial attacks, distributed training with gradient compression, and split learning. Our code is available at \url{https://github.com/brain-lab-research/ALSO}.
    \end{itemize}
    
\section{Background}
    
    Distributionally Robust Optimization has emerged as a powerful framework for decision-making under uncertainty \citep{delage2010distributionally, lin2022distributionally, wiesemann2024distributionally}. DRO has successful applications in separate DL fields such as Reinforcement Learning \citep{lotidis2023wasserstein, kallus2022doubly, liu2022distributionally}, Semi-Supervised Learning \citep{blanchet2020semi}, Sparse Neural Network training \citep{sapkota2023distributionally}. However, none of these methods are for general-purpose use. 
    
    Most general DRO methods use simple SGD updates \citep{carmon2022distributionally} or apply
    Variance Reduction (VR) techniques \citep{mehta2024drago, mehta2023stochastic, levy2020large}. The main goal of such methods is to reduce the complexity of the optimization process for convex functions and to have a step cost independent of data size. Despite their success in the convex domain, neural networks are inherently non-convex, presenting additional challenges. VR methods are usually ineffective in DL \citep{defazio2019ineffectiveness}, because of data augmentation and batch normalization, which disrupt finite-sum structure. However, recently proposed \texttt{MARS} optimizer \citep{yuan2024mars} achieves good performance in language modeling tasks -- the field, in which none of the techniques mentioned above are used. Additionally, \texttt{MARS} utilizes STORM \citep{cutkosky2019momentum} which is closer to \texttt{ALSO} negative momentum (see Algorithm \ref{algorithm:also_optimistic}), rather than to classical VR methods like SAGA \citep{defazio2014saga} or SVRG \citep{johnson2013accelerating} used in most DRO methods. Furthermore, SAGA based methods like state-of-the-art DRO methods \citep{mehta2024drago} require storing a table of size $n \times d$ -- a significant limitation for large Deep Learning models with millions of parameters. It is also important to note that these methods heavily depend on deterministic $f_i$, i.e. require either assigning unique weights to individual objects or the loss computation for the whole group, and usually have limited experimental validation in neural network training scenarios.
    
    From another perspective, several attempts have been made to develop DRO methods specifically for Deep Learning. For instance, in \citep{liu2021just} the authors propose a heuristic algorithm without theoretical guarantees that requires two separate training phases to produce the final model. An alternative approach is presented in \citep{sagawa2019distributionally}, where the authors propose an algorithm with convergence guarantees for the convex case and apply it to neural network training. However, this work implements a simple gradient step for $\theta$ update, while the most successful DL optimizers are adaptive \citep{kingma2014adam, choi2019empirical}. Another approach is proposed by \citep{qi2021online}, where authors address the non-convex scenario. They solve the inner maximization problem exactly, resulting in the following formulation:
    \begin{equation}
    \label{eq:exp_problem}
        \min_{\theta \in \mathbb{R}^d} \left\{ \frac{1}{n} \sum_{i=1}^n \exp\left[ \tau^{-1} f_i (\theta) \right] + \frac{\tau}{2} \| \theta \|^2_2 \right\}.
    \end{equation}
    While reformulation \eqref{eq:exp_problem} eliminates the need to store and update $\pi$, it has several important limitations. First, since the problem \eqref{eq:exp_problem} involves expressions of the form $\exp[\tau^{-1} \cdot]$, small values of $\tau$ can lead to extremely large values that may be computationally intractable. Second, modern deep neural network training methods are iterative, with initial weight vectors $\theta^{k}$ typically far from optimal. However, in \eqref{eq:exp_problem}, we immediately compute the optimal vector $\pi^*$, which can be problematic in early training stages when higher errors on some samples may simply reflect undertraining rather than inherent difficulty. Furthermore, using the exact value of $\pi^*$ may lead to overfitting to outliers in the training set, despite the regularization term in the problem \eqref{eq:adv_pi_problem}. Finally, the approach in \eqref{eq:exp_problem} fundamentally assumes that each data point has its own weight. When multiple objects share a weight, $f_i$ becomes the sum of losses for these objects. This constraint limits batching strategies (if we use a subset of objects with the same weight to compute the stochastic gradient of \eqref{eq:exp_problem}, we obtain a biased estimation of gradient), making the proposed approach less practical for Deep Learning applications.
\section{Problem Statement}
    As discussed in the introduction, it is a common requirement to assign the same weight to several samples based on their properties such as class \citep{he2009learning, lin2017deep} or worker identification \citep{mohri2019agnostic}. The straightforward idea is to retain the problem \eqref{eq:dro_objective}, but use $f_i$ as the mean loss on the objects of the group $i$. However, this objective hides the structure of the problem (i.e. $f_i$ is the sum), resulting in methods that require deterministic $f_i$ and implies that we need to compute the whole $f_i$ to make step, which aligns poorly with model training pipeline, where we use mini-batches to make a step (i.e. the whole $f_i$ is unavailable). To make this structure more precise, we use the following modified objective: 
    \begin{equation}
        \label{eq:adv_pi_problem}
        \begin{split}
            \min_{\theta \in \mathbb{R}^d} \max_{\mathbf{\pi} \in \Delta_{c-1} \cap U} \left\{ h(\theta, \pi) := \sum_{i=1}^c \pi_i \left(\frac{c}{n}\sum_{j=1}^{n_i}f_{i, j}(\theta)\right) 
            + \frac{\tau}{2}\|\theta\|^2_2
            - \lambda \text{KL} \left[\mathbf{\pi} \| \hat{\mathbf{\pi}} \right] \right\} .
        \end{split}
    \end{equation}
    In the problem \eqref{eq:adv_pi_problem} weights $\pi_i$ are assigned to each of $c$ object groups. The group $i$ contains $n_i$ samples with loss functions $f_{i,j}$, $j \in \overline{1, n_i}$. These groups can be built based on sample class, worker ID, or each sample can compose its own group ($c = n$, $n_i=1$ $\forall i$ in such scenario), making problems \eqref{eq:adv_pi_problem} and \eqref{eq:dro_objective} almost equal. To additionally restrict deviation from the starting distribution, we use a regularization technique, where $\lambda > 0$ is the regularization parameter. The KL-divergence term $\text{KL}\left[\pi \| \hat\pi \right]$ is introduced specifically to avoid degenerate solutions in which $\pi$ collapses onto a single group. Compared to the Euclidean norm, KL-divergence is a more natural choice for probability distributions on the simplex: it better respects the underlying geometry and penalizes shifts in high-probability components more strongly, thereby stabilizing the updates. As a baseline, we can set $\hat{\mathbf{\pi}} = \mathcal{U}\left(\overline{1, c}\right) \in \Delta_{c-1}$ as the uniform discrete distribution; however, sometimes it is better to define it in a different way (for such example see Subsection \ref{sec:unbalanced_data}).  It is worth highlighting that we choose constant multiplier $\frac{c}{n}$ so that if  we substitute $\pi = \hat{\mathbf{\pi}} = \mathcal{U}\left(\overline{1, c}\right)$ into \eqref{eq:adv_pi_problem}, the resulting equation is exactly the same as ERM.

\section{Optimistic Mirror-Prox}
\label{sec:optimistic_mirror_prox}

To begin with, we propose a non-adaptive algorithm for the problem \eqref{eq:adv_pi_problem} (with $c=n$ and $U=\Delta_{n-1}$, then $j=1$ for all $f_{ij}$). Although this case is not the main focus of this paper, it is an important intermediate result, as it provides the foundation for developing our main algorithm by transferring this approach to the non-convex case with added adaptivity. The development of our algorithm is motivated by the evolution of optimization methods for saddle point problems \eqref{eq:min_max}. 
\begin{equation*}
    \label{eq:min_max}
    \min_{x\in \mathcal{X}} \max_{y \in \mathcal{Y}} g(x, y)
\end{equation*}
The easiest option to obtain methods for saddle point problems is to adapt gradient schemes from minimization tasks.  In this way, it is possible to obtain, for example, the Stochastic Gradient Descent Ascent (SGDA) method. However, this scheme is inadequate both from the theoretical perspective and in terms of its application to the simplest problems of the form \eqref{eq:min_max} \cite{goodfellow2016nips, beznosikov2023smooth}. Therefore, it is suggested to use more advanced algorithms such as Extragradient \cite{korpelevich1976extragradient}. For our non-Euclidean geometry, it makes sense to consider an appropriate modification of Extragradient -- Mirror-Prox \cite{juditsky2011solving}. However, both Extragradient and Mirror-Prox require two oracle calls per iteration. To address this, so-called Optimistic version of these algorithms can be applied \cite{popov1980modification}. It requires only one oracle call per iteration. 
It turns out that the Extragradient and Optimistic updates outperform SGDA not only in the theory, but also in deep learning, particularly in training GANs \cite{daskalakis2017training, gidel2018variational, mertikopoulos2018optimistic, chavdarova2019reducing, liang2019interaction, peng2020training}. 

We note that this algorithm is mostly auxiliary, even though it represents particular interest (see Discussion after Theorem \ref{theorem:optimistic_mirror_prox}).
\begin{algorithm}{Optimistic Mirror-Prox for \eqref{eq:adv_pi_problem} with $c=n$, $U=\Delta_{n-1}$}
   \label{algorithm:optimistic_mirror_prox}
\begin{algorithmic}[1]
    \State {\bf Parameters:} stepsize $\gamma$, momentum $\alpha$, number of iterations $N$.
    \State {\bf Initialization:} choose  $x^0 \in \mathbb{R}^d, \pi^0 \in \Delta_{n}$.
    \For{$k = 0, 1, 2, \dots, N$}
        \State $g^{k+1} = \sum_{i=1}^{n}\pi_{i}\nabla_\theta f_{i,1}(\theta^{k})$
        \State
        \text{$p^{k+1} = \sum_{i=1}^n-e_i\cdot\textstyle{f_{i,1}}(\theta^{k})$, where $e_i$ is vector with $1$ in $i$-th position and zeros in others
        }
        \State $\theta^{k+1} = \theta^{k} - \gamma_\theta \cdot \left((1 + \alpha)g^{k+1} - \alpha g^{k} + \tau \theta^{k}\right)$
        \State $\pi^{k+1} = \text{SM} \left[\log \pi^{k} -\gamma\left((1 + \alpha)p^{k+1} - \alpha p^{k} + \tau \log (\pi^{k} / \hat{\pi}))\right) \right]$
    \EndFor
\end{algorithmic}
\end{algorithm}

\textbf{Discussion.} Algorithm \ref{algorithm:optimistic_mirror_prox} follows a structured pipeline. It employs steps that closely resemble gradient descent for updating the parameter $\theta$, while utilizing Softmax (denoted as $\text{SM}$ in Algorithm~\ref{algorithm:optimistic_mirror_prox}) to update $\pi$, which corresponds to a Mirror Ascent step. To achieve theoretical guarantees, Optimistic Mirror-Prox introduces an additional hyperparameter $\alpha$ and implements a negative momentum technique, preventing steps from becoming too sharp.

Now, let us outline key Assumptions \eqref{as:lipgrad} and \eqref{as:convex} for the functions $f_i(\cdot)$ in the problem \eqref{eq:adv_pi_problem} required for convergence analysis.
\begin{assumption}
    \label{as:lipgrad}
    For all $(i,j)$ the functions $f_{i,j}$ from \eqref{eq:adv_pi_problem} are $K_{i,j}$-Lipschitz continuous and $L_{i,j}$-smooth on $\Theta$ with respect to the Euclidean norm $\| \cdot \|_2$ , i.e., for any $\theta^1, \theta^2 \in \Theta$ the following inequality holds:
    \begin{equation*}
    \begin{split}
        &\|\nabla f_{i,j}(\theta^1) - \nabla f_{i,j}(\theta^2)\| \leq L_{i,j} \|\theta^1 - \theta^2\|_2 ~\text{ and }~
        |f_{i,j}(\theta^1) - f_{i,j}(\theta^2)|_2 \leq K_{i,j} \|\theta^1 - \theta^2\|_2.
    \end{split}
    \end{equation*}
\end{assumption}
\begin{assumption}
    \label{as:convex}
    For all $(i,j)$ functions $f_{i,j}$ from \eqref{eq:adv_pi_problem} are convex on $\Theta$, i.e., for any $\theta^1, \theta^2 \in \Theta$ the following inequality holds 
    \begin{equation*}
        \langle \nabla f_{i,j}(\theta^1) - \nabla f_{i,j}(\theta^2), \theta^1 - \theta^2 \rangle \geq 0.
    \end{equation*}
\end{assumption}
Assumptions \ref{as:lipgrad} and \ref{as:convex} are classical in the analysis of the problem \eqref{eq:emp_risk} \citep{polak1971computational, amari1993backpropagation, nesterov2013introductory, veprikov2024new, solodkin2024methods}. Note that Assumption \ref{as:convex} is not required for our main algorithm.

\begin{theorem}
\label{theorem:optimistic_mirror_prox}
    Let Assumptions \ref{as:lipgrad} and \ref{as:convex} be satisfied and let
    $
        \Phi_k := \| \theta^k - \theta^* \|^2_2 + \text{KL} \left[\mathbf{\pi}^* \| \mathbf{\pi}^k\right].
    $
    Let the problem \eqref{eq:adv_pi_problem} be solved by Algorithm \ref{algorithm:optimistic_mirror_prox}. Assume that the stepsize $\gamma$ is chosen such that $0 <\gamma \leq 1/(2L_F)$, where 
    \begin{equation*}
        L_F^2 = \mathcal{O} \left[ \max_{(i,j)}\{L_{i,j}^2\} + \max_{(i,j)}\left\{ K_{i,j}^2 \right\}  \right],
    \end{equation*}
    and momentum $\alpha$ is chosen such that $\alpha = (1 + \gamma \tau)^{-1}$. Then, for all $k \geq 1$ it holds that 
    \begin{equation*}
        \Phi_k
        =
        \mathcal{O} \left[ \left(1 - \frac{\gamma \tau}{2}\right)^{k} \Phi_0  \right],
    \end{equation*}
    where $(\theta^*, \pi^*)$ is the solution of the problem \eqref{eq:vi_problem}. In other words, if one takes $\gamma = 1 / (2 L_F)$, then to achieve $\varepsilon$-accuracy (in terms of $\Phi_N \leq \varepsilon$) one would need at most 
    \begin{equation*}
        \mathcal{O} \left[ \frac{L_F}{\tau} \cdot \log\left( \frac{\Phi_0}{\varepsilon} \right) \right] ~\text{ iterations of Algorithm \ref{algorithm:optimistic_mirror_prox}.}
    \end{equation*}
\end{theorem}

Full proof of Theorem \ref{theorem:optimistic_mirror_prox} is provided in Appendix \ref{appendix:optimistic_mirror_prox}.

\textbf{Discussion.} Since Algorithm \ref{algorithm:optimistic_mirror_prox} is not stochastic, its resulting oracle complexity is $\mathcal{O}[n \cdot (L_F / \tau) \cdot \log(\Phi_0 / \varepsilon)]$. We compare our findings with the recently proposed variance reduced algorithm \texttt{DRAGO} \citep{mehta2024drago}. According to Table 1 in \citep{mehta2024drago}, their method achieves state-of-the-art complexity proportional to $\mathcal{O}[n\sqrt{n} \cdot \log(1 / \varepsilon)]$. However, this complexity is worse than our result. While \texttt{DRAGO} can handle gradients of several summands of the problem \eqref{eq:adv_pi_problem}, our Optimistic Mirror-Prox achieves lower computational cost through its reduced oracle complexity in terms of $n$ and techniques such as gradient accumulation. Another algorithm, \texttt{LSVRG} \citep{mehta2023stochastic}, achieves complexity similar to ours but imposes restrictive assumptions on the regularization multiplier of the $KL[\cdot|\cdot]$ divergence. Furthermore, from a memory perspective, both \texttt{DRAGO} and \texttt{LSVRG} require additional space to store parameters for Variance Reduction, whereas non-VR methods like ours only need space for gradients.

\section{ALSO -- Adaptive Loss Scaling Optimizer}
\label{sec:also}
Despite the linear convergence of Optimistic Mirror-Prox it does not satisfy our requirements: an adaptive step for $\theta$, the ability to handle a stochastic oracle, and convergence in the non-convex case. 
Building upon the important foundation of Optimistic Mirror-Prox, we now introduce \texttt{ALSO} (Algorithm \ref{algorithm:also_optimistic}) -- the \texttt{A}daptive \texttt{L}oss \texttt{S}caling \texttt{O}ptimizer -- which effectively addresses our requirements.
As we mentioned above, Optimistic Mirror-Prox utilizes GD-like step over $\theta$, to enhance adaptivity, we replace this GD step with Adam \citep{kingma2014adam}, resulting in our proposed \texttt{ALSO} algorithm (Algorithm \ref{algorithm:also_optimistic}) for solving general form of the problem \eqref{eq:adv_pi_problem}. Note that we leave the same step over $\pi$ as before.
In practice, nearly all works that employ Euclidean Optimistic method for deep learning tasks do not use its theoretical version, but rather an adaptive variant (typically with Adam-style stepsizes) \cite{daskalakis2017training, gidel2018variational, mertikopoulos2018optimistic, chavdarova2019reducing, liang2019interaction, peng2020training}. This substitution is often justified as a standard procedure in DL. However, we question this approach, as establishing theoretical guarantees for adaptive methods is a nontrivial and technically demanding task (see Appendix \ref{appendix:also-theory}). In this work, we do not follow this simplified route; instead, we provide a rigorous analysis of the adaptive method (see Theorem \ref{also:convergence-new}).

\setlength{\textfloatsep}{5pt}
\begin{algorithm}{\texttt{ALSO}}
       \label{algorithm:also_optimistic}
    \begin{algorithmic}[1]
        \State {\bf Parameters:} $\gamma_\theta, \gamma_{\pi}$ -- stepsize for $\theta$ and $\pi$; $\beta_1, \beta_2$, $\varepsilon$ for Adam; momentum $\alpha$; $\lambda, \tau$ -- regularization parameters for $\pi$ and $\theta$; number of iterations $T$; $\hat{\pi}$ -- regularization distribution.
        
        \State {\bf Initialization:} $m^0 = g^{0} = p^{0} = \mathbf{0}$, $v_0 = 0$, $\pi^0 = \hat{\pi}$, $\hat{\gamma}_\pi = \gamma_\pi / (1+\gamma_\pi\lambda)$
        
        \For{$k = 0, 1, 2, \dots, T$}
            \State \label{alg:also_sampling_line} Sample $B$ objects: $\{ (c^k_1, i_1^k), ..., (c^k_B, i_B^k) \}$ -- pairs (group, index)
            \State \label{alg:also_theta_grad} $g^{k+1} = \frac{c}{B}\sum_{j=1}^{B}\pi_{c^k_j}\nabla_\theta f_{c^k_j,i_j^k}(\theta^{k})$
            \State \label{alg:also_neg_momentum_theta}$\hat{g}^{k+1} = (1 + \alpha)g^{k+1} - \alpha g^{k} + \tau \theta^{k}$
            \State  \label{alg:also_pi_grad} $p^{k+1} = \frac{c}{B}\sum_{j=1}^{B} e_{c_j^k}\cdot f_{c^k_j,i_j^k}(\theta^{k})$, where $e_i$ is vector with $1$ in $i$-th position and zeros in others
            \State \label{alg:also_neg_momentum_pi} $\hat{p}^{k+1} = (1 + \alpha)p^{k+1} - \alpha p^{k}$
            \State \label{alg:also_theta_update} $\theta^{k+1} = \theta^{k} - \gamma_\theta \cdot \text{Adam}(\hat{g}^{k+1}, \beta_1, \beta_2, \varepsilon)$
            \State \label{alg:also_pi_update} Option I: $\pi^{k+1} = \text{SM} \left[\log \pi^{k} - \hat{\gamma}_\pi(\hat{p}^{k+1} + \lambda \log (\pi^{k} / \hat{\pi})) \right]$
            \State \label{alg:also_pi_update_2} Option II: $\pi^{k+1} = \argmin_{\pi \in U \cap \Delta_{c-1}} \left\{\hat{\gamma_\pi}\langle\hat{p}^{k+1} + \lambda\log \frac{\pi}{\hat{\pi}},\pi\rangle+\text{KL}[\pi \|\pi^k]\right\}$
        \EndFor
    \end{algorithmic}
    \end{algorithm}

    \textbf{Discussion.} In contrast to Algorithm \ref{algorithm:optimistic_mirror_prox}, where full gradients are used, in Lines \ref{alg:also_theta_grad}, \ref{alg:also_pi_grad} of Algorithm \ref{algorithm:also_optimistic}, we use gradients obtained by a straightforward sampling strategy: we unite all objects into a single group and then sample from it. This approach is mathematically equivalent to combining the two sums in the equation \eqref{eq:adv_pi_problem} and selecting $B$ terms from the unified sum. This method allows for seamless integration of \texttt{ALSO} into standard Deep Learning training pipelines. Nevertheless, alternative sampling strategies are viable. For instance, one might first sample groups and subsequently sample objects within each selected group, if it is suitable for specific applications (see Appendix \ref{appendix:sampling}). In Lines \ref{alg:also_neg_momentum_theta}, \ref{alg:also_neg_momentum_pi} we utilize negative momentum -- a common technique used to prevent too sharp steps and obtain convergence. While introduction of this term is inspired by Optimistic Mirror-Prox, the similar term is used in \texttt{MARS} \citep{yuan2024mars} to reduce the variance. This observation further confirms this design choice. We additionally ablate it in Appendix \ref{appendix:ablation}. In Line \ref{alg:also_theta_update} we utilize Adam step to update $\theta$. It is worth noting that Adam itself can be seen as a particular case of \texttt{ALSO}: if we set the hyperparameters so that $\pi$ remains constant and equal to $1/c$, the procedure reduces to Adam. In Lines \ref{alg:also_pi_update}, \ref{alg:also_pi_update_2} we present two options for the $\pi$ update. Option~I employs $U = \Delta_{c-1}$ for simplicity and is used in practical implementation. Option~II provides a more general formulation and is particularly valuable for theoretical analysis. The step over $\pi$ has time complexity $O(c)$, which in theory can be costly. However, in many applications $c$ is relatively small (see Section \ref{sec:experiments}). Furthermore, for most DL models, gradient computations consume the majority of training time \citep{jiang2021optimizer}. Based on these observations, we determined that updating $\pi$ using simple arithmetic operations with $O(c)$ complexity satisfies practical requirements. We validate this assessment experimentally in Appendix~\ref{appendix:time_analysis}.

\subsection{Convergence of ALSO}
    We now present assumptions for the convergence analysis.    
    
    \begin{assumption}
        \label{as:uncertainty_set}
        The admissible domain 
        \(\mathcal{D}_\pi := \Delta_{c-1} \cap U\) 
        is nonempty, closed, and convex. Moreover, 
        regularizer \(\hat{\pi} \in \operatorname{Int}(\mathcal{D}_\pi) \)
    \end{assumption}

    \begin{assumption}
        \label{as:stoch_grad} 
        At each iteration of Algorithm~\ref{algorithm:also_optimistic} we have access to oracles $g = g(\theta,\pi)$ and $p = p(\theta,\pi)$, which provide unbiased estimates of the gradients for the problem \eqref{eq:adv_pi_problem} using batch size $B$. Moreover,
        \[
            \mathbb{E}\big\| g(\theta, \pi) - \nabla_\theta h(\theta , \pi)  \big\|^2_2 \leq \frac{\sigma^2}{B},
            \qquad 
            \mathbb{E}\big\| p(\theta, \pi) - \nabla_\pi h(\theta , \pi)  \big\|^2_2 \leq \frac{\sigma^2}{B}.
        \]
    \end{assumption}
    For example, if one uses straightforward sampling (Lines \eqref{alg:also_theta_grad}, \eqref{alg:also_pi_grad}) without any other source of stochasticity (e.g., no dropout, augmentations, etc.), then $\sigma^2 = \mathcal{O}\left(K^2\cdot\max\left\{c^2, \frac{c^3\sum_{i=1}^{c}n_i^2}{n^2}\right\}\right)$, where $K = \max_{(i,j)} K_{i,j}$. See Appendix \ref{appendix:sampling} for derivation.

    \begin{definition}[Stationary point, cf.~\citep{lin2019descentascent}]\label{def:stationary}
    A point $\theta$ is called an $\varepsilon$-stationary point ($\varepsilon \ge 0$) of a differentiable function $\Phi$ if $\|\nabla \Phi(\theta)\| \le \varepsilon$. If $\varepsilon = 0$, then $\theta$ is a stationary point.
    \end{definition}

    In our setting, the primal objective is $\Phi(\theta) := \max_{\pi \in \mathcal{D}_\pi} h(\theta,\pi), \label{eq:primal_obj}$
    which is differentiable since $h(\theta,\pi)$ is smooth with respect to $\theta$ and the maximization is over a compact convex set.  
    Therefore, following \citep{lin2019descentascent}, it is sufficient to measure convergence of 
    Algorithm~\ref{algorithm:also_optimistic} by the gradient norm $\|\nabla \Phi(\theta)\|$, as  
    small gradients certify approximate stationarity of the original min--max problem~\eqref{eq:adv_pi_problem}. Moreover, due to stochasticity in the updates, it is natural to adopt the criterion  $\mathbb{E}\|\nabla \Phi(\theta)\|^2 \le \varepsilon^2$.

    Now we are ready to present the following main theorem, which establishes the complexity bounds of Algorithm~\ref{algorithm:also_optimistic}.

     \begin{theorem} \label{also:convergence-new}
        Under Assumptions~\ref{as:uncertainty_set},~\ref{as:lipgrad},~\ref{as:stoch_grad}, the required number of iterations to achieve $\varepsilon$-stationarity~\ref{def:stationary} ($\mathbb{E}\|\nabla \Phi(\theta)\|^2 \le \varepsilon^2$) for the problem \eqref{eq:adv_pi_problem} by \texttt{ALSO} (Algorithm~\ref{algorithm:also_optimistic}) with $\gamma_\theta = \mathcal{O}(\frac{L^4}{\lambda^4}) $, $\gamma_\pi = \frac{\lambda}{8L^2}$, $\beta_1 = \mathcal{O}(\frac{\varepsilon\lambda^2}{L^2})$, $\beta_2 = 1 - \mathcal{O}(\varepsilon^2)$, $B = \mathcal{O}(\frac{\sigma^2}{\varepsilon^2})$ is 
        \[
            T = \mathcal{O}\left(\frac{L^2\cdot(K+\sigma)}{\lambda^2\varepsilon^2} \cdot \max\!\left\{ 
                \Delta_\Phi \frac{L^2}{\lambda^2},\;
                D_0
            \right\}\right),
        \]
        where $\Delta_\Phi = \Phi(\theta^0) - \min_{\theta \in \mathbb{R}^d}\Phi(\theta)$, $D_0 = KL(\pi^*(\theta^0) \| \pi^0), \pi^*(\theta)=\argmax_{\pi\in\mathcal{D}_\pi} h(\theta, \pi)$ and $L^2=\mathcal{O}\left(\left(\frac{c}{n}\max_i \sum_{j=1}^{n_i}L_{i,j} + \tau + \frac{c}{n}\max_i \sum_{j=1}^{n_i}K_{i,j} \right)^2 +\lambda^2\right)$, $K=\frac{c}{n}\max_i \sum_{j=1}^{n_i}K_{i,j}$.
    \end{theorem}
        
    Appendix~\ref{appendix:also-theory} provides a detailed derivation and discusses parameter selection.

    \textbf{Discussion.} 
    This convergence result matches the guarantees of the standard SGDA method~\citep{lin2019descentascent} in terms of both iteration complexity $O(\frac{1}{\varepsilon^2})$ and batch size in the stochastic regime $O(\frac{1}{\varepsilon^2})$, resulting in total computational complexity $O(\frac{1}{\varepsilon^4})$. Furthermore, our rate matches lower-bound from \citep{li2021complexity}. In contrast to~\citep{lin2019descentascent}, we incorporate Adam-type updates on the $\theta$-side and provide a dedicated analysis of the Adam estimator to obtain such bounds. 
    Moreover, unlike SGDA, \texttt{ALSO} leverages a non-Euclidean geometry, instead of Euclidean projection used in \citep{lin2019descentascent}.

\section{Experiments}
\label{sec:experiments}
    We evaluate \texttt{ALSO} in several setups characterized by significant data heterogeneity. Specifically:
    \begin{itemize}
        \item  \textbf{Learning from Unbalanced Data} (Section \ref{sec:unbalanced_data}). We evaluate \texttt{ALSO} in an extremely class-imbalanced setup. Here, we assign weights to individual objects (i.e., no grouping is used).
        \item \textbf{Tabular DL} (Section \ref{sec:tabdl}). Tabular data is central to many real-world industrial problems and is often characterized by complex data heterogeneity, such as heavy-tailed and non-symmetric targets, extreme distributional shifts, and class imbalance (see Table \ref{table:tabdl-datasets-properties} for details). In this setup, we assign weights to individual objects (i.e., no grouping is used).
        \item \textbf{Robust Training to Adversarial Attacks} (Section \ref{sec:robust_adversarial}). The considered attacks vary in strength, which makes some easier to defend against than others. In this task, we assign weights to the attacks rather than to individual objects (i.e., grouping is used).
        \item \textbf{Distributed Training} (Section \ref{sec:distributed_learning}). Data heterogeneity is a well-known problem in distributed training, making it a natural setting to evaluate \texttt{ALSO}. Here, we assign weights to the workers instead of the individual objects (i.e., grouping is used).
        \item  \textbf{Split Learning} (Section \ref{sec:split_learning}). The heterogeneity arises from split learning formulation: model with shared encoder trains on different tasks. In this experiment, we assign weights to each class, not to individual objects (i.e., grouping is used).
    \end{itemize}

    We compare \texttt{ALSO} with standard DL baselines, including \texttt{AdamW} and \texttt{Static Weights} (see Appendix \ref{appendix:unbalanced_cifar}), and several DRO methods that tackle problems similar to ours. We use both classical DRO methods like \texttt{Spectral Risk}  \citep{mehta2023stochastic}, and state-of-the-art methods such as \texttt{DRAGO} \citep{mehta2024drago} (noted for fast convergence), \texttt{FastDRO} \citep{levy2020large} (a scalable method), \texttt{RECOVER} \citep{qi2021online} (a non-convex method). All the methods are discussed in Section \ref{sec:intro}. Baselines were implemented using official code when suitable, or based on the paper otherwise. Details on hyperparameter tuning can be found in Appendix \ref{appendix:experimental_summary}.  In short, all methods were tuned for the same number of iterations using either the Optuna package \citep{akiba2019optuna} or a grid search. To reduce the hyperparameter search space, we fix $\alpha=1$. This decision is supported by theory (see \citep{popov1980modification}) and prior empirical studies, which have shown that setting $\alpha$ near $1$ is an effective choice \citep{mertikopoulos2018optimistic}.
    %\citep{mertikopoulos2018optimistic, daskalakis2017training, antonakopoulos2019adaptive}.

\subsection{Learning from Unbalanced Data}
\label{sec:unbalanced_data}

    The purpose of this experiment is to demonstrate that \texttt{ALSO} can perform effectively in scenarios where the training dataset suffers from class imbalance. We consider a classification task on the CIFAR-10 dataset \citep{krizhevsky2009learning} using the ResNet-18 model \citep{he2016deep}. To simulate class imbalance, the ten original classes in the dataset are grouped into two based on the parity of its class. Subsequently, a proportion of samples from the second class is removed from both the training and validation datasets. Importantly, the test dataset, used to compute performance metrics, remains balanced. To quantify the class imbalance, we introduce the unbalanced coefficient (uc), which specifies the ratio of samples between the first and second classes as: $\nicefrac{\# \text{ 1 class}}{\# \text{ 2 class}}$ = uc, where $\#$ is the number of samples in the corresponding classes. For this experiment, we consider the values uc $\in \{ 1, 2, 5, 10, 20, 30, 40, 50 \}$. The results of the experiment are presented in Figure \ref{fig:unbalanced_cifar_f1}. 
    Analyzing the metric values, we observe that the proposed method \texttt{ALSO} outperforms all the compared baselines. The performance difference is particularly noticeable for large values of the unbalanced coefficient ($\geq 30$), where one class significantly outweighs the other. 
    
\begin{figure}[h!]

    \centering
    \begin{minipage}{0.43\linewidth}
        \includegraphics[width=0.95\linewidth]{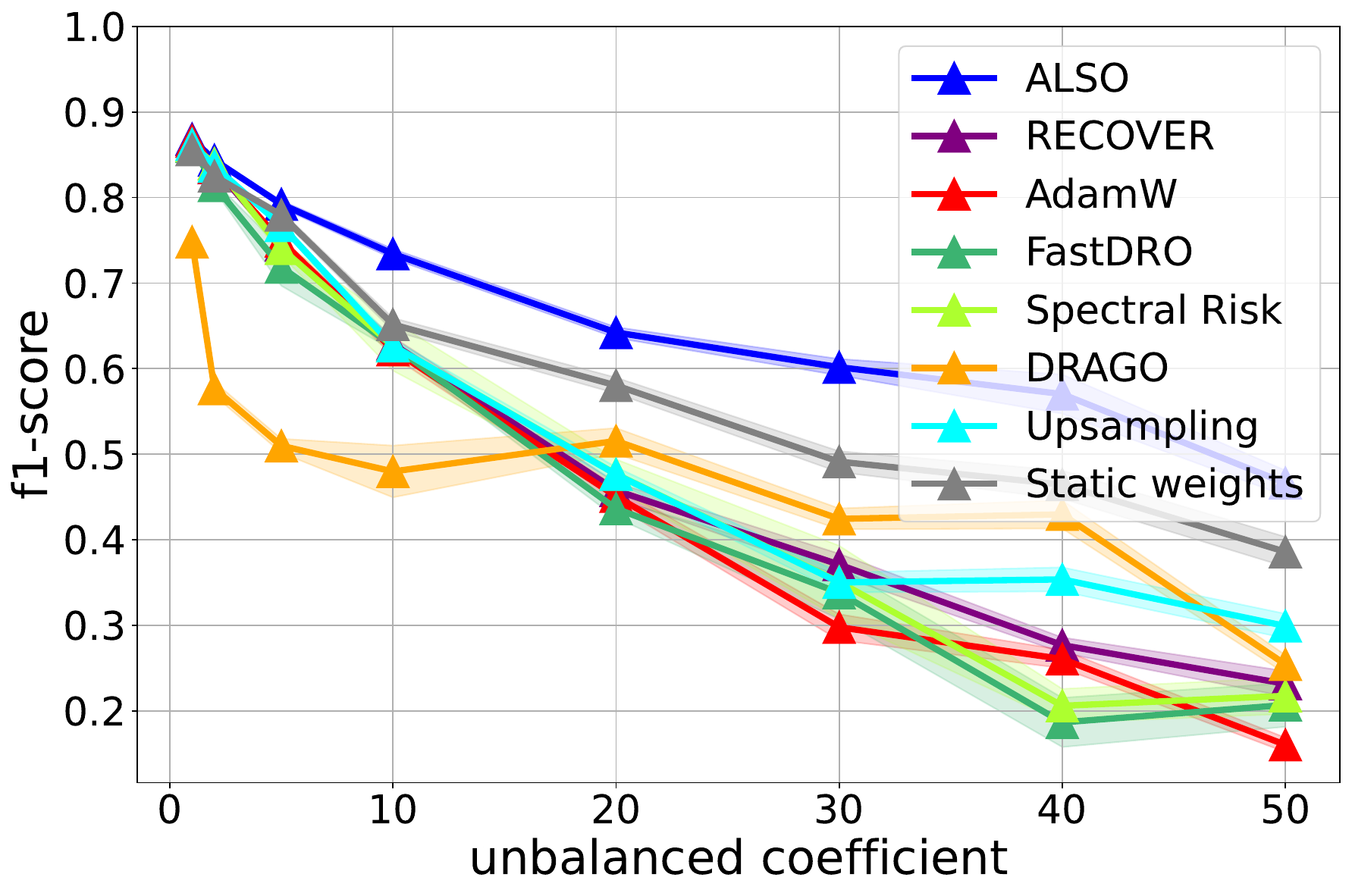}
    \end{minipage}
    \caption{Performance comparison of optimization techniques designed for training in the presence of class imbalance
    The final f1-score was averaged over 20 runs, see Appendix \ref{appendix:unbalanced_cifar} for details.}
    \label{fig:unbalanced_cifar_f1}
\end{figure}
\subsection{Tabular Deep Learning}
\label{sec:tabdl}
     We evaluate the training procedure over $14$ tabular datasets from \citep{gorishniy2024tabr, rubachev2024tabred} and $15$ runs on them. Notably, the selected datasets possess characteristics particularly relevant for DRO methods: significant distribution shift between train-test splits, class imbalance, or heavy-tailed target distributions in regression tasks. As a model, we choose MLP-PLR \citep{gorishniy2022embeddings} as it is a strong baseline in the tabular DL field. Detailed dataset characteristics, hyperparameter tuning procedures, and training specifications can be found in Appendix \ref{appendix:tabdl}. The results of the algorithms comparison are presented in Table \ref{table:tabdl-mlp-plr}. \texttt{ALSO} demonstrates the best performance on the most datasets and can be considered as an alternative to both conventional DL methods and specialized DRO methods.

    \begin{table*}[h!]
    \centering
    \caption{ Performance comparison of \texttt{ALSO}, \texttt{AdamW} with uniform weights and \textit{static weights} and Distributionally Robust Optimization methods -- \texttt{DRAGO}, \texttt{Spectral Risk}, \texttt{FastDRO}, \texttt{RECOVER} on tabular Deep Learning datasets. Bold entries represent the best method on each dataset according to mean, underlined entries represent methods, which performance is best with standard deviations over 15 runs. Metric is written near dataset name, $\uparrow$ means that higher values indicate better performance, $\downarrow$ means otherwise. }
    \label{table:tabdl-mlp-plr}
    \resizebox{\textwidth}{!}{
    \begin{tabular}{lcccccccc}
    \toprule
    \textbf{Dataset} & \texttt{ALSO} & \texttt{AdamW}  & \texttt{DRAGO} & \texttt{Spectral Risk} & \texttt{FastDRO} & 
    \texttt{RECOVER} & \texttt{Static Weights} & \\
    
     \midrule Weather (RMSE $\downarrow$) & $\mathbf{1.4928 \pm 0.0042}$ & ${1.5208} \pm {0.0037}$ & ${1.5803 \pm 0.0103}$ & ${1.5189 \pm 0.0047}$ & ${1.5184 \pm 0.0041}$ & ${1.5547} \pm {0.0034}$ & ${1.5161 \pm 0.0046}$ \\
    
     \midrule Ecom Offers (ROC-AUC $\uparrow$) & $\underline{0.5976 \pm 0.0020}$ & ${0.5810} \pm {0.0039}$ & $\mathbf{0.5983 \pm 0.0019}$ & ${0.5796 \pm 0.0034}$ & ${0.5900 \pm 0.0126}$ & ${0.5859} \pm {0.0031}$ & ${0.5803 \pm 0.0033}$ \\
    
     \midrule Cooking Time (RMSE $\downarrow$) & $\mathbf{0.4806 \pm 0.0003}$ & ${0.4813} \pm {0.0003}$ & ${0.4843 \pm 0.0008}$ & ${0.4810 \pm 0.0004}$ & $\underline{0.4809 \pm 0.0004}$ & ${0.4813} \pm {0.0006}$  & ${0.4818 \pm 0.0006}$ \\
    
     \midrule Maps Routing (RMSE $\downarrow$) & $\mathbf{0.1612 \pm 0.0001}$ & ${0.1618} \pm {0.0002}$  & ${0.1651 \pm 0.0005}$ & ${0.1619 \pm 0.0003}$ & ${0.1620 \pm 0.0003}$ & ${0.1621} \pm {0.0003}$ & ${0.1617 \pm 0.0002}$  \\
    
     \midrule Homesite Insurance (ROC-AUC $\uparrow$) & $\mathbf{0.9632 \pm 0.0003}$ & ${0.9621} \pm {0.0005}$ & ${0.9536 \pm 0.0018}$ & ${0.9609 \pm 0.0005}$ & ${0.9614 \pm 0.0008}$ & ${0.9612} \pm {0.0005}$ & ${0.9619 \pm 0.0003}$\\
    
     \midrule Delivery ETA (RMSE $\downarrow$) & $\mathbf{0.5513 \pm 0.0020}$ & $\underline{0.5519 \pm 0.0017}$ & ${0.5555 \pm 0.0016}$ & $\underline{0.5528 \pm 0.0013}$ & $\underline{0.5528 \pm 0.0017}$ & ${0.5551} \pm {0.0035}$ & ${0.5555 \pm 0.0031}$ \\
    
     \midrule Homecredit Default (ROC-AUC $\uparrow$) & $\mathbf{0.8585 \pm 0.0012}$ & $\underline{0.8579 \pm 0.0012}$ & ${0.8463 \pm 0.0013}$ & $\underline{0.8575 \pm 0.0012}$ & $\underline{0.8579 \pm 0.0014}$ & $\underline{0.8576 \pm 0.0011}$ & ${0.8557 \pm 0.0012}$ \\
    
     \midrule Sberbank Housing (RMSE $\downarrow$) & $\mathbf{0.2424 \pm 0.0024}$  & $\underline{0.2434 \pm 0.0027}$ & ${0.2694 \pm 0.0070}$ & ${0.2453 \pm 0.0036}$ & ${0.2458 \pm 0.0044}$ & ${0.2589} \pm {0.0093}$ & ${0.2465 \pm 0.0080}$ \\
    
     \midrule Black Friday (RMSE $\downarrow$) & $\mathbf{0.6842 \pm 0.0004}$ & ${0.6864} \pm {0.0005}$ & ${0.7011 \pm 0.0040}$ & ${0.6861 \pm 0.0004}$ & ${0.6861 \pm 0.0003}$ & ${0.6963} \pm {0.0012}$ & ${0.6870 \pm 0.0008}$ \\
    
     \midrule Microsoft (RMSE $\downarrow$) & $\mathbf{0.7437 \pm 0.0004}$ & ${0.7442} \pm {0.0003}$ & ${0.7496 \pm 0.0010}$ & $\underline{0.7441 \pm 0.0003}$ & ${0.7448 \pm 0.0004}$ & ${0.7486} \pm {0.0002}$ & ${0.7467 \pm 0.0004}$ \\
    
     \midrule California Housing (RMSE $\downarrow$) & $\mathbf{0.4495 \pm 0.0046}$ & ${0.4602} \pm {0.0042}$ & ${0.6326 \pm 0.2073}$ & ${0.4681 \pm 0.0050}$ & ${0.4639 \pm 0.0024}$ & ${0.4787} \pm {0.0042}$ & ${0.4651 \pm 0.0040}$  \\
     
     \midrule Churn Modeling (ROC-AUC $\uparrow$) & $\mathbf{0.8666 \pm 0.0027}$ & ${0.8616} \pm {0.0015}$ & ${0.7960 \pm 0.0010}$ & ${0.8626 \pm 0.0020}$ & ${0.8622 \pm 0.0020}$ & ${0.8604} \pm {0.0033}$ & ${0.8249 \pm 0.0073}$ \\
    
     \midrule Adult (ROC-AUC $\uparrow$) & $\underline{0.8699 \pm 0.0001}$ & ${0.8688 \pm 0.0012}$  & ${0.7640 \pm 0.0014}$ & ${0.8687 \pm 0.0009}$ & $\mathbf{0.8702 \pm 0.0009}$ & ${0.8683} \pm {0.0013}$ & ${0.8498 \pm 0.0051}$  \\
     
     \midrule Higgs Small (ROC-AUC $\uparrow$) & $\underline{0.7280 \pm 0.0009}$ & $\underline{0.7274 \pm 0.0017}$ & ${0.6263 \pm 0.0573}$ & $\mathbf{0.7282 \pm 0.0021}$ & $\mathbf{0.7282 \pm 0.0009}$ & ${0.7267 \pm 0.0013}$ & ${0.7222 \pm 0.0022}$ \\
    \bottomrule
    \end{tabular}
    }
    \end{table*}

\subsection{Robust Training to Adversarial Attacks} \label{sec:robust_adversarial}

    In this section, we compare \texttt{ALSO} with baselines on the task of robust training of DL model \citep{madry2017towards}. At the first stage, a small CNN \citep{lecun1998gradient} is trained with \texttt{AdamW} for 1 epoch on the MNIST dataset \citep{lecun2010mnist}. Then this pretrained model is trained with adversarial attacks (various transformations from torchvision \citep{marcel2010torchvision}, and the FGSM attack~\citep{Musa_2021}) to obtain a more robust model. As a criterion for the quality of the models we use:
    $
        \text{Accuracy} = \frac{1}{m} \sum_{i=1}^m \text{Accuracy}(\text{Attack}_i)
    $,
    where $\text{Attack}_i$ denotes the quality on the test dataset with the  $i\text{-th attack.}$
    In this section, we slightly change the pipeline of the DRO algorithms; namely, at each iteration $k$ we sample the index $i \sim \text{Cat}( \pi^k )$, that corresponds $i$-th attack. During \texttt{AdamW} training, we sample $i$ from a uniform distribution. 
    Experimental results (see Figure \ref{fig:adversarial_comparison_also_adamw}) demonstrate that \texttt{ALSO} outperforms both \texttt{AdamW} and DRO baselines, highlighting its effectiveness in addressing problems formulated as \eqref{eq:adv_pi_problem}.

    \begin{figure}[h!]
        \centering
        \begin{minipage}{0.43\linewidth}
        \includegraphics[width=1.\linewidth]{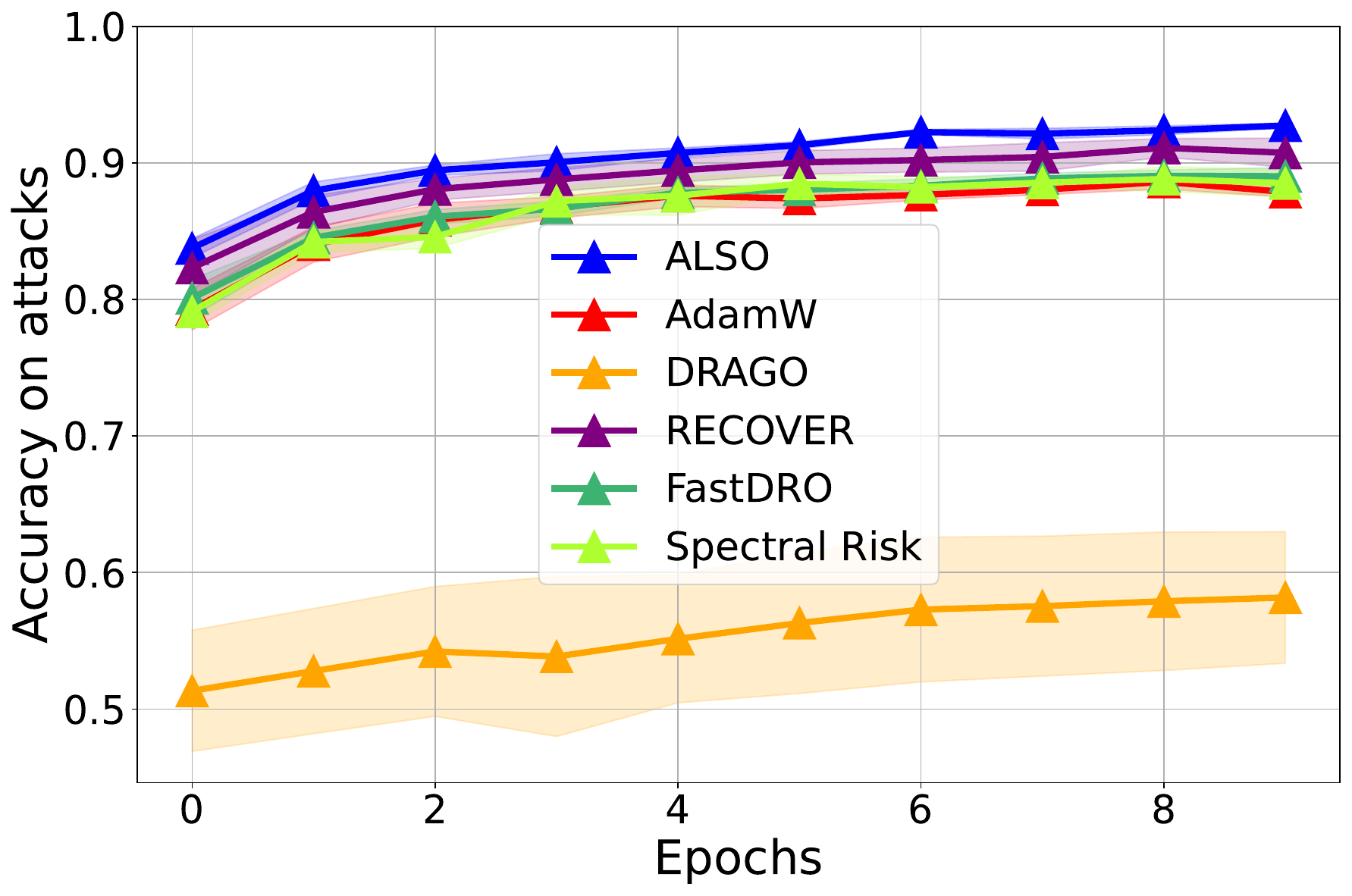}
        \end{minipage}
        \caption{Comparison of accuracy of \texttt{ALSO} with baselines on the attacked test dataset. See details in Appendix \ref{appendix:robust_adversarial}}
        \label{fig:adversarial_comparison_also_adamw}
    \end{figure}
    
\subsection{Distributed Training} \label{sec:distributed_learning}
    
    In this experiment, we consider the problem \eqref{eq:emp_risk} as a distributed optimization problem, where $n$ workers have their own local data on the device. We focus on the case where gradient updates are compressed before being sent to the server. We consider the formulation \eqref{eq:adv_pi_problem}, in which $\pi_i$ is no longer the weight of object $i$, but the weight of worker $i$, and accordingly, the larger $\pi_i$ is, the more worker $i$ will transmit information to the server. We return to ResNet-18 \citep{he2016deep} on the CIFAR-10 dataset \citep{krizhevsky2009learning}, where Perm-K \citep{szlendak2021permcompressors} is chosen as the compressor. In all DRO methods, each worker transmits a personalized fraction $\pi_i$ of gradient coordinates to the server, which generalizes the Perm-K approach. As shown in Figure~\ref{fig:distributed_cifar_f1}, applying the \texttt{ALSO} algorithm in the distributed setup demonstrates superiority over all baselines. 

    \begin{figure}[h!]
    \centering
    \begin{minipage}{0.43\linewidth}
    \includegraphics[width=1.\linewidth]{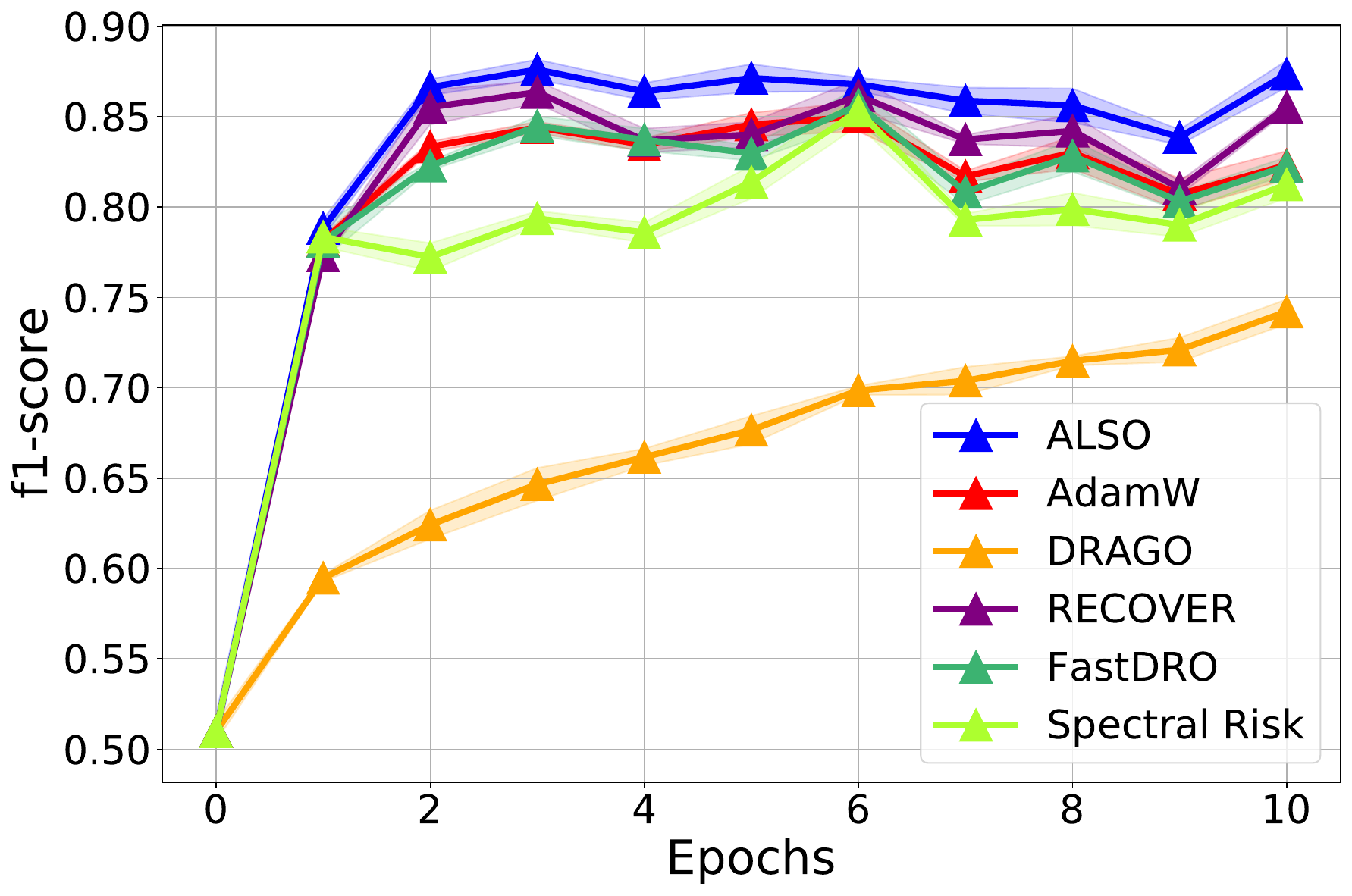}
    \end{minipage}
    \caption{Comparison of f1-score of \texttt{ALSO} with other baselines on the distributed problem. See details in Appendix \ref{appendix:distributed_cifar}}
    \label{fig:distributed_cifar_f1}
    \end{figure}

\subsection{Split Learning}
\label{sec:split_learning}

    In this section, we compare \texttt{ALSO} with baselines in the Split Learning task \citep{vepakomma2018split}. The idea of split learning is to train a shared encoder across multiple tasks distributed over different workers, while maintaining independent heads for each task’s predictions \citep{thapa2021advancements, kim2020multiple}. We use the ResNet-18 \citep{he2016deep} without pretrained weights and simulate a scenario where a new worker joins the training process with the Flowers102 dataset \citep{nilsback2008automated}, while training is already started on the Food101 dataset \citep{food101}. To enhance the performance of the worker that joins the training process at a later stage, we assign class-specific weights for both datasets.
    We compare \texttt{ALSO} optimizer and baselines by measuring Accuracy@5 on both datasets (see Figure \ref{fig:split_learning}). The results show that \texttt{ALSO} outperforms all other methods in terms of faster and more stable convergence, as well as better final metrics. Additional details are provided in Appendix \ref{appendix:split_learning}.

    \begin{figure}[h!]
        \centering
        \begin{minipage}{0.43\linewidth}
            \centering
            \includegraphics[width=\linewidth]{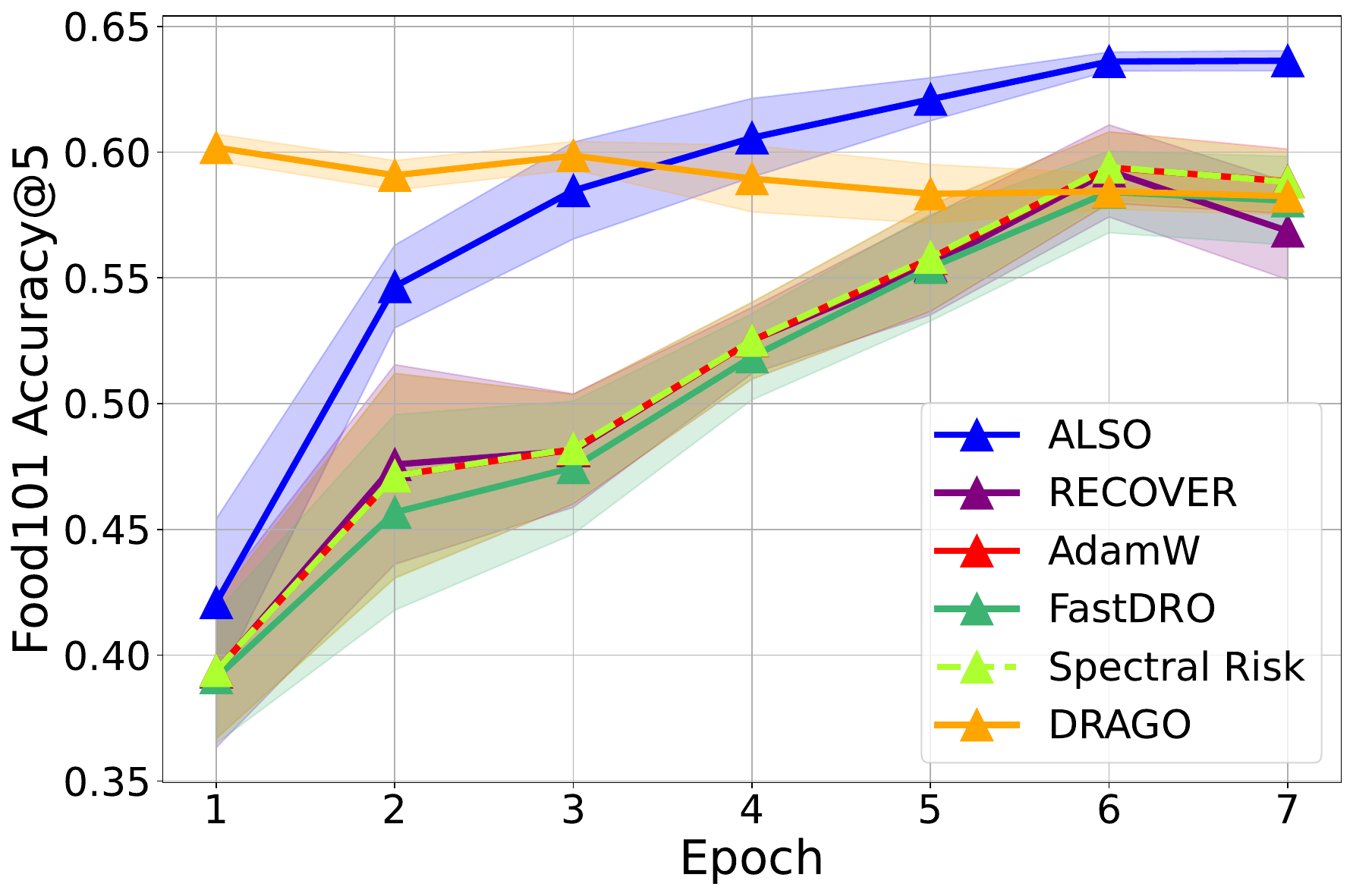}
        \end{minipage}
        \begin{minipage}{0.43\linewidth}
            \centering
            \includegraphics[width=\linewidth]{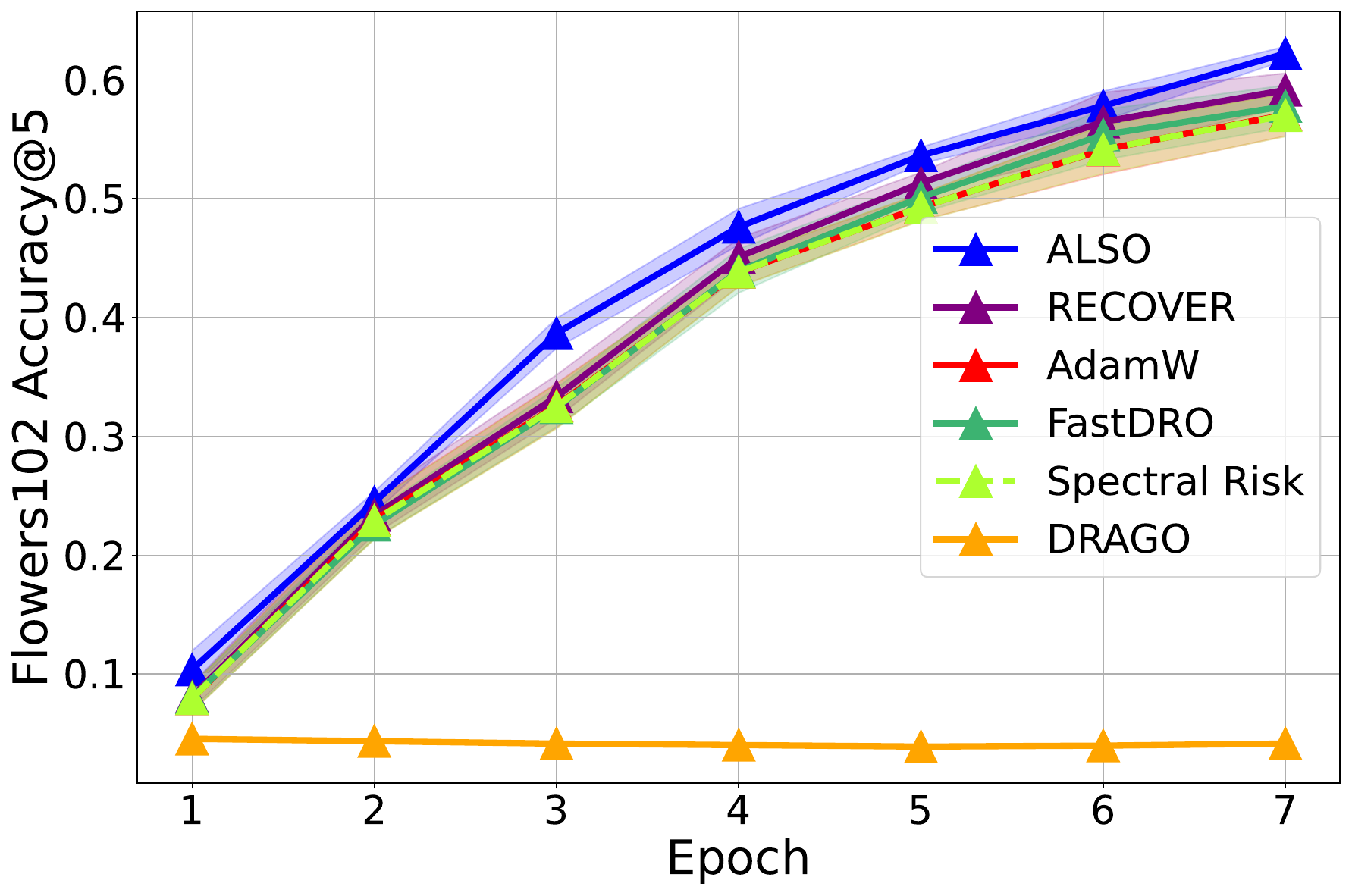}
        \end{minipage}
        \caption{ Metrics comparison for models trained with \texttt{ALSO}, \texttt{AdamW} and Distributionally Robust Optimization methods: \texttt{DRAGO}, \texttt{Spectral Risk}, \texttt{FastDRO}, \texttt{RECOVER} on Flowers102 and Food101 datasets.
        \ref{appendix:split_learning}.}
        \label{fig:split_learning}
    \end{figure}

\section{Ablation Study Summary}

    Due to space constraints, our full ablation studies are presented in Appendix \ref{appendix:ablation}.  Key findings:

    \begin{itemize}
        \item \textbf{Computational Overhead} (Section \ref{appendix:time_analysis}). \texttt{ALSO}'s overhead is insignificant compared to training with \texttt{AdamW}.
        \item \textbf{Hyperparameter Sensitivity} (Section \ref{appendix:hyperparams_sensitivity}). \texttt{ALSO} is stable across a wide range of hyperparameters ($\gamma_\pi, \lambda$), indicating it requires minimal tuning.
        \item \textbf{Design Choices} (Section \ref{appendix:design_choices}). We validate that our design, including momentum ($\alpha$) and a non-adaptive $\pi$ update (Alg. \ref{algorithm:also_optimistic}), is a robust and effective design choice.
        \item \textbf{Tuning Comparison} (Section \ref{appendix:ablation_hparams}). We show \texttt{ALSO}'s performance gain is not due to better hyperparameter tuning by running \texttt{AdamW} with its parameters.
    \end{itemize}

\section{Limitations}

    The main limitation of our approach is the problem we tackle. If one has data heterogeneity or needs distributional robustness, \texttt{ALSO} application is reasonable, otherwise it seems redundant to apply any DRO method. The method's effectiveness relies on meaningful data grouping; if groups are formed arbitrarily, \texttt{ALSO} is unlikely to provide gains. Additionally, usage of DRO methods raises such societal risks as bias and fairness (distribution robustness could inadvertently amplifying biases present in the data), ethical trade-offs (balancing groups' interests involves ethical judgments that must be made transparently). While DRO methods in general and \texttt{ALSO} in particular offer significant potential to the DL community, their integration requires careful consideration of these risks.

\section{Related Work}
\label{sec:vi}
    \textbf{Adaptive methods.} Adaptive optimization is central to modern Deep Learning, where methods such as Adagrad \citep{streeter2010less, duchi2011adaptive}, RMSProp \citep{tieleman2012lecture}, and Adam \citep{kingma2014adam} improve training by adjusting learning rates based on gradient history. Numerous variants extend Adam, e.g., NAdam \citep{dozat2016incorporating}, AMSGrad \citep{reddi2019convergence}, AdamW \citep{loshchilov2017decoupled}. However, these methods target minimization, while many DL problems are more naturally expressed as saddle-point formulations, which require different techniques \citep{browder1966existence, nemirovski2004prox, korpelevich1976extragradient, popov1980modification}. Recent works \citep{daskalakis2017training, gidel2018variational, mertikopoulos2018optimistic, chavdarova2019reducing, liang2019interaction, peng2020training} adapted Adam-like schemes to these settings, demonstrating strong empirical results but relying on limited theory. More rigorous studies have since appeared, e.g. AdaGrad variants \citep{liu2019towards}, Adam-type analyses \citep{dou2021one}, and scaled adaptive methods \citep{beznosikov2022scaled}. Nonetheless, existing results largely focus on the convex-concave Euclidean case and are insufficient for addressing non-convex, distributionally robust objectives such as our formulation \eqref{eq:adv_pi_problem}.

    \textbf{Weighting in Deep Learning.} The idea of weighting each training example has been well studied in the literature \citep{byrd2019effect}. Basic examples of these techniques are the classical method in statistics -- importance sampling \citep{kahn1953methods} -- and AdaBoost \citep{freund1997decision}, where harder examples are selected to train subsequent classifiers. The main applications of loss weighting are learning from unbalanced data \citep{he2009learning, lin2017focal}, continual learning,  which often involves re-weighting past and current samples to ensure that earlier knowledge is not forgotten \citep{aljundi2019gradient}. Another application is making the training process more stable and robust \citep{pang2019libra, bi2022pangu, kendall2017uncertainties, ren2018learning}. There are different approaches for weights assignment: based on specific tasks \citep{pang2019libra, bi2022pangu}, use heuristics for weighting \citep{lin2017focal, dong2017class}, employ a meta-learning approach \citep{ren2018learning, jiang2018mentornet}. Another aspect is non-uniform sampling, which selects examples with varying probabilities to improve optimization. For instance, it has improved convergence in randomized Kaczmarz methods \citep{needell2015randomized}, enhanced stochastic optimization in prox-SMD/SDCA algorithms \citep{zhao2015stochastic}, and been used in SGD variants based on individual example loss \citep{loshchilov2015online}.

\section{Conclusion}

    This paper introduces \texttt{ALSO}, an adaptive optimizer designed to bridge the gap between Distributionally Robust Optimization (DRO) and practical Deep Learning. By incorporating adaptive updates and support for standard batching even with group-based weighting, \texttt{ALSO} effectively addresses the common need to handle data heterogeneity. We provide theoretical convergence guarantees in a stochastic, non-convex setting and demonstrate through extensive experiments across diverse tasks with different challenging types of heterogeneity that \texttt{ALSO} consistently outperforms both standard DL and existing DRO methods. Our work establishes \texttt{ALSO} as a powerful and practical tool for improving the robustness and performance of Deep Learning models in challenging, heterogeneous scenarios.

\end{mainpart}

\begin{appendixpart}
\section{Sampling variants}
\label{appendix:sampling}
In this section, we explore several object sampling strategies for \texttt{ALSO} and demonstrate that each produces unbiased estimates of the gradients for problem \eqref{eq:adv_pi_problem}.  For clarity in our analysis of unbiasedness, we consider a batch size $B=1$ (since batch elements are sampled independently), and we omit iteration indices, using notation $(i, j)$
to represent object indices.

\textbf{Uniform Sampling Across All Objects}. We first examine the sampling approach presented in Lines \ref{alg:also_sampling_line}, \ref{alg:also_theta_grad}, \ref{alg:also_pi_grad} of Algorithm \ref{algorithm:also_optimistic}. Here, a pair $(i, j)$ is sampled with probability $\frac{1}{n}$. This yields:

$$\mathbb{E}g = \mathbb{E}\left(c \pi_i \nabla f_{i, j}(\theta)\right)= \sum_{(i, j)} \frac{c}{n}\pi_i\nabla f_{i, j}(\theta) = \sum_{i=1}^c \pi_i \left(\frac{c}{n }\sum_{j=1}^{n_i}\nabla f_{i, j}(\theta)\right)$$

$$\mathbb{E}p = \mathbb{E}\left(c e_i \cdot f_{i, j}(\theta)\right)= \sum_{(i, j)} \frac{c}{n}e_i\cdot f_{i, j}(\theta) = \sum_{i=1}^c e_i \left(\frac{c}{n }\sum_{j=1}^{n_i}f_{i, j}(\theta)\right)$$

Now let us compute the variance bound:

$$\mathbb{E}_{k,l}\|c \pi_k \nabla f_{k, l}(\theta) -  \sum_{i=1}^c \pi_i \left(\frac{c}{n }\sum_{j=1}^{n_i}\nabla f_{i, j}(\theta)\right) \|^2 =$$
$$ = \sum_{(k,l)} \frac{1}{n}\|c \pi_k \nabla f_{k, l}(\theta) -  \sum_{i=1}^c \pi_i \left(\frac{c}{n }\sum_{j=1}^{n_i}\nabla f_{i, j}(\theta)\right) \|^2 = $$
$$= \sum_{(k,l)} \frac{1}{n}\|\frac{c}{n }\sum_{i=1}^c \sum_{j=1}^{n_i}\left(\pi_i\nabla f_{i, j}(\theta) - \pi_k\nabla f_{k,l}(\theta)\right) \|^2 \leq$$
$$\leq \sum_{(k,l)} \frac{c^2}{n^3}\left(\sum_{i=1}^c \sum_{j=1}^{n_i}\|\pi_i\nabla f_{i, j}(\theta) - \pi_k\nabla f_{k,l}(\theta) \|\right)^2 \leq$$
$$\leq \sum_{(k,l)} \frac{c^2}{n^3}\left(\sum_{i=1}^c \sum_{j=1}^{n_i}\left(\pi_i\|\nabla f_{i, j}(\theta)\| + \pi_k\|\nabla f_{k,l}(\theta) \|\right)\right)^2$$
Since $\|\nabla f_{i, j}(\theta)\| \leq K_{i,j} \leq \max_{i,j} K_{i,j} =: K$:
$$\mathbb{E}_{k,l}\|c \pi_k \nabla f_{k, l}(\theta) -  \sum_{i=1}^c \pi_i \left(\frac{c}{n }\sum_{j=1}^{n_i}\nabla f_{i, j}(\theta)\right) \|^2 \leq \sum_{(k,l)} \frac{c^2}{n^3}\left(\sum_{i=1}^c \sum_{j=1}^{n_i}\left(\pi_i + \pi_k\right)2K\right)^2 =$$
$$= \sum_{(k,l)} \frac{c^2}{n^3}\left(\sum_{i=1}^c \left(\pi_i + \pi_k\right)2n_iK\right)^2$$
Using Cauchy-Schwarz inequality:
$$\sum_{(k,l)} \frac{c^2}{n^3}\left(\sum_{i=1}^c \left(\pi_i + \pi_k\right)2n_iK\right)^2 \leq \sum_{(k,l)} \frac{4c^2K^2}{n^3}\left(\sum_{i=1}^c \left(\pi_i + \pi_k\right)^2 \sum_{i=1}^c n_i^2\right)$$
Since $(a+b)^2 \leq 2a^2+2b^2$:
$$\sum_{(k,l)} \frac{4c^2K^2}{n^3}\left(\sum_{i=1}^c \left(\pi_i + \pi_k\right)^2 \sum_{i=1}^c n_i^2\right) \leq \sum_{(k,l)} \frac{8c^2K^2 \sum_{i=1}^c n_i^2}{n^3}\sum_{i=1}^c \left(\pi_i^2 + \pi_k^2\right) =$$
$$=\sum_{(k,l)} \pi_k^2\frac{8c^3K^2 \sum_{i=1}^c n_i^2}{n^3}\sum_{i=1}^c \pi_i^2 $$
Since $p_i \in \Delta_{c-1} \Rightarrow \sum_{i=1}^c \pi_i^2 \leq1$:
$$\sum_{(k,l)} \pi_k^2\frac{8c^3K^2 \sum_{i=1}^c n_i^2}{n^3}\sum_{i=1}^c \pi_i^2 \leq \sum_{k=1}^c\sum_{l=1}^{n_k} \pi_k^2\frac{8c^3K^2 \sum_{i=1}^c n_i^2}{n^3} \leq \frac{8c^3K^2 \sum_{i=1}^c n_i^2}{n^3} \sum_{k=1}^c n_k \pi_k^2 \leq$$
$$\leq \frac{8c^3K^2 \sum_{i=1}^c n_i^2}{n^3} \sum_{k=1}^c n_k = \frac{8c^3K^2 \sum_{i=1}^c n_i^2}{n^2}$$

Now we will similarly consider $p$:

$$\mathbb{E}_{k,l}\|c e_k f_{k, l}(\theta) -  \sum_{i=1}^c \left(\frac{c}{n}e_i\sum_{j=1}^{n_i} f_{i, j}(\theta)\right) \|^2 = \sum_{(k,l)} \frac{1}{n} \|\frac{c}{n}\sum_{i=1}^c\sum_{j=1}^{n_i} \left(e_k f_{k,l}(\theta) - e_i f_{i,j}(\theta)\right)\|^2 \leq$$
$$\leq \sum_{(k,l)} \frac{c^2}{n^3} \left(\sum_{i=1}^c\sum_{j=1}^{n_i} \|e_k f_{k,l}(\theta) - e_i f_{i,j}(\theta)\|\right)^2 = $$
$$= \sum_{(k,l)} \frac{c^2}{n^3} \left(\sum_{(i,j)} \|e_k f_{k,l}(\theta) - e_k f_{k,l}(\theta^*) + e_k f_{k,l}(\theta^*) - e_i f_{i,j}(\theta^*) + e_i f_{i,j}(\theta^*) - e_i f_{i,j}(\theta)\|\right)^2$$
where $\theta^* = \argmin_{\theta \in \mathbb{R}^d} \max_{\pi \in \Delta_{c-1}}h(\theta, \pi)$. Thus:
$$\mathbb{E}_{k,l}\|c e_k f_{k, l}(\theta) -  \sum_{i=1}^c \left(\frac{c}{n}e_i\sum_{j=1}^{n_i} f_{i, j}(\theta)\right) \|^2 \leq \sum_{(k,l)} \frac{c^2}{n^3} (\sum_{(i,j)} \|e_k f_{k,l}(\theta) - e_k f_{k,l}(\theta^*)\| + \|e_k f_{k,l}(\theta^*)\|$$
$$+ \|e_i f_{i,j}(\theta^*)\| + \|e_i f_{i,j}(\theta^*) - e_i f_{i,j}(\theta)\|)^2$$

Since $\|e_i f_{i,j}(\theta^*) - e_i f_{i,j}(\theta)\| = \|e_i(f_{i,j}(\theta^*) - f_{i,j}(\theta))\| = |f_{i,j}(\theta^*) - f_{i,j}(\theta)| \leq K_{i,j}\|\theta^* - \theta\| \leq K ||\theta^*-\theta||$ and $\|e_k f_{k,l}(\theta^*)\| = |f_{k,l}(\theta^*)| \leq \max_{k,l}|f_{k,l}(\theta^*)| =: G$:
$$\mathbb{E}_{k,l}\|c e_k f_{k, l}(\theta) -  \sum_{i=1}^c \left(\frac{c}{n}e_i\sum_{j=1}^{n_i} f_{i, j}(\theta)\right) \|^2 \leq \sum_{(k,l)} \frac{c^2}{n^3} \left(\sum_{(i,j)} (2G + 2K\|\theta - \theta^*\|)\right)^2 = $$
$$= \sum_{(k,l)} \frac{c^2}{n^3} \left(n (2G + 2K\|\theta - \theta^*\|)\right)^2 =\frac{c^2}{n^3}n^3 (2G + 2K\|\theta - \theta^*\|)^2 =$$
$$= c^2(2G + 2K\|\theta - \theta^*\|)^2 \leq 8c^2 (G^2+K^2\|\theta-\theta^*\|)$$

Thus:
$$\sigma^2 \leq \max\left\{\frac{8c^3K^2 \sum_{i=1}^c n_i^2}{n^2}, 8c^2 (G^2+K^2\|\theta-\theta^*\|)\right\} = \mathcal{O}(K^2)$$

\textbf{Two-Stage Group-Object Sampling}. An alternative approach involves a two-stage sampling process: first sample a group index with uniform probability $\frac{1}{c}$, then sample an object from this group with uniform probability $\frac{1}{n_i}$. This gives a probability $\frac{1}{c n_i}$ for selecting object $(i, j)$. To maintain unbiased gradient estimates, we modify the scaling in Lines \ref{alg:also_theta_grad}, \ref{alg:also_pi_grad} as follows:

$$g = \frac{c^2n_i}{n}\pi_i \nabla f_{i, j}(\theta)$$

$$p = \frac{c^2n_i}{n}e_i \cdot f_{i, j}(\theta)$$

Then

$$\mathbb{E}g = \mathbb{E}\left(\frac{c^2n_i}{n}\pi_i \nabla f_{i, j}(\theta)\right)= \sum_{i=1}^c \sum_{j=1}^{n_i} \frac{c^2n_i}{n}\pi_i \nabla f_{i, j}(\theta) \frac{1}{cn_i} = \sum_{i=1}^c \pi_i \left(\frac{c}{n }\sum_{j=1}^{n_i}\nabla f_{i, j}(\theta)\right)$$

$$\mathbb{E}p = \mathbb{E}\left(\frac{c^2n_i}{n}e_i \cdot f_{i, j}(\theta)\right)= \sum_{i=1}^c \sum_{j=1}^{n_i} \frac{c^2n_i}{n}e_i\cdot f_{i, j}(\theta) \frac{1}{cn_i} = \sum_{i=1}^c e_i \left(\frac{c}{n }\sum_{j=1}^{n_i}\nabla f_{i, j}(\theta)\right)$$

\textbf{Probability-Weighted Group Sampling}. A third variant samples group $i$ according to its weight $\pi_i$, i.e. $i \sim Cat(\pi)$ followed by uniform sampling of $j$ with probability $\frac{1}{n_i}$. This gives a selection probability of $\frac{\pi_i}{n_i}$ for pair $(i, j)$. We adjust the scaling factors as:

$$g = \frac{cn_i}{n} \nabla f_{i, j}(\theta)$$

$$p =  \frac{cn_i}{n\pi_i}e_i \cdot f_{i, j}(\theta)$$

Then

$$\mathbb{E}g = \mathbb{E}\left(\frac{cn_i}{n} \nabla f_{i, j}(\theta)\right)= \sum_{i=1}^c \sum_{j=1}^{n_i} \frac{cn_i}{n\pi_i}\pi_i \nabla f_{i, j}(\theta) \frac{\pi_i}{n_i} = \sum_{i=1}^c \pi_i \left(\frac{c}{n }\sum_{j=1}^{n_i}\nabla f_{i, j}(\theta)\right)$$

$$\mathbb{E}g = \mathbb{E}\left(\frac{cn_i}{n\pi_i}e_i \cdot f_{i, j}(\theta)\right)= \sum_{i=1}^c \sum_{j=1}^{n_i} \frac{cn_i}{n\pi_i}e_i \cdot f_{i, j}(\theta) \frac{\pi_i}{n_i} = \sum_{i=1}^c e_i \left(\frac{c}{n }\sum_{j=1}^{n_i} f_{i, j}(\theta)\right)$$

\textbf{Note.} We employ the third sampling technique in our Robust Training experiments (see Section \ref{sec:robust_adversarial}). To see it let $n_i=k$ $\forall i$, where $k$ is the dataset length. In this scenario $n=c \cdot k$ is effective dataset size (i.e. attacked object can be considered as separate object). Since $j$ is independent of $i$ now we can reverse sampling order: first sample $j$, then sample $i$. This implementation -- sample objects and then sample attacks for them -- allows seamlessly integrate \texttt{ALSO} into a standard training procedure.

\newpage
\section{Missing Experiment Details}
\label{appendix:experimental_summary}
\subsection{Unbalanced Data Details (Section \ref{sec:unbalanced_data})}
\label{appendix:unbalanced_cifar}

\textbf{Baselines description.} Now, let us discuss described basic imbalance handling techniques. The first of these techniques is known as \textit{upsampling} \citep{he2009learning, kahn1953methods}, the idea is to sample objects for gradient calculation at the current optimization step not uniformly, but proportionally to the class ratio of each object in the training dataset. For the $\hat{\pi}$ regularizer in the problem \eqref{eq:adv_pi_problem}, we utilize this modified distribution instead of the vanilla uniform distribution $\mathcal{U}(\overline{1, n})$. This choice results in a significant improvement in the performance.
The second technique is called \textit{static weights} \citep{he2009learning}. 
Its idea is similar to the previous method, however, instead of modifying the sampling distribution, objects are sampled uniformly. The class imbalance is then addressed by multiplying the loss function for each object by a weight equal to the inverse ratio of the number of objects belonging to that class in the training dataset.

\textbf{Data preprocessing.}
For all optimizers the same preprocessing was used for fair comparison. We modified the images from CIFAR-10 train dataset with Normalizing and classical computer vision augmentations: Random Crop \citep{takahashi2019data}, Random  Horizontally Flip.

\textbf{Training neural networks.}
We use cross-entropy as the loss function.
We do not apply learning rate schedules since we tune hyperparameters.
We use a predefined batch size equal to $64$ and maximum number of epochs equal to $20$.

\textbf{Hyperparameter tuning.}
Hyperparameter tuning is performed with the TPE sampler (200 iterations) with $5$ epoch from the Optuna package \citep{akiba2019optuna}. Hyperparameter tuning spaces for experiment are provided in Table \ref{table:unbalaced_cifar_hyperparams}.

\begin{table}[h!]
\centering
\resizebox{\textwidth}{!}{
\begin{tabular}{lll}
    \toprule
    Parameter           & Distribution \\
     \midrule
    Learning rate       & $\mathrm{LogUniform}[1e\text{-}4, 1e\text{-}2]$ \\
    Weight decay        & $\mathrm{LogUniform}[1e\text{-}6, 1e\text{-}2]$ \\
     \midrule
    $\pi$-Learning rate ($\gamma_\pi$ from \texttt{ALSO}, used for \texttt{ALSO}, \texttt{DRAGO})      & $\mathrm{LogUniform}[1e\text{-}5, 1e\text{-}3]$ \\
    
    $\pi$-regularization ($\lambda$ from \texttt{ALSO}, used for \texttt{ALSO}, \texttt{DRAGO}, \texttt{RECOVER}, \texttt{Spectral Risk})       & $\mathrm{LogUniform}[1e\text{-}3, 1]$ \\
    \bottomrule
\end{tabular}}
\caption{The hyperparameter tuning space for unbalanced data experiment.}
\label{table:unbalaced_cifar_hyperparams}
\end{table}

\textbf{Evaluation.}
The tuned hyperparameters are evaluated under $20$ random seeds.
The mean test metric and its standard deviation over these random seeds are then used to compare algorithms as described in Section \ref{sec:unbalanced_data}.

\newpage
\subsection{Tabular Deep Learning Details (Section \ref{sec:tabdl})}
\label{appendix:tabdl}

\begin{table*}[!h]
\centering
\resizebox{\textwidth}{!}{
\begin{tabular}{@{}llccccccllc@{}}
\toprule
Name & \# Train & \# Validation & \# Test & \# Num & \# Bin & \# Cat & Task type & Metric & Heterogeniety & Batch size \\
\midrule
Sberbank Housing & $18\,847$ & $4\,827$ & $4\,647$ & $365$ & $17$ & $10$ & Regression & RMSE & Heavy-tailed & 256 \\
Ecom Offers & $109\,341$ & $24\,261$ & $26\,455$ & $113$ & $6$ & $0$ & Binclass & ROC AUC & Extreme shift & 1024 \\
Maps Routing & $160\,019$ & $59\,975$ & $59\,951$ & $984$ & $0$ & $2$ & Regression & RMSE & - &1024 \\
Homesite Insurance & $224\,320$ & $20\,138$ & $16\,295$ & $253$ & $23$ & $23$ & Binclass &ROC AUC & Class imbalance & 1024 \\
Cooking Time & $227\,087$ & $51\,251$ & $41\,648$ & $186$ & $3$ & $3$ & Regression & RMSE & Heavy-tailed & 1024 \\
Homecredit Default & $267\,645$ & $58\,018$ & $56\,001$ & $612$ & $2$ & $82$ & Binclass & ROC AUC & High uncertainty & 1024 \\
Delivery ETA & $279\,415$ & $34\,174$ & $36\,927$ & $221$ & $1$ & $1$ & Regression & RMSE & Non-symmetric & 1024 \\
Weather & $106\,764$ & $42\,359$ & $40\,840$ & $100$ & $3$ & $0$ & Regression & RMSE & Non-symmetric  & 1024 \\
Churn Modelling & $6\,400$ & $1\,600$ & $2\,000$ & $10$ & $3$ & $1$ & Binclass & ROC AUC & Noisy data &128 \\
California Housing & $13\,209$ & $3\,303$ & $4\,128$ & $8$ & $0$ & $0$ & Regression & RMSE & Heavy-tailed & 256 \\
Adult & $26\,048$ & $6\,513$ & $16\,281$ & $6$ & $1$ & $8$ & Binclass & ROC AUC & High uncertainty &256 \\
Higgs Small & $62\,751$ & $15\,688$ & $19\,610$ & $28$ & $0$ & $0$ & Binclass & ROC AUC & - &512 \\
Black Friday & $106\,764$ & $26\,692$ & $33\,365$ & $4$ & $1$ & $4$ & Regression & RMSE & Heavy-tailed & 512 \\
Microsoft & $723\,412$ & $235\,259$ & $241\,521$ & $131$ & $5$ & $0$ & Regression & RMSE & - & 1024 \\
\bottomrule
\end{tabular}
}
\caption{Properties of the datasets from \citep{gorishniy2024tabr, rubachev2024tabred}. “\# Num”, “\# Bin”, and “\# Cat” denote the number of numerical, binary, and categorical features, respectively}

\label{table:tabdl-datasets-properties}
\end{table*}

We mostly follow the experiment setup from \citep{gorishniy2024tabm}. As such, most of the text below is copied from \citep{gorishniy2024tabm}.

\textbf{Data preprocessing.}
For each dataset, for all optimizers, the same preprocessing was used for fair comparison.
For numerical features, by default, we used a slightly modified version of the quantile normalization from the Scikit-learn package \citep{pedregosa2011scikit} (see the source code), with rare exceptions when it turned out to be detrimental (for such datasets, we used the standard normalization or no normalization).
For categorical features, we used one-hot encoding.
Binary features (i.e. the ones that take only two distinct values) are mapped to $\{0,1\}$ without any further preprocessing.

\textbf{Training neural networks.}
We use cross-entropy for classification problems and mean squared error for regression problems as loss function.
We do not apply learning rate schedules.
We do not use data augmentations.
We apply global gradient clipping to $1.0$.
For each dataset, we used a predefined dataset-specific batch size.
We continue training until there are $\texttt{patience}$ consecutive epochs without improvements on the validation set; we set $\texttt{patience} = 16$.

\textbf{Hyperparameter tuning.}
In most cases, hyperparameter tuning is performed with the TPE sampler (100 iterations) from the Optuna package \citep{akiba2019optuna}.
Hyperparameter tuning spaces for experiment are provided in Table \ref{table:tabdl-hyperparams}.

\textbf{Evaluation.}
On a given dataset, for a given model, the tuned hyperparameters are evaluated under multiple (in most cases, $15$) random seeds.
The mean test metric and its standard deviation over these random seeds are then used to compare algorithms as described in Table \ref{table:tabdl-datasets-properties}.

\begin{table}[h!]
\centering
{\renewcommand{\arraystretch}{1.2}

\resizebox{\textwidth}{!}{
\begin{tabular}{lll}
    \toprule
    Parameter           & Distribution \\
    \midrule
    \# layers           & $\mathrm{UniformInt}[1,5]$ \\
    Width (hidden size) & $\mathrm{UniformInt}[64,1024]$ \\
    Dropout rate        & $\{0.0, \mathrm{Uniform}[0.0,0.5]\}$ \\
   
    n\_frequencies  & $\mathrm{UniformInt}[16,96]$ \\
    
    d\_embedding      & $\mathrm{UniformInt}[16,32]$ \\
    
    frequency\_init\_scale & $\mathrm{LogUniform}[1e\text{-}2, 1e\text{1}]$ \\
    
    \midrule
    Learning rate       & $\mathrm{LogUniform}[3e\text{-}5, 1e\text{-}3]$ \\
    
    Weight decay        & $\{0, \mathrm{LogUniform}[1e\text{-}4, 1e\text{-}1]\}$ \\
    
    $\pi$-Learning rate ($\gamma_\pi$ from \texttt{ALSO}, used for \texttt{ALSO}, \texttt{DRAGO})      & $\mathrm{LogUniform}[1e\text{-}5, 1e\text{-}3]$ \\
    
    $\pi$-regularization ($\lambda$ from \texttt{ALSO}, used for \texttt{ALSO}, \texttt{DRAGO}, \texttt{RECOVER}, \texttt{Spectral Risk})       & $\mathrm{LogUniform}[1e\text{-}3, 1]$ \\
    
    Size (used for \texttt{FastDRO})        & $\mathrm{Uniform}[0, 1]$ \\

    n\_draws (used for \texttt{Spectral Risk})        & $\mathrm{LogUniform}[1e\text{-}3, 1]$\\
    \bottomrule
\end{tabular}}
}
\caption{The hyperparameter tuning space for tabular Deep Learning experiment.}
\label{table:tabdl-hyperparams}
\end{table}

\newpage
\subsection{Robust Training to Adversarial Attacks (Section \ref{sec:robust_adversarial})}
\label{appendix:robust_adversarial}
We provide a detailed description of the experimental pipeline employed in our study. Our approach is based on a modified version of the \texttt{ALSO} pipeline. Specifically, at each iteration, we sample an index $i$ from a categorical distribution parameterized by $\pi^k$, apply the corresponding attack, and then proceed with the \texttt{ALSO} step. For comparison, the baseline pipeline consists of standard optimization using the \texttt{AdamW} optimizer, where the index $i$ is sampled from a uniform distribution.

The general procedure for each algorithm can be summarized as follows:
\begin{enumerate}
    \item Sample a mini-batch $(X^B_{\text{train}}, y^B_{\text{train}})$ from the training set $X_{\text{train}}$.
    \item Sample $i \sim \text{Categorical}(\pi^k)$ to select the attack for the batch, or sample $i$ from a uniform distribution in the baseline case.
    \item Perform an optimizer step to update $\theta$ and $\pi$ (if required)
\end{enumerate}
The hyperparameters used in our experiments are as follows: $\tau = 1$ (for DRO algorithms), $\gamma_\pi = 0.1$ for the \texttt{ALSO}, and a learning rate of $\text{lr} = 10^{-3}$ for all pipelines.

\begin{figure}[h!]
    \centering
    \includegraphics[width=0.8\linewidth]{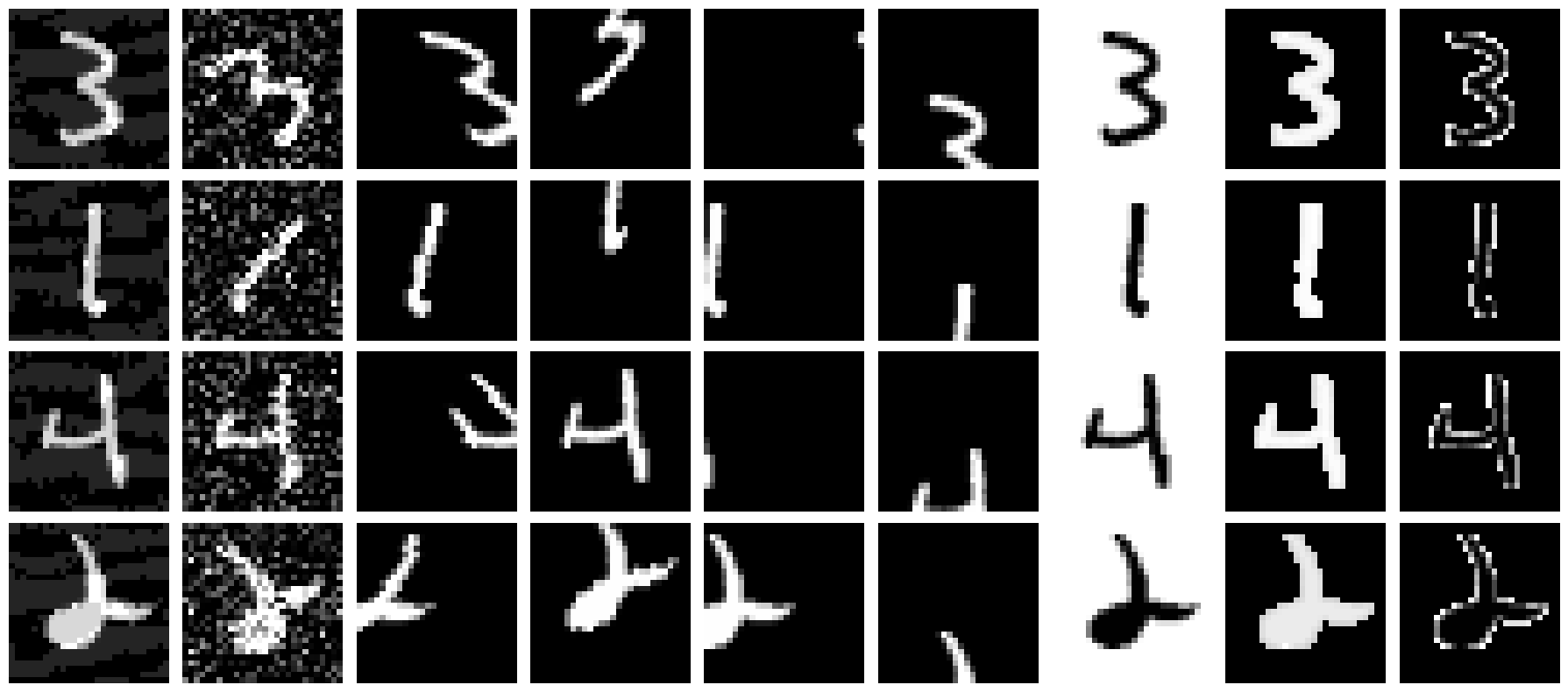}
        \caption{Examples of applied attacks to the test and train datasets.}
    \label{fig:adversarial_attacks}
\end{figure}

\subsection{Distributed Training Details (Section \ref{sec:distributed_learning})} \label{appendix:distributed_cifar}

\textbf{Data preprocessing.}
For all optimizers the same preprocessing was used for fair comparison. We modified the images from CIFAR-10 train dataset with Normalizing and classical computer vision augmentations: Random Crop \citep{takahashi2019data}, Random  Horizontally Flip.

\textbf{Parameter selection.} Different numbers of workers, different class distributions, and different class distributions among workers were considered during the experiments.

\textbf{Training neural networks.}
We use cross-entropy as the loss function. We use a predefined batch size equal to $1024$ and maximum number of epochs equal to $20$.

\newpage
\subsection{Split Learning (Section \ref{sec:split_learning})}
\label{appendix:split_learning}

\textbf{Split Learning Motivation.} The idea behind split learning is to train a shared encoder across multiple tasks distributed over different workers, while maintaining independent heads for each task’s predictions \citep{thapa2021advancements}. This approach enables collaborative training without sharing raw data, enhancing privacy, and reduces computational and communication overhead, making it suitable for low-resource or budget-constrained settings \citep{kim2020multiple}.

\textbf{Experiment Details.} First, we train a ResNet-18 model on the Food101 dataset \citep{food101} using the \texttt{AdamW} optimizer for 3 epochs. Next, we simulate the split learning process by introducing the Flowers102 dataset \citep{nilsback2008automated} into the training scheme. Training proceeds by alternating between datasets every epoch: one epoch on Food101, then one epoch on Flowers102, and so on. During each epoch, we train the shared encoder, while using separate linear heads for each dataset. For the DRO methods, we apply class weights for both datasets.

\textbf{Technical Details.} The default learning rate was set to $3 \times 10^{-4}$. Baseline hyperparameters were selected based on prior experiments and tuned over up to 5 iterations. The $\pi$-learning rate and $\pi$-decay parameters were kept at their default values of $1e\text{-}5$ and $1e\text{-}2$, respectively, without further tuning. All experiments were conducted on an NVIDIA Tesla V100 GPU.

\textbf{Comparison Details.} We compared the \texttt{ALSO} optimizer against baseline methods using the Accuracy@5 metric. Except for \texttt{DRAGO}, all DRO baselines improved accuracy on the newly introduced Flowers102 dataset at the expense of some accuracy loss on the original Food101 dataset, relative to \texttt{AdamW}. In contrast, \texttt{ALSO} outperformed \texttt{AdamW} on both datasets, demonstrating its ability to acquire out-of-domain knowledge without degrading performance on the initial task.

\newpage
\section{Ablation Study}
\label{appendix:ablation}

\subsection{\texttt{ALSO} Step Time Analysis}
\label{appendix:time_analysis}
\begin{figure}[h!]
    \centering
    \begin{minipage}{0.7\linewidth}
        \centering
        \includegraphics[width=\linewidth]{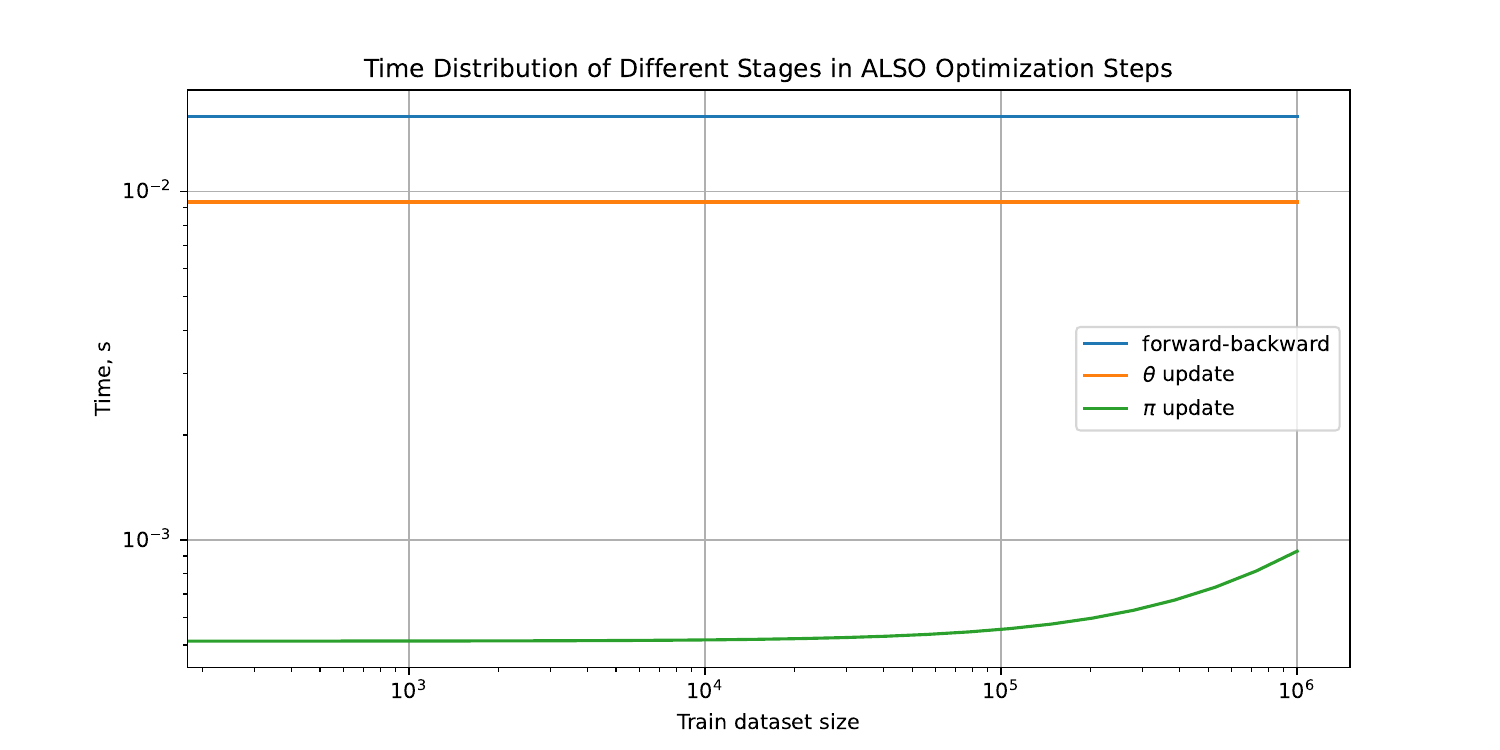}
    \end{minipage}
    \caption{Time distribution over dataset size of three main parts of optimization process with \texttt{ALSO}: gradient computation (forward-backward), $\theta$ update and $\pi$ update. The trained model is ResNet-18 with batch size. Time of each part is averaged across $25$ training steps. We want to highlight, that gradient computations are required for all first order optimization methods, and this measurement is used only for comparison. }
    \label{fig:time_analysis}
\end{figure}
To analyze the time consumption of each component in the optimization process with $\texttt{ALSO}$, we conduct an experiment training ResNet-18 \citep{he2016deep} with a fixed batch size of $64$ across various dataset sizes, measured time is averaged across $25$ iterations. This approach is chosen because while $\pi$ updates depend on dataset size, gradient computation and $\theta$ updates do not. We test dataset sizes up to 1 million samples, which exceeds our largest experimental dataset, which contains approximately $800 000$ samples. The experiment was conducted on one NVIDIA GeForce RTX 2080 Ti GPU. We want to highlight, that gradient computations are required for all first order optimization methods, and this measurement is used  only for comparison. 

The results, presented in Figure \ref{fig:time_analysis}, reveal a clear hierarchy in computational demands. Gradient computation (forward-backward passes) consistently requires significantly more time than both $\theta$ and $\pi$ updates across all dataset sizes, which is consistent with \citep{jiang2021optimizer}. Furthermore, $\theta$ updates consistently demand more computational time than $\pi$ updates. This experiment leads to conclusion that the explicit weight vector update ($\pi$ update) is computationally negligible relative to the overall training step time.

\subsection{Hyperparameters sensitivity}
\label{appendix:hyperparams_sensitivity}

This ablation study examines \texttt{ALSO}'s sensitivity to its $\pi$-specific hyperparameters: the $\pi$-learning rate ($\gamma_\pi$) and $\pi$-regularization ($\lambda$). We conducted full 2D sweeps for both parameters, fixing model weight learning rates and regularization to isolate their impact. Results from the imbalanced data setting (Section \ref{sec:unbalanced_data}) show consistent performance across varying imbalance coefficients (Figure \ref{fig:unbalanced_rob}). Similarly, in Split Learning (Section \ref{sec:split_learning}), 2D sweeps confirm broad robustness on Food101 and Flowers102 datasets (Figure \ref{fig:split_rob}). Across all experiments, \texttt{ALSO} proves largely insensitive to $\gamma_\pi$ and $\lambda$ settings, suggesting strong performance is achievable without extensive tuning.

\begin{figure}[h!]
    \centering
    \includegraphics[width=0.45\linewidth]{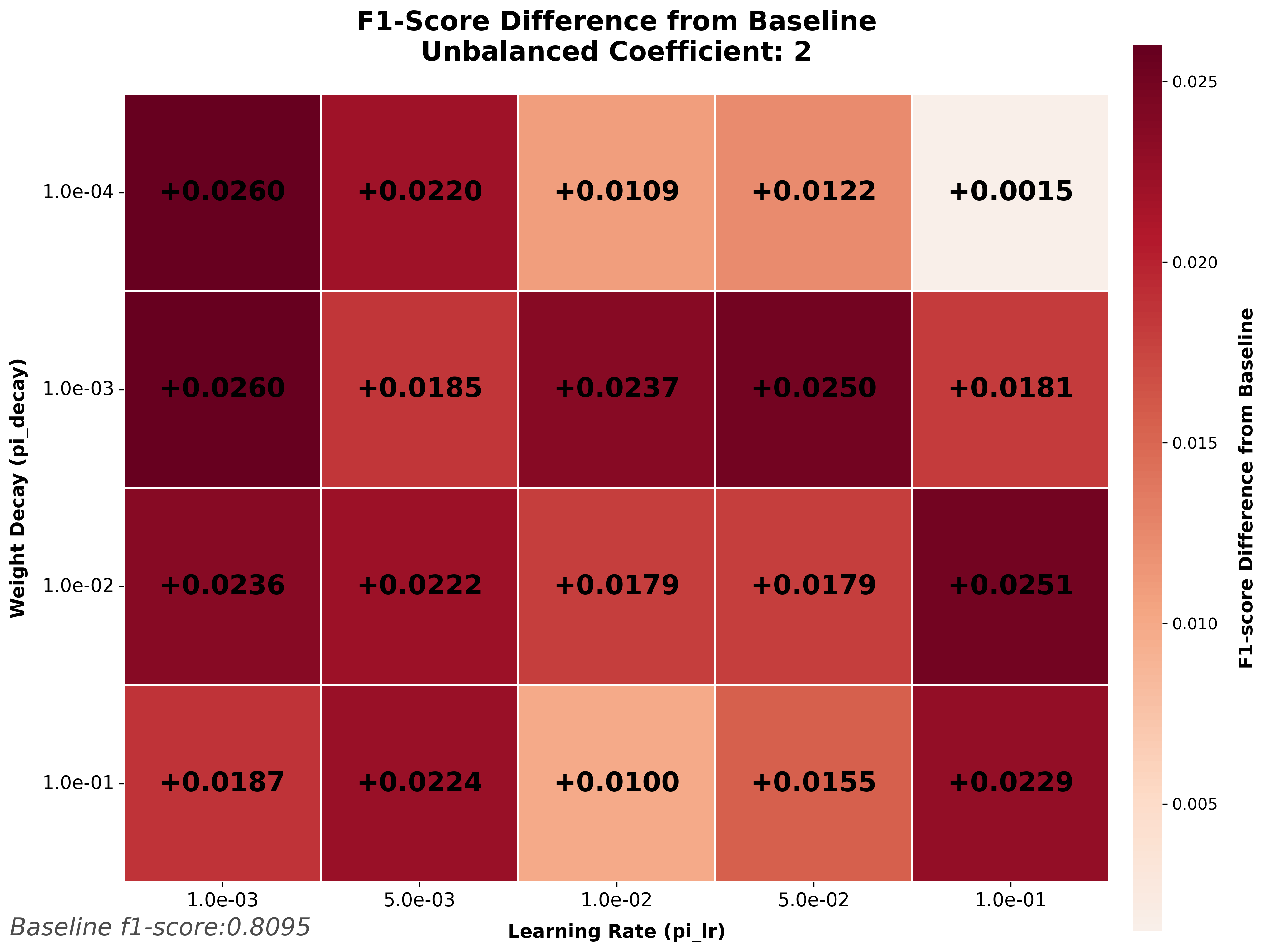}
    \includegraphics[width=0.45\linewidth]{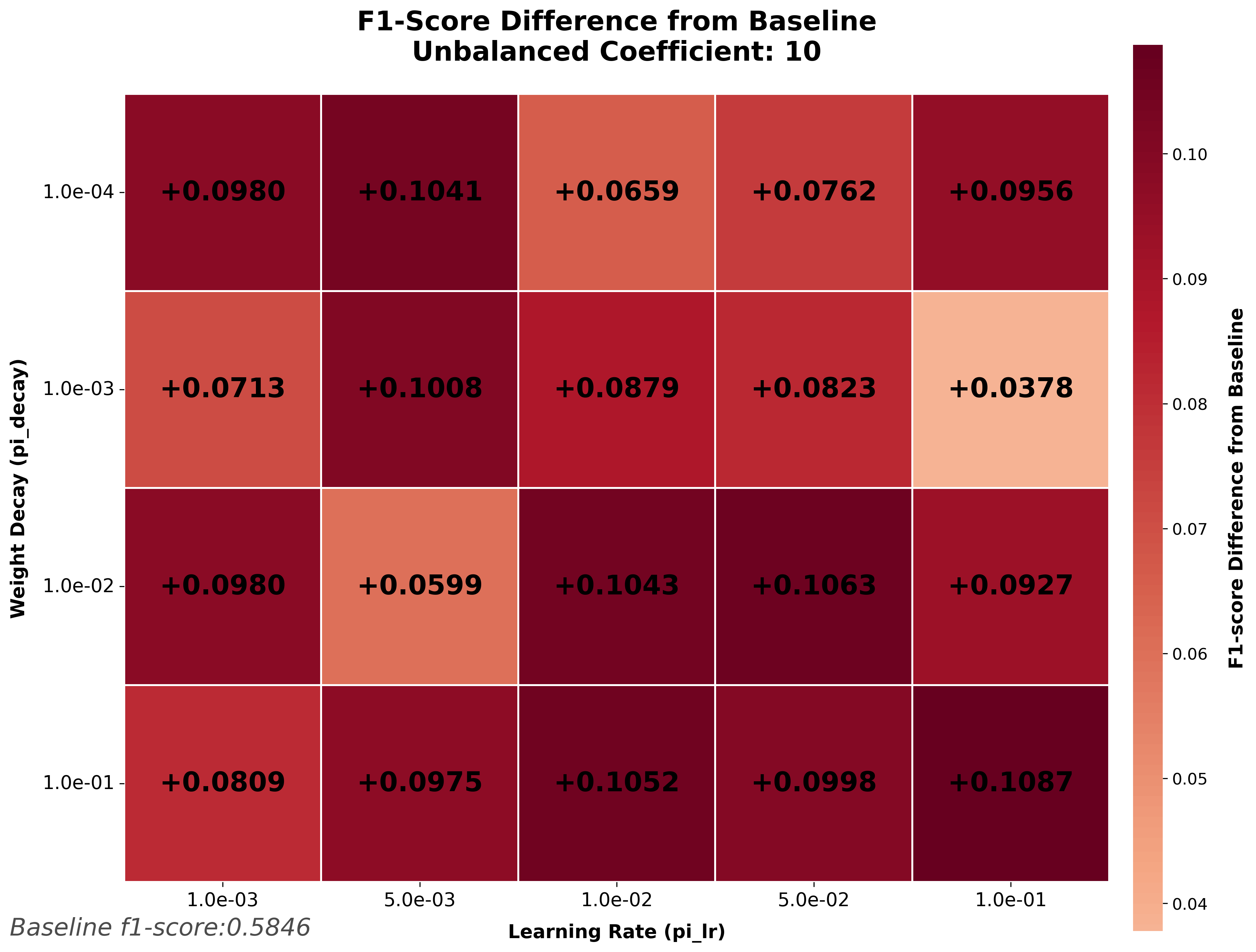}
    \caption{Robustness of \texttt{ALSO} to $\pi$-hyperparameters ($\Delta$F1-score vs. \texttt{AdamW} baseline). Each cell shows the F1-score difference between \texttt{ALSO} and \texttt{AdamW} with static weights (baseline), over a full 2D grid of $\pi$-learning rate ($\gamma_\pi$) and $\pi$-regularization ($\lambda$). All cells are red (positive $\Delta$F1), indicating that ALSO consistently outperforms the baseline across the entire grid and for different imbalance coefficients (2 and 10).}
    \label{fig:unbalanced_rob}
\end{figure}

\begin{figure}[h!]
    \centering
    \includegraphics[width=0.45\linewidth]{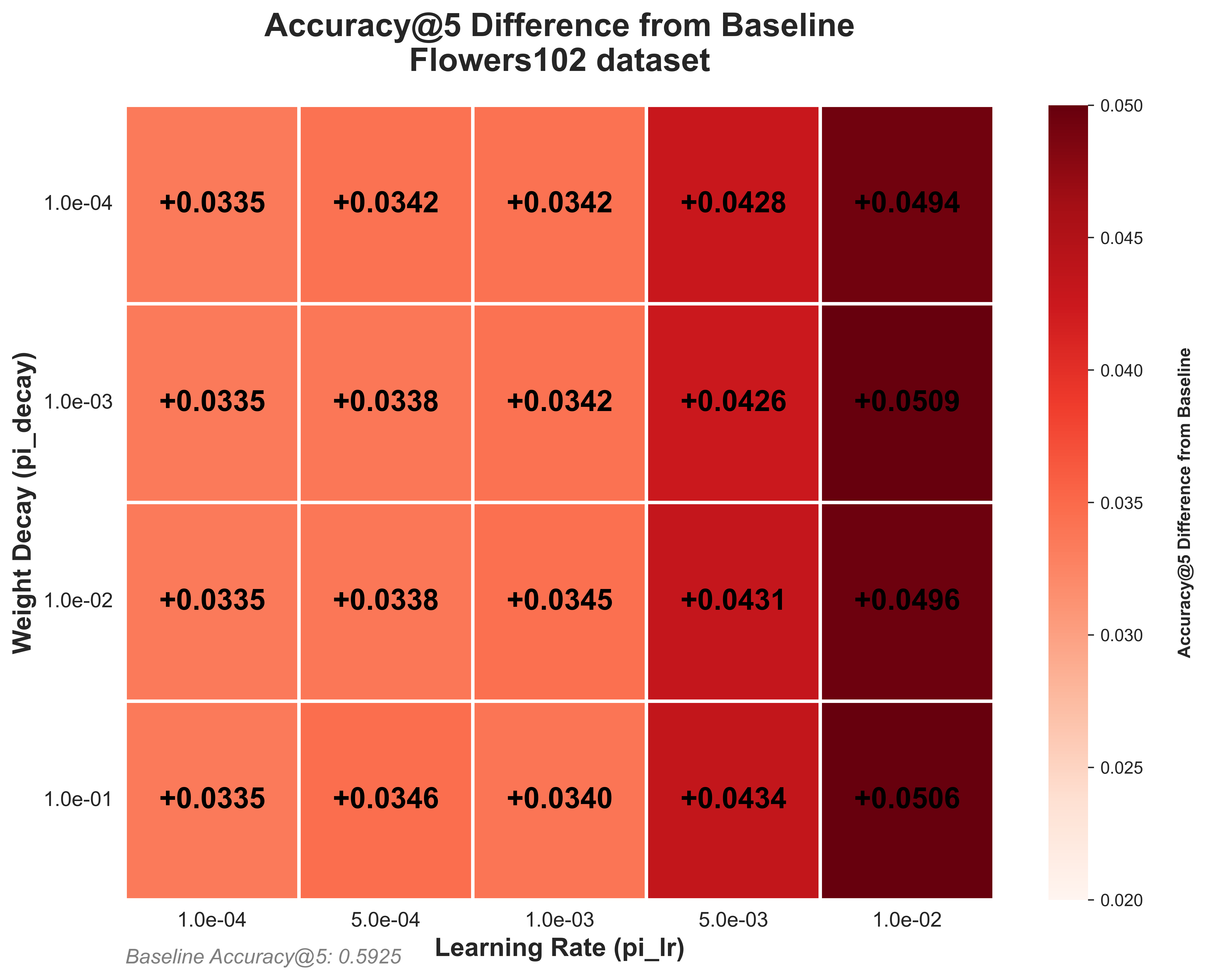}
    \includegraphics[width=0.45\linewidth]{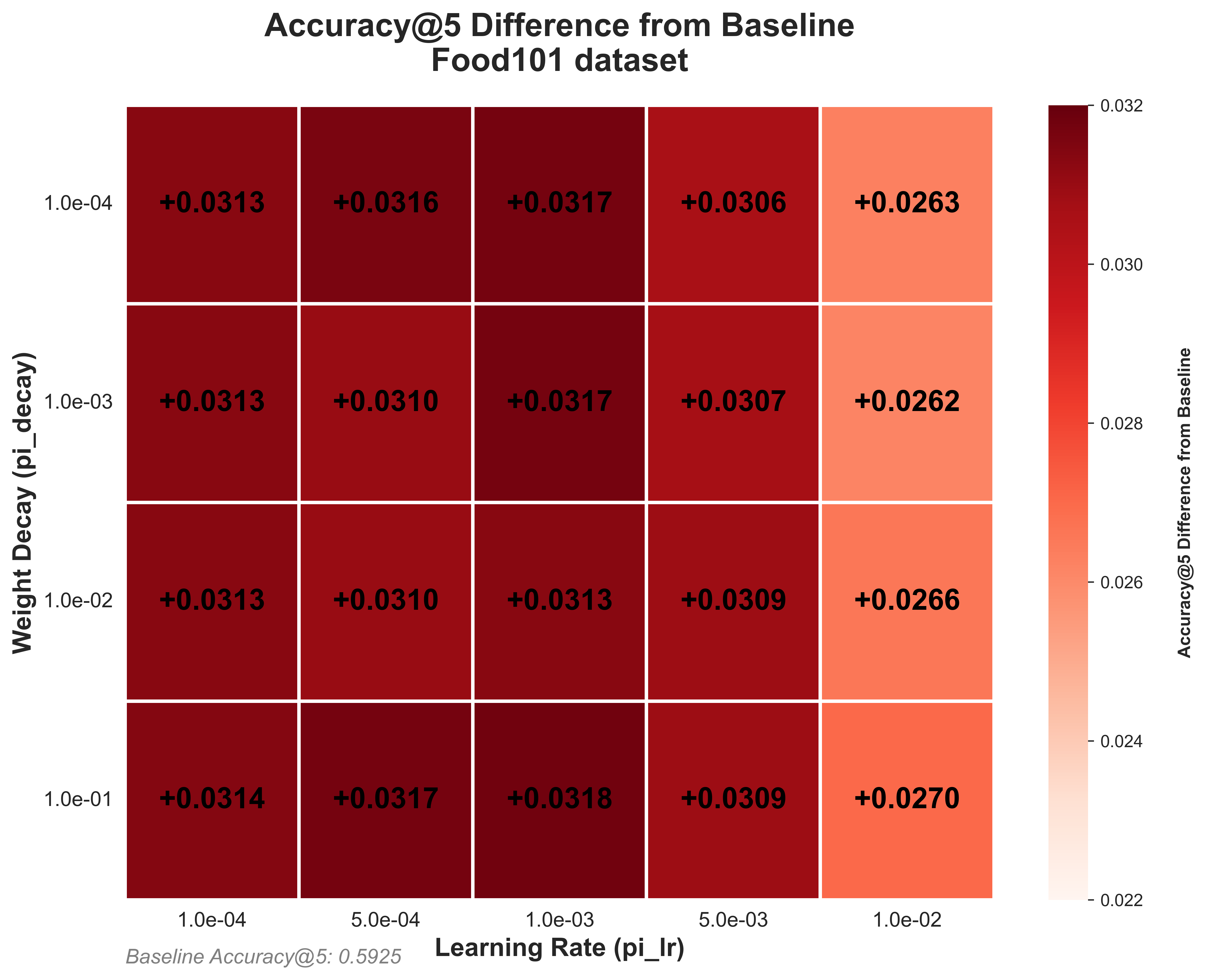}
    \caption{Robustness of \texttt{ALSO} to $\pi$-hyperparameters ($\Delta$F1-score vs. \texttt{AdamW} baseline). Each cell shows the F1-score difference between \texttt{ALSO} and \texttt{AdamW} with static weights (baseline), over a full 2D grid of $\pi$-learning rate ($\gamma_\pi$) and $\pi$-regularization ($\lambda$). All cells are red (positive $\Delta$F1), indicating that ALSO consistently outperforms the baseline across the entire grid and for both datasets from Split Learning section.}
    \label{fig:split_rob}
\end{figure}

\subsection{Design choices}
\label{appendix:design_choices}

This section presents an empirical evaluation of key design choices in the proposed algorithm, focusing on the optimistic step and the non-adaptive update rule for the parameter $\pi$. We compare the performance of three algorithm variants:

\begin{enumerate}
\item Vanilla \texttt{ALSO}: The standard implementation of the proposed algorithm (Algorithm \ref{algorithm:also_optimistic}).

\item Descent-Ascent \texttt{ALSO} ($\alpha = 0$): A variant where the optimistic step is removed by setting the optimistic coefficient $\alpha$ to zero.

\item \texttt{A$^\pi$LSO}: A modified version of \texttt{ALSO} that employs the Adam optimizer for updating the weight vector $\pi$.
\end{enumerate}

The algorithms were evaluated across three distinct experimental settings: Learning from Unbalanced Data (Section \ref{sec:unbalanced_data}), Tabular Deep Learning (Section \ref{sec:tabdl}), and Split Learning (Section \ref{sec:split_learning}). The results are summarized in Figure \ref{fig:unbalanced_cifar_appendix}, Table \ref{table:design-choices-tabdl} and Figure \ref{fig:split_learning_appendix}.

The Descent-Ascent variant has a significantly lower performance compared to the other two algorithms, indicating the importance of the optimistic step. The \texttt{A$^\pi$LSO} algorithm achieves comparable performance to vanilla \texttt{ALSO} in some scenarios (Table \ref{table:design-choices-tabdl}, Figure \ref{fig:split_learning_appendix}). However, in the Unbalanced Data experiment, \texttt{A$^\pi$LSO} demonstrates degraded performance when the unbalanced coefficient is large ($\geq 10$).

Considering both performance and ease of implementation, we recommend vanilla \texttt{ALSO} as a robust baseline. While \texttt{A$^\pi$LSO} can provide competitive results in certain settings, it introduces additional hyperparameters and computational overhead associated with the Adam optimizer for $\pi$. Therefore, \texttt{A$^\pi$LSO} may be considered when sufficient computational resources are available for hyperparameter tuning and multiple experimental runs.

\begin{figure}[!h]
    \centering
    \includegraphics[width=0.45\linewidth]{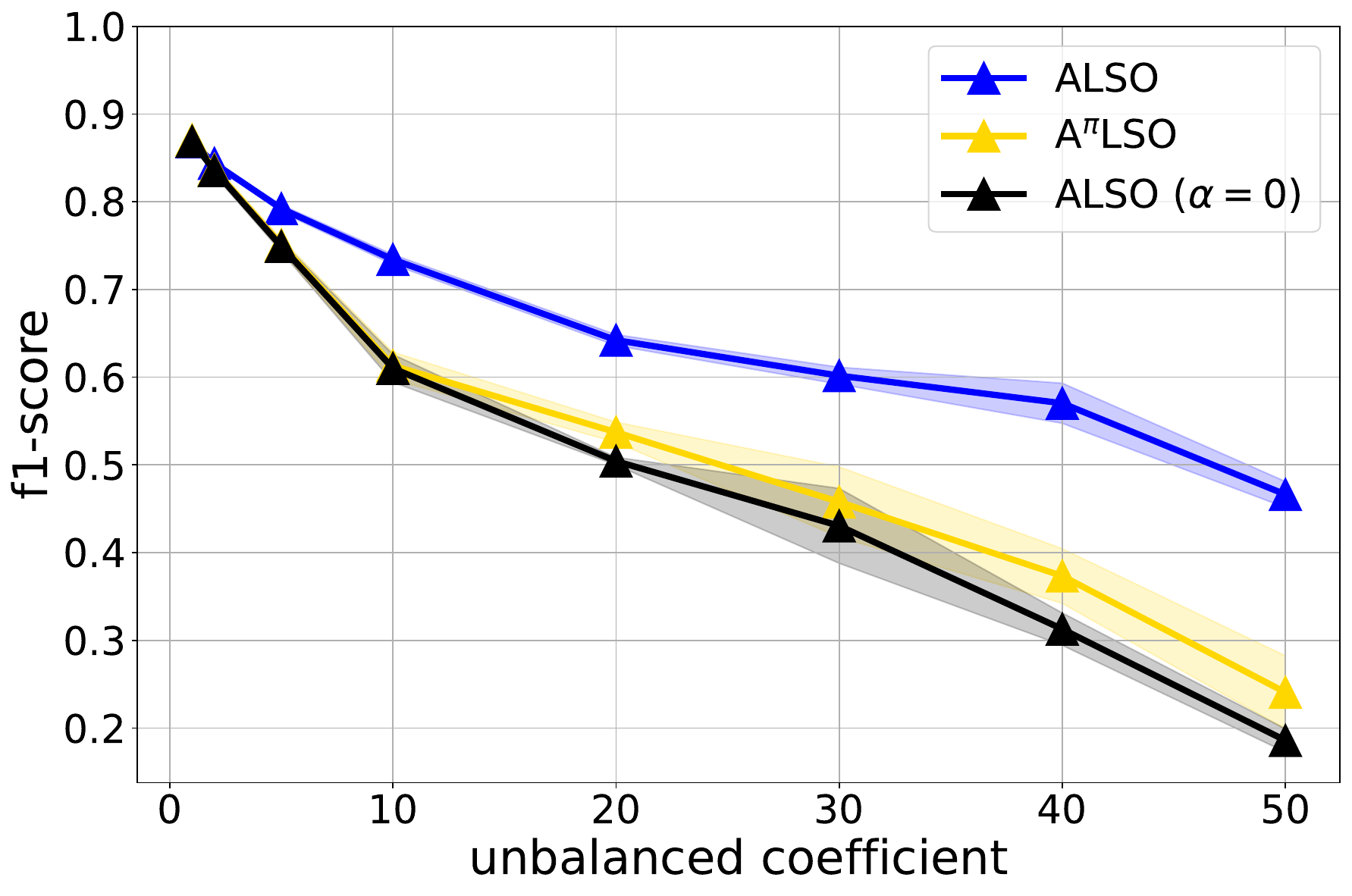}
    \caption{Performance comparison of \texttt{ALSO}, \texttt{ALSO} with $\alpha=0$ (descent-ascent), and \texttt{A$^\pi$LSO} (adaptive step over $\pi$) on the unbalanced CIFAR experiment from Section \ref{sec:unbalanced_data}. Hyperparameter tuning is performed in the same manner as in the main experiment.}
    \label{fig:unbalanced_cifar_appendix}
\end{figure}

\begin{figure}[h!]
    \centering
    \begin{minipage}{0.45\linewidth}
        \centering
        \includegraphics[width=\linewidth]{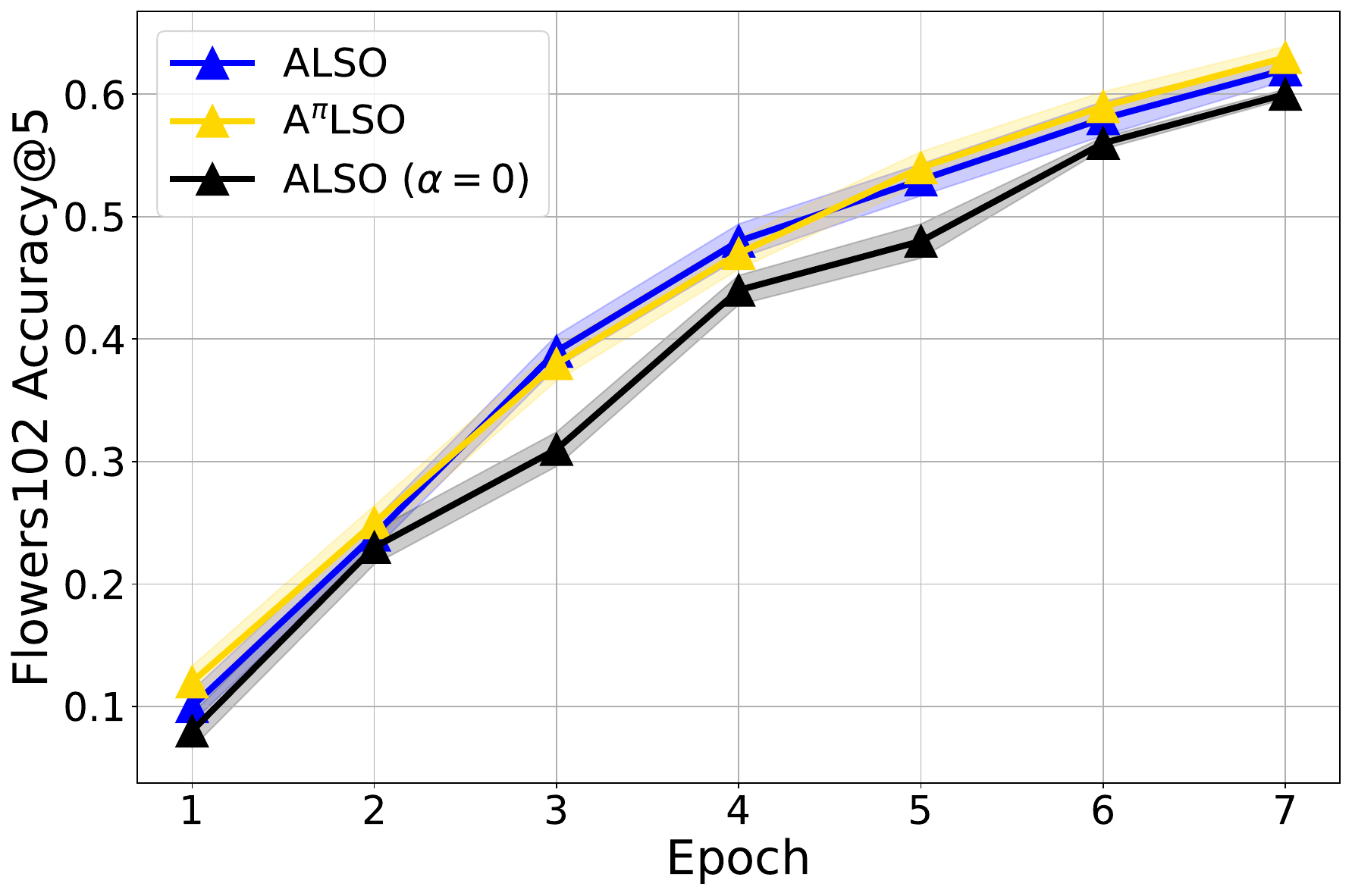}
    \end{minipage}
    \begin{minipage}{0.45\linewidth}
        \centering
        \includegraphics[width=\linewidth]{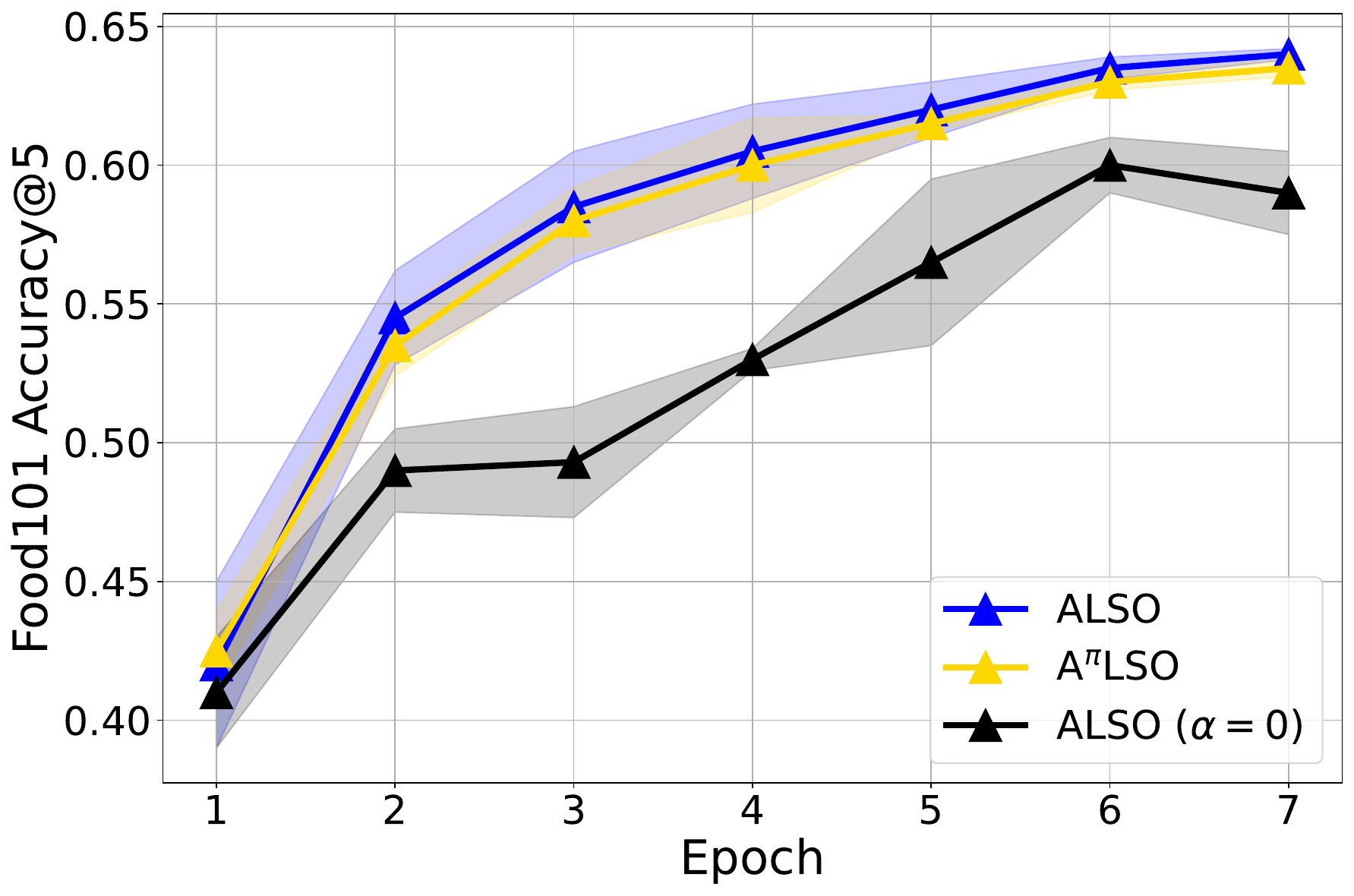}
    \end{minipage}
    \caption{Performance comparison of \texttt{ALSO}, \texttt{ALSO} with $\alpha=0$ (descent-ascent), and \texttt{A$^\pi$LSO} (adaptive step over $\pi$) on the Split Learning experiment from Section \ref{sec:split_learning}}
    \label{fig:split_learning_appendix}
\end{figure}

\begin{table*}[h!]
\centering
\resizebox{0.6\textwidth}{!}{
\begin{tabular}{@{}lcccccccc@{}}
\toprule
\textbf{Dataset} & \texttt{ALSO} & \texttt{ALSO} $\alpha=0$ & \texttt{A$^\pi$LSO} \\

\midrule Weather (RMSE $\downarrow$) & $\mathbf{1.4928 \pm 0.0042}$ & ${1.5209 \pm 0.0036}$ & $\underline{1.4967 \pm 0.0066}$ \\

\midrule Ecom Offers (ROC-AUC $\uparrow$) & $\mathbf{0.5976 \pm 0.0020}$ & $\underline{0.5975 \pm 0.0020}$ & ${0.5915 \pm 0.0087}$ \\

\midrule Cooking Time (RMSE $\downarrow$) & $\mathbf{0.4806 \pm 0.0003}$ & ${0.4810 \pm 0.0003}$ & $\mathbf{0.4806 \pm 0.0004}$  \\

\midrule Maps Routing (RMSE $\downarrow$) & $\underline{0.1612 \pm 0.0001}$ & ${0.1613} \pm {0.0002}$ & $\mathbf{0.1611 \pm 0.0001}$ \\

\midrule Homesite Insurance (ROC-AUC $\uparrow$) & $\mathbf{0.9632 \pm 0.0003}$ & $\underline{0.9630 \pm 0.0004}$ & ${0.9626 \pm 0.0003}$ \\

\midrule Delivery ETA (RMSE $\downarrow$) & $\underline{0.5513 \pm 0.0020}$ & ${0.5536 \pm 0.0030}$ & $\mathbf{0.5507 \pm 0.0011}$ \\

\midrule Homecredit Default (ROC-AUC $\uparrow$) & $\mathbf{0.8587 \pm 0.0012}$ & ${0.8587 \pm 0.0008}$ & $\mathbf{0.8587 \pm 0.0011}$ \\

\midrule Sberbank Housing (RMSE $\downarrow$) & $\underline{0.2424 \pm 0.0024}$ & ${0.2457 \pm 0.0044}$ & $\mathbf{0.2401 \pm 0.0073}$ \\

\midrule Black Friday (RMSE $\downarrow$) & $\underline{0.6842 \pm 0.0004}$ & $\underline{0.6843 \pm 0.0013}$ & $\mathbf{0.6838 \pm 0.0005}$ \\

\midrule Microsoft (RMSE $\downarrow$) & $\underline{0.7437 \pm 0.0003}$ & $\underline{0.7435 \pm 0.0003}$ & $\mathbf{0.7438 \pm 0.0003}$ \\

\midrule California Housing (RMSE $\downarrow$) & ${0.4495 \pm 0.0046}$ & ${0.4533 \pm 0.0043}$ & $\mathbf{0.4455 \pm 0.0032}$ \\
 
\midrule Churn Modeling (ROC-AUC $\uparrow$) & $\mathbf{0.8666 \pm 0.0027}$& ${0.8597 \pm 0.0076}$ & $\underline{0.8646 \pm 0.0019}$ \\

\midrule Adult (ROC-AUC $\uparrow$) & $\mathbf{0.8699 \pm 0.0001}$ & $\underline{0.8698 \pm 0.0002}$ & $\underline{0.8698 \pm 0.0014}$ \\
 
\midrule Higgs Small (ROC-AUC $\uparrow$) & $\underline{0.7280 \pm 0.0009}$  & $\underline{0.7279 \pm 0.0013}$ & $\mathbf{0.7288 \pm 0.0012}$ \\
\bottomrule
\end{tabular}
}
\caption{
Performance comparison of \texttt{ALSO}, \texttt{ALSO} $\alpha=0$ (descent-ascent) and \texttt{A$^\pi$LSO} (adaptive step over $\pi$). The trained model is MLP-PLR \citep{gorishniy2022embeddings}. Bold entries represent the best method on each dataset according to mean, underlined entries represent methods, which performance is best with standard deviations over 15 seeds taken into account. Metric is written near dataset name, $\uparrow$ means that higher values indicate better performance, $\downarrow$ means that lower values indicate better performance. Hyperparameter tuning is performed in the same manner as in the main experiment.
}
\label{table:design-choices-tabdl}
\end{table*}

\subsection{Tuning Comparison}
\label{appendix:ablation_hparams}

\begin{table*}[h!]
\centering
\resizebox{0.6\textwidth}{!}{
\begin{tabular}{@{}lcccccccc@{}}
\toprule
\textbf{Dataset} & \texttt{Adam} & \texttt{AdamW} & \texttt{ALSO} \\

\midrule Weather (RMSE $\downarrow$) & ${1.5199 \pm 0.0034}$ & ${1.5199 \pm 0.0034}$ & $\mathbf{1.4928 \pm 0.0042}$ \\

\midrule Ecom Offers (ROC-AUC $\uparrow$) & $\underline{0.5972 \pm 0.0020}$ & ${0.5717} \pm {0.0020}$ & $\mathbf{0.5976 \pm 0.0020}$ \\

\midrule Cooking Time (RMSE $\downarrow$) & ${0.4810 \pm 0.0005}$ & ${0.4810 \pm 0.0005}$ & $\mathbf{0.4806 \pm 0.0003}$  \\

\midrule Maps Routing (RMSE $\downarrow$) & ${0.1617} \pm {0.0002}$ & ${0.1625} \pm {0.0002}$ & $\mathbf{0.1612 \pm 0.0001}$ \\

\midrule Homesite Insurance (ROC-AUC $\uparrow$) & ${0.9614} \pm {0.0003}$ & ${0.9593} \pm {0.0005}$ & $\mathbf{0.9632 \pm 0.0003}$ \\

\midrule Delivery ETA (RMSE $\downarrow$) & ${0.5550 \pm 0.0021}$ & ${0.5544 \pm 0.0014}$ & $\mathbf{0.5513 \pm 0.0020}$ \\

\midrule Homecredit Default (ROC-AUC $\uparrow$) & $\underline{0.8581 \pm 0.0009}$ & $\underline{0.8581 \pm 0.0009}$ & $\mathbf{0.8585 \pm 0.0012}$ \\

\midrule Sberbank Housing (RMSE $\downarrow$) & ${0.2457 \pm 0.0046}$ & ${0.2455 \pm 0.0047}$ & $\mathbf{0.2424 \pm 0.0024}$ \\

\midrule Black Friday (RMSE $\downarrow$) & $\mathbf{0.6842 \pm 0.0006}$ & ${0.6869} \pm {0.0006}$ & $\mathbf{0.6842 \pm 0.0004}$ \\

\midrule Microsoft (RMSE $\downarrow$) & $\underline{0.7440 \pm 0.0002}$ & ${0.7442} \pm {0.0003}$ & $\mathbf{0.7437 \pm 0.0004}$ \\

\midrule California Housing (RMSE $\downarrow$) & ${0.4554 \pm 0.0034}$ & ${0.4734 \pm 0.0038}$ & $\mathbf{0.4495 \pm 0.0046}$ \\
 
\midrule Churn Modeling (ROC-AUC $\uparrow$) & ${0.8620 \pm 0.0075}$ & ${0.8618 \pm 0.0038}$ & $\mathbf{0.8666 \pm 0.0027}$ \\

\midrule Adult (ROC-AUC $\uparrow$) & ${0.8693 \pm 0.0010}$ & ${0.8689 \pm 0.0009}$ & $\mathbf{0.8699 \pm 0.0001}$ \\
 
\midrule Higgs Small (ROC-AUC $\uparrow$) & $\underline{0.7271 \pm 0.0013}$ & ${0.7248 \pm 0.0013}$ & $\mathbf{0.7280 \pm 0.0009}$ \\
\bottomrule
\end{tabular}
}
\caption{Performance comparison of \texttt{Adam}, \texttt{AdamW} and \texttt{ALSO} with hyperparameters found for \texttt{ALSO}. The trained model is MLP-PLR \citep{gorishniy2022embeddings}. Bold entries represent the best method on each dataset according to mean, underlined entries represent methods, which performance is best with standard deviations over 15 seeds taken into account. Metric is written near dataset name, $\uparrow$ means that higher values indicate better performance, $\downarrow$ means that lower values indicate better performance.}
\label{table:tabdl-ablation}
\end{table*}

We evaluate \texttt{Adam} and \texttt{AdamW} with hyperparameters for \texttt{ALSO} to isolate effect of dynamic weights usage. The results are presented in Table \ref{table:tabdl-ablation}. As we can see, the choice of hyperparameters, does not explain, why \texttt{ALSO} outperforms \texttt{Adam}.

\subsection{Weights Analysis}
\label{appendix:weights_changig}

Here we perform analysis of $\pi$ vector behavior. In Figure \ref{fig:pi_change} we can see that weights are changing during training process. For some tasks weights converge, while for other they are still changing. This effect can be explained that we use early stopping or stop training before converges. 

More interesting is comparison of values of default loss and weighted. As we can see in Figure \ref{fig:loss_comparision} weighted loss increases losses on some batches, and decreases on other. It means that intuition behind \eqref{eq:adv_pi_problem} is probably the same as we propose in Section \ref{sec:intro}.
\newpage
\begin{figure}[h!]
    \centering
    \includegraphics[width=1.0\linewidth]{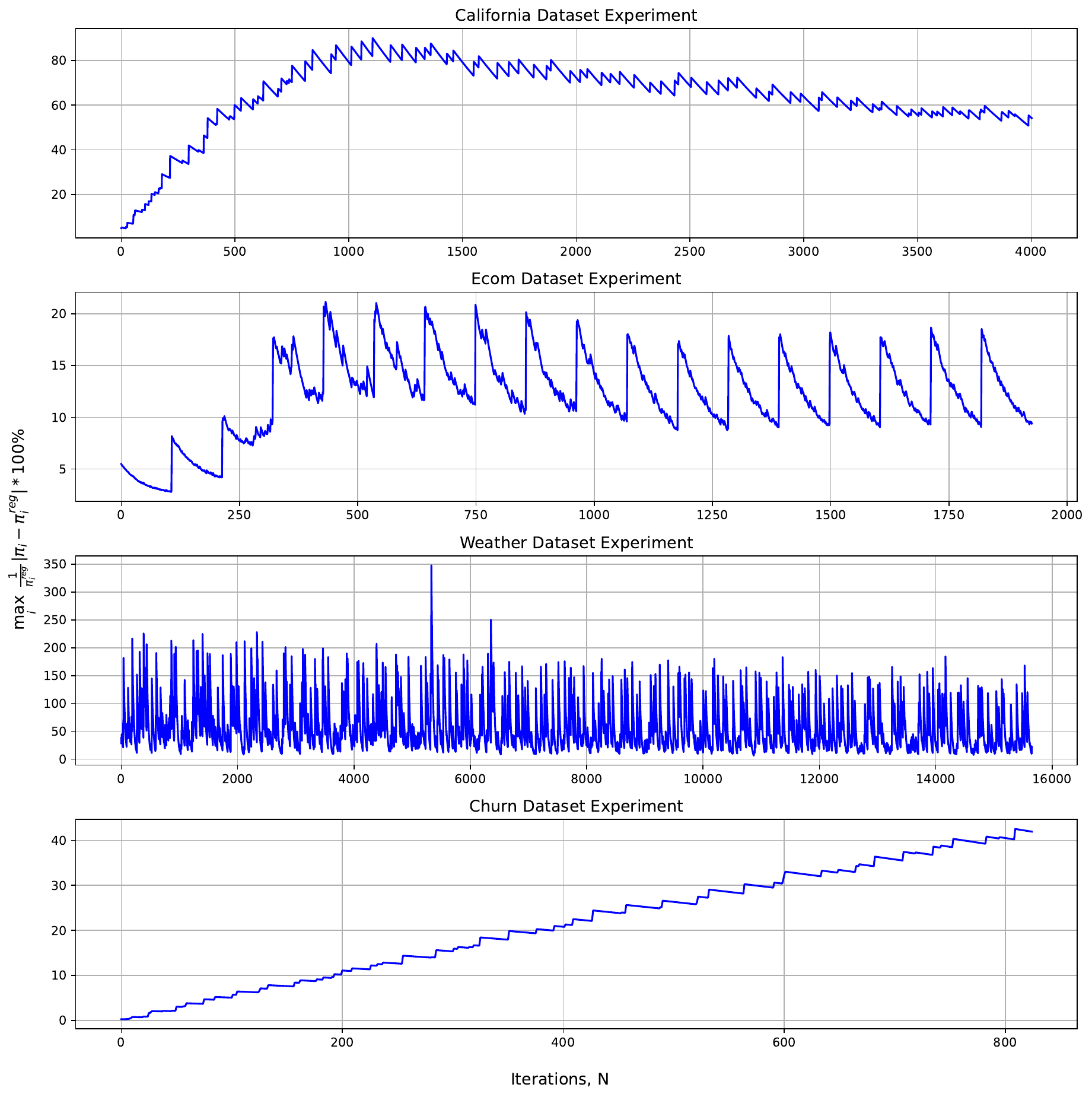}
    \caption{
    Maximum percentage difference between $\pi$ and $\hat{\pi}$ during training of several our experiments with \texttt{ALSO}.
    }
    \label{fig:pi_change}
\end{figure}

\newpage
\begin{figure}[h!]
    \centering
    \includegraphics[width=1.0\linewidth]{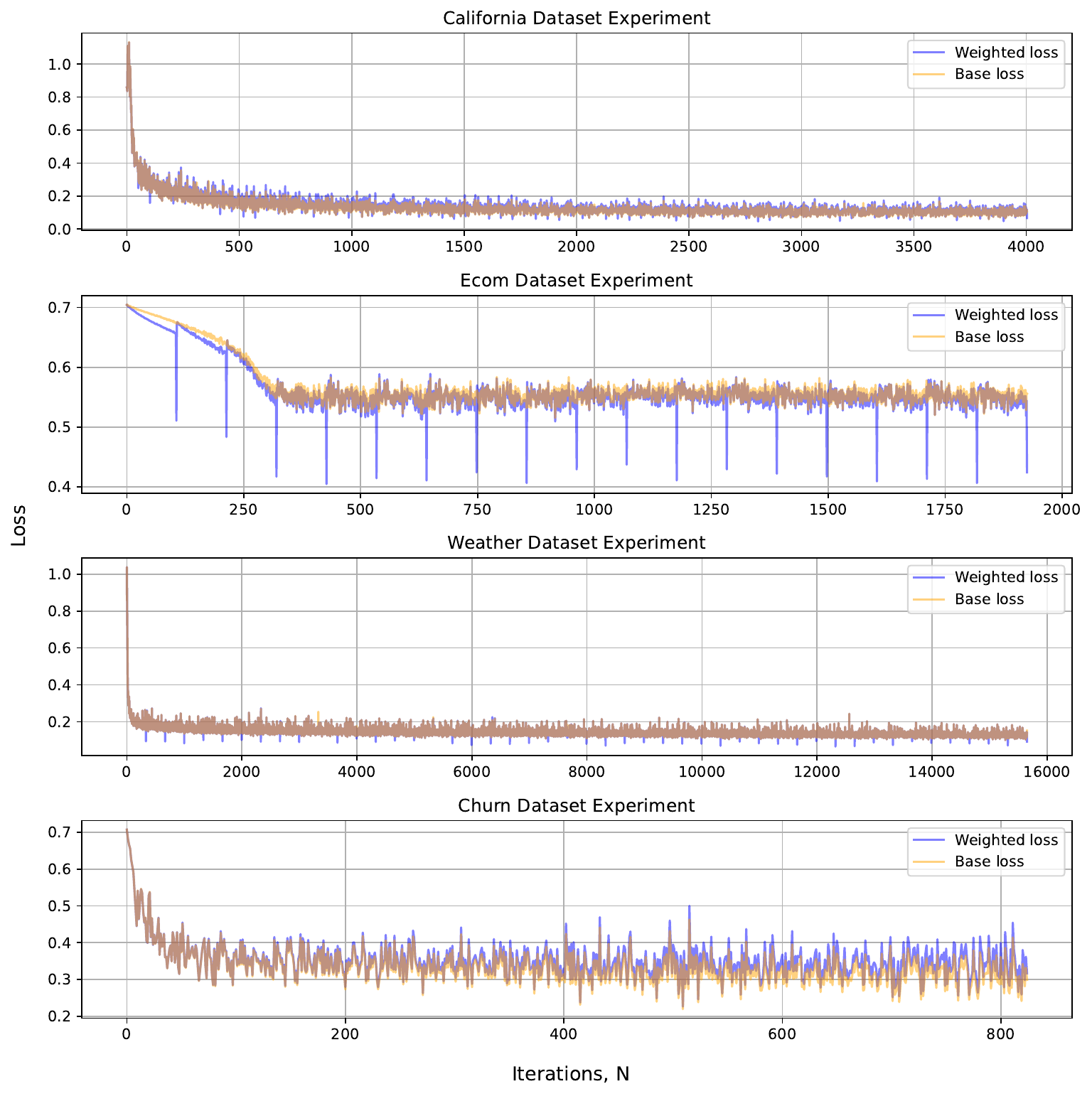}
    \caption{
    Comparison of weighted loss and non-weighted loss during f several our experiments with \texttt{ALSO}. At each iteration we report base loss on batch and weighted loss on batch.
    }
    \label{fig:loss_comparision}
\end{figure}
\newpage
\section{Extension of Section \ref{sec:optimistic_mirror_prox}}
\subsection{Variational Inequalities}
It is widely accepted in the modern literature to study the saddle point problem \eqref{eq:min_max}, and, correspondingly, the main problem of the paper \eqref{eq:adv_pi_problem}, within the more general framework of Variational Inequalities (VI) \citep{stampacchia1964formes, beznosikov2020distributed, mokhtari2020unified, hsieh2020explore, gorbunov2022stochastic}
with non-smooth regularization added, since we use the KL divergence in \eqref{eq:adv_pi_problem}. The task is to find $z^* \in \mathcal{Z}$ such that for all $z \in \cZ$ it hols that:
\begin{equation}
\begin{split}
        \label{eq:vi_problem}
        &\dotprod{F(z^*)}{z - z^*} + \tau V(z, \hat{z}) - \tau V(z^*, \hat{z}) \geq 0,
\end{split}
\end{equation}
where  $\mathcal{Z}$ is some convex vector space and $F: \mathcal{Z} \to \mathbb{R}^d$ is an operator. The function $V(z_1, z_2)$ represents a Bregman divergence, which serves as a non-smooth regularizer (for example, the KL divergence in problem \eqref{eq:adv_pi_problem}). We now provide the formal definition of $V(z_1, z_2)$. Let $\omega(\cdot)$ be a proper differentiable and $1$-strongly convex function with respect to $\| \cdot \|$ on $\mathcal{Z}$. Then for any $z_1, z_2 \in \mathcal{Z}$ we can define the Bregman divergence as 
\begin{equation}
\label{eq:bregman}
    V(z_1, z_2) := \omega(z_1) - \omega(z_2) - \dotprod{\nabla \omega(z_2)}{z_1 - z_2} .
\end{equation}
The definition \eqref{eq:bregman} is a generalization of the concept of norm for arbitrary convex sets. For example, the KL divergence from the equation \eqref{eq:adv_pi_problem} is a special case of the Bregman divergence on the simplex $\Delta_{n-1}$ with $\| \cdot \| = \| \cdot \|_1$ and a generating negative entropy function of the form $\omega_{\text{KL}}(u) = \sum_{j=1}^n u_j \log(u_j)$. 

In order to proceed from the problem \eqref{eq:vi_problem} to \eqref{eq:min_max}, one should set $z = [x, y]^T$ and $F(z) = [\nabla_x g(x, y), -\nabla_y g(x, y)]^T$. It is common to consider methods for solving saddle-point problems \eqref{eq:min_max} together with the solution of VI \eqref{eq:vi_problem}.

\textbf{Application of Variational Inequalities.} Although VI were inspired by min-max problem \eqref{eq:min_max}, the formulation \eqref{eq:vi_problem} has further numerous significant special cases, such as the classical minimization \citep{nesterov2005smooth} and fixed point \citep{reich1983some, taiwo2021inertial} problems. The setting \eqref{eq:vi_problem} is applied in classical disciplines such as game theory, economics, equilibrium theory and convex analysis \citep{stampacchia1964formes, browder1966existence, rockafellar1969convex, sibony1970methodes, luenberger1984linear}. However, formulation \eqref{eq:vi_problem} has gained the most popularity with the rise of machine learning and artificial intelligence models. VI problem arises in the GAN optimization \citep{arjovsky2017wasserstein, goodfellow2020generative, aggarwal2021generative}, in the reinforcement learning \citep{omidshafiei2017deep, jin2020efficiently} and adversarial training \citep{ben2009robust, madry2017towards}.
, sparse matrix factorizations \citep{bach2008convex}, unsupervised learning \citep{esser2010general, chambolle2011first}, non-smooth optimization \citep{nesterov2005smooth} and discriminative clustering \citep{joachims2005support}.

Now we connect our problem \eqref{eq:adv_pi_problem} to \eqref{eq:vi_problem}. For simplicity, let $n=c$, $U = \Delta_{n-1}$, $n_i=1$, $f_{i, 1} := f_i$. Then:
\begin{proposition}
\label{proposition:pi_to_vi}
    The formulation \eqref{eq:adv_pi_problem} is a special case of the VI problem \eqref{eq:vi_problem} with
    \begin{equation*}
    \label{eq:pi_to_vi}
    \begin{split}
        &z := [\theta, \pi]^T, ~\hat{z} := [\textbf{0}, \hat{\pi}]^T,~ \mathcal{Z} = \mathbb{R}^d \times \Delta_{n-1}, 
        \\&V(z_1, z_2) := \frac{1}{2}\| \theta^1 - \theta^2 \|^2_2 + \text{KL} \left[\mathbf{\pi}^1 ~\|~ \mathbf{\pi}^2\right],
        \\&F(z) := \left[\sum_{i=1}^n \pi_i \nabla {f}_i(\theta), \; -{f}_1(\theta), ..., -{f}_n(\theta)\right]^T .
    \end{split}
    \end{equation*}
\end{proposition}

We now introduce the common assumptions required for the analysis of solving \eqref{eq:vi_problem}.

\begin{assumption}
    \label{as:lipvar}
    The operator $F$ is $L_F$-Lipschitz continuous on $\mathcal{Z}$, i.e., for any $z_1, z_2 \in \mathcal{Z}$ the following inequality holds 
    \begin{equation*}
        \|F(z_1) - F(z_2)\|_*\leq L_F \|z_1 - z_2\|,
    \end{equation*}
    where $\| \cdot \|$ is the norm with respect to which the generating function $\omega(\cdot)$ of the Bregman divergence $V(\cdot, \cdot)$ form the problem \eqref{eq:vi_problem} is $1$-strongly convex.
\end{assumption}

\begin{assumption}
    \label{as:monotone}
    The operator $F$ is monotone on $\mathcal{Z}$, i.e., for all $z_1, z_2 \in \mathcal{Z}$ the following inequality holds 
    \begin{equation*}
        \langle F(z_1) - F(z_2), z_1 - z_2\rangle \geq 0.
    \end{equation*}
\end{assumption}
Assumptions \ref{as:lipvar} and \ref{as:monotone} are classical in the analysis of the problem \eqref{eq:vi_problem} in the deterministic case \citep{korpelevich1976extragradient, gidel2018variational, tseng1995linear, hsieh2019convergence, mokhtari2020unified}.

\subsection{Optimistic Mirror-Prox}

This section introduces an optimistic Mirror-Prox algorithm \cite{popov1980modification} designed to solve problem \eqref{eq:vi_problem}. We derive a convergence rate for this algorithm and then establish its relationship to Algorithm \ref{algorithm:optimistic_mirror_prox}, which is employed for solving the saddle point problem \eqref{eq:min_max}. In this section we will use $f_{i,j}$ and $f_i$ as synonyms, since $j$ is always equal to $1$, because of $c=n$ in Algorithm \ref{algorithm:optimistic_mirror_prox}.

\begin{algorithm}{Optimistic Mirror-Prox}
   \label{algorithm:optimistic_mirror_prox_app}
\begin{algorithmic}[1]
    \State {\bf Parameters:} stepsize $\gamma$, momentum $\alpha$, number of iterations $N$.
    \State {\bf Initialization:} choose  $z^{-1} = z^{0} \in \mathcal{Z}$.
    \For{$k = 0, 1, 2, \dots, N$}
        \State $g^k = (1 + \alpha) F(z^k) - \alpha F(z^{k-1})$
        \State $z^{k+1}=\arg\!\min\limits_{z \in \cZ} \{\langle \gamma g^k, z \rangle+V(z, z^k)+\gamma \tau V(z, \hat{z})\}$ 
    \EndFor
\end{algorithmic}
\end{algorithm}

We now provide proof of the convergence rate of Algorithm \ref{algorithm:optimistic_mirror_prox_app}.

\begin{theorem}
\label{theorem:optimistic_mirror_prox_app}
    Let Assumptions \ref{as:lipvar} and \ref{as:monotone} be satisfied. Let the problem \eqref{eq:vi_problem} be solved by Algorithm \ref{algorithm:optimistic_mirror_prox_app}. Assume that the stepsize $\gamma$ is chosen such that $0 <\gamma \leq 1/(2L_F)$ and momentum $\alpha$ is chosen such that $\alpha = (1 + \gamma \tau)^{-1}$. Then, for all $k \geq 1$ it holds that 
    \begin{equation*}
        V(z^*, z^{k})
        =
        \mathcal{O} \left[ \left(1 - \frac{\gamma \tau}{2}\right)^{k} V(z^*, z^0) \right].
    \end{equation*}
    where $z^*$ is the solution of the problem \eqref{eq:vi_problem}. In other words, if one takes $\gamma = 1 / (2 L_F)$, then to achieve $\varepsilon$-accuracy (in terms of $V(z^*, z^N) \leq \varepsilon$) one would need at most 
    \begin{equation*}
        \mathcal{O} \left[ \frac{L_F}{\tau} \cdot \log\left( \frac{V(z^*, z^0)}{\varepsilon} \right) \right] ~\text{ iterations of Algorithm \ref{algorithm:optimistic_mirror_prox}.}
    \end{equation*}
\end{theorem}

Full proof of Theorem \ref{theorem:optimistic_mirror_prox_app} is provided in next Section \ref{appendix:optimistic_mirror_prox}.

In the main part of the paper (Section \ref{sec:optimistic_mirror_prox}), we used Algorithm \ref{algorithm:optimistic_mirror_prox_app} directly for the min-max problem \eqref{eq:min_max}, skipping the complicated definitions about Variational Inequalities. Therefore, we need to transition from Algorithm \ref{algorithm:optimistic_mirror_prox_app} to Algorithm \ref{algorithm:optimistic_mirror_prox}, and correspondingly, from Assumptions \ref{as:lipvar} and \ref{as:monotone} to Assumptions \ref{as:lipgrad} and \ref{as:convex}. The next two propositions show how we can do this.

\begin{proposition}
\label{proposition:pi_to_vi_ass}
    Let Assumptions \ref{as:lipgrad} and \ref{as:convex} be satisfied. Then the target operator $F(\cdot)$ for the problem \eqref{eq:adv_pi_problem} from Proposition \ref{eq:pi_to_vi} fits under Assumptions \ref{as:lipvar} and \ref{as:monotone} with 
    \begin{equation*}
        L_F^2 = \mathcal{O} \left[ \max_{i \in \overline{1, n}}\{L_i^2\} + \max_{i \in \overline{1, n}}\left\{ K_i^2 \right\}  \right].
    \end{equation*}
\end{proposition}

\begin{proposition}
\label{theorem:also_step}
    Consider the problem \eqref{eq:adv_pi_problem} and the step of Mirror-Prox like algorithm for solving it:
    $$z^{\text{new}} = \arg\min_{z \in \cZ} \left\{ \langle \gamma g, z \rangle + V(z, z^{\text{old}}) + \gamma \tau V(z, \hat{z})\right\}$$
    where $\hat{z} = (0, \hat{\pi})$ and $g = (g^\theta, g^\pi)$ is the target function from \eqref{eq:adv_pi_problem}. Then the update rule is:
    \begin{equation*}
    \begin{split}
        &\theta^{\text{new}} = \theta^{\text{old}} - \frac{\gamma}{1 + \gamma\tau}(g^\theta + \tau \theta^{\text{old}}),
        \\&\pi^{\text{new}} = \text{SM} \left[\log \pi^{\text{old}} - \frac{\gamma}{1+\gamma\tau}\left(g^\pi + \tau \log \frac{\pi^{\text{old}}}{\hat{\pi}}\right)\right],
    \end{split}
    \end{equation*}
    where $\text{SM}$ denotes softmax function.
\end{proposition}

Full proof of Propositions \ref{proposition:pi_to_vi_ass} and \ref{theorem:also_step} is given in Sections \ref{appendix:pi_to_vi_ass} and \ref{appendix:also_step}.

Combining the results from Theorem \ref{theorem:optimistic_mirror_prox_app} and Propositions \ref{proposition:pi_to_vi_ass} and \ref{theorem:also_step} directly yields the convergence rate of Algorithm \ref{algorithm:optimistic_mirror_prox} (Theorem \ref{theorem:optimistic_mirror_prox} in Section \ref{sec:optimistic_mirror_prox}).

\subsection{Proof of the convergence rate of Optimistic-Mirror Prox (Theorem \ref{theorem:optimistic_mirror_prox_app})}
\label{appendix:optimistic_mirror_prox}

In the proof of Theorem \ref{theorem:optimistic_mirror_prox} we use technical lemma.

\begin{lemma}[Bregman divergence properties]
\label{lemma:bregman_step}
        For any Bregman divergence $V$ on the set $\mathcal{Z}$, for any $u \in \mathcal{Z}^*$, $z^1, \hat{z} \in \mathcal{Z}$ and $c \in \mathbb{R}$, if we define 
        \begin{equation}
        \label{eq:reg_mirror_prox_algo_tmp}
        \begin{split}
            z^{\dag} := \arg\min_{z \in \cZ} \left\{ \langle u, z \rangle + V(z, z^1) + c V(z, \hat{z})\right\}.
        \end{split}
        \end{equation}
        Then, for all $z \in \mathcal{Z}$ it holds that 
        \begin{equation*}
            (1+c) V(z, z^\dag) \leq V(z, z^1) - V(z^\dag, z^1) 
            -
            \langle u, z^\dag - z \rangle
            +
            c V(z, \hat{z}) - c V(z^\dag, \hat{z}) .
        \end{equation*}
\end{lemma}

\begin{proof}
    Using optimality condition in the equation \eqref{eq:reg_mirror_prox_algo_tmp} we obtain that for all $z \in \mathcal{Z}$ it holds that:
    \begin{equation*}
        \langle u + \nabla \omega(z^\dag) - \nabla \omega(z^1) + c \nabla \omega(z^\dag) - c \nabla \omega(\hat{z}), z^\dag - z \rangle \leq 0
    \end{equation*}
    Using the Law of cosines of the Bregman divergence  we can obtain that:
    \begin{equation*}
        \langle \nabla \omega(z_1) - \nabla \omega(z_2), z_1 - z_3 \rangle = V(z_3, z_1) + V(z_1, z_2) - V(z_3, z_2),
    \end{equation*}
    we can obtain that for all $z \in \mathcal{Z}$ it holds that:
    \begin{equation*}
        \langle u, z^\dag - z \rangle + V(z, z^\dag) + V(z^\dag, z^1) 
        -
        V(z, z^1) + c V(z, z^\dag) + c V(z^\dag, \hat{z}) - cV(z, \hat{z}) \leq 0
    \end{equation*}
    Re-arranging last inequality we obtain for all $z \in \mathcal{Z}$:
    \begin{equation*}
        (1 + c) V(z, z^\dag) 
        \leq 
        V(z, z^1) - V(z^\dag, z^1)
        - \langle u, z^\dag - z \rangle
        +
        c V(z, \hat{z}) - c V(z^\dag, \hat{z}).
    \end{equation*}
    This finishes the proof.
\end{proof}

\begin{proof}[Proof of Theorem \ref{theorem:optimistic_mirror_prox}]
    Using Lemma \ref{lemma:bregman_step} with $u = \gamma [ (1 + \alpha) F(z^k) - \alpha F(z^{k-1})], z^1 = z^k$ and $c = \gamma \tau V(z, \hat{z})$, we can obtain that for all $z \in \mathcal{Z}$ it holds that 
    \begin{equation}
    \label{eq:optimistic_tmp_1}
    \begin{split}
          \left(1 + \gamma \tau \right)V(z, z^{k+1}) 
          &\leq 
          V(z, z^k) - V(z^{k+1}, z^k) 
          + \gamma\tau V(z, \hat{z}) - \gamma \tau V(z^{k+1}, \hat{z}) 
          \\& - \gamma \langle (1 + \alpha) F(z^k) - \alpha F(z^{k-1}), z^{k+1} - z \rangle .
     \end{split}
    \end{equation}

    Consider the dot product in \eqref{eq:optimistic_tmp_1}. By using straightforward algebra we can obtain that 
    \begin{equation*}
    \begin{split}
        &- \gamma \langle (1 + \alpha) F(z^k) - \alpha F(z^{k-1}), z^{k+1} - z \rangle
        =
        -\underbrace{\gamma \langle F(z^k) - F(z^{k+1}), z^{k+1} - z \rangle}_{\circledOne}
        \\&-\underbrace{\gamma \alpha \langle F(z^k) - F(z^{k-1}), z^{k} - z \rangle}_{\circledTwo}
        -\underbrace{\gamma \alpha \langle F(z^k) - F(z^{k-1}), z^{k+1} - z^{k} \rangle}_{\circledThree}
        \\&-\underbrace{\gamma \langle  F(z^{k+1}), z^{k+1} - z \rangle}_{\circledFour} .
    \end{split}
    \end{equation*}

    Consider $\circledThree$. Since Assumption \ref{as:lipvar} is fulfilled, we can obtain that
    \begin{equation}
    \label{eq:optimistic_tmp_circ3}
    \begin{split}
        -\gamma \alpha \langle F(z^k) - F(z^{k-1}), z^{k+1} - z^{k} \rangle
        &\leq 
        \gamma^2 L^2 \alpha^2 \| z^k - z^{k-1} \|^2 + \frac{1}{4} \|z^{k+1} - z^{k} \|^2
        \\&\leq
        2 \gamma^2 L^2 \alpha^2 V(z^k, z^{k-1}) + \frac{1}{2} V(z^{k+1}, z^{k}) .
    \end{split}
    \end{equation}

    Consider $\circledFour + \gamma\tau V(z, \hat{z}) - \gamma \tau V(z^{k+1}, \hat{z})$. By using Assumption \ref{as:monotone} and the definition of the solution $z^* \in \mathcal{Z}$ of the problem \eqref{eq:vi_problem} we can obtain that 
    \begin{equation}
    \label{eq:optimistic_tmp_circ4}
    \begin{split}
        -\gamma \langle  F(z^{k+1}), z^{k+1} - z^* \rangle
        &+ \gamma\tau V(z, \hat{z}) - \gamma \tau V(z^{k+1}, \hat{z})
        =
        \\&-\gamma \langle  F(z^{k+1}) - F(z^*), z^{k+1} - z^* \rangle
        \\&-\gamma \langle  F(z^*), z^{k+1} - z^* \rangle
        + \gamma\tau V(z^*, \hat{z}) - \gamma \tau V(z^{k+1}, \hat{z})
        \\&\leq
        -\gamma \left[\langle  F(z^*), z^{k+1} - z^* \rangle
        - \tau V(z^*, \hat{z}) + \tau V(z^{k+1}, \hat{z}) \right]
        \leq
        0.
    \end{split}
    \end{equation}

    Consider $\circledTwo$. For the moment, we simply introduce the notation $a_k := -\circledTwo = -\gamma \alpha \langle F(z^k) - F(z^{k-1}), z^{k} - z \rangle$, and deal with it later in this proof. In this case $\circledOne$ is of the form $\circledOne = - \alpha^{-1} a_{k+1}$. Using this notation and the results of equations \eqref{eq:optimistic_tmp_circ3} and \eqref{eq:optimistic_tmp_circ4}, expression \eqref{eq:optimistic_tmp_1} takes the form
    \begin{equation*}
        \left(1 + \gamma \tau \right)V(z^*, z^{k+1}) 
        +
        \alpha^{-1} a_{k+1}
        \leq
        V(z^*, z^{k}) + a_k
        +
        2 \gamma^2 L^2 \alpha^2 V(z^k, z^{k-1}) - \frac{1}{2} V(z^{k+1}, z^{k}) .
    \end{equation*}

    For convenience, let us introduce another notation: $\Phi_k := V(z^*, z^{k}) + a_k$, also set $\alpha = (1 + \gamma \tau)^{-1}$, then we obtain result of the form
    \begin{equation*}
        \Phi_{k+1}
        \leq
        \alpha \Phi_k
        +
        \alpha \left[ 2 \gamma^2 L^2 \alpha^2 V(z^k, z^{k-1}) - \frac{1}{2} V(z^{k+1}, z^{k}) \right] .
    \end{equation*}

    We now start to roll-out the recursion from step $k$ to the step $k - m$:
    \begin{align}
        \Phi_{k+1}
        &\leq
        \alpha \Phi_k
        +
        \alpha \left[ 2 \gamma^2 L^2 \alpha^2 V(z^k, z^{k-1}) - \frac{1}{2} V(z^{k+1}, z^{k}) \right] \nonumber
        \\&\leq
        \alpha\left\{ 
            \alpha \Phi_{k-1}
            +
            \alpha \left[ 2 \gamma^2 L^2 \alpha^2 V(z^{k-1}, z^{k-2}) - \frac{1}{2} V(z^{k}, z^{k-1}) \right] 
        \right\} \nonumber
        \\&\quad~+
        \alpha \left[ 2 \gamma^2 L^2 \alpha^2 V(z^k, z^{k-1}) - \frac{1}{2} V(z^{k+1}, z^{k}) \right] \nonumber
        \\&\leq
        \alpha^2 \left\{             
            \alpha \Phi_{k-2}
            +
            \alpha \left[ 2 \gamma^2 L^2 \alpha^2 V(z^{k-2}, z^{k-3}) - \frac{1}{2} V(z^{k-1}, z^{k-2}) \right] 
        \right\} \nonumber
        \\&\quad~+
        \alpha^2 \left[ 2 \gamma^2 L^2 \alpha^2 V(z^{k-1}, z^{k-2}) - \frac{1}{2} V(z^{k}, z^{k-1}) \right]  \nonumber
        \\&\quad~+
        \alpha \left[ 2 \gamma^2 L^2 \alpha^2 V(z^k, z^{k-1}) - \frac{1}{2} V(z^{k+1}, z^{k}) \right] \nonumber
        \\&\dots \nonumber
        \\&\leq
        \alpha^{m+1} \Phi_{k - m} 
        -
        \sum_{j=0}^{m-1} \alpha^{j+2} \left( \frac{1}{2} - 2 \gamma^2 \alpha L^2 \right) V(z^{k-j}, z^{k-j-1}) \nonumber
        \\&- \frac{1}{2} \alpha V(z^{k+1}, z^k) 
        +
        2 \gamma^2 L^2 \alpha^{m+3} V(z^{k-m}, z^{k-m-1}). \label{eq:optimistic_tmp_2}
    \end{align}

    If we consider $\gamma \leq 1/(2 L)$, then $1/2 - 2 \gamma^2 \alpha L^2 \geq 1/2 - 2 \gamma^2 L^2 \geq 0$ and we can omit the sum in the equation \eqref{eq:optimistic_tmp_2}. Taking $m = k$ in \eqref{eq:optimistic_tmp_2} we obtain:
    \begin{equation*}
        \Phi_{k+1}
        \leq 
        \alpha^{k+1} \Phi_0
        - \frac{1}{2} \alpha V(z^{k+1}, z^k) 
        +
        2 \gamma^2 L^2 \alpha^{k+2} V(z^{0}, z^{-1}).
    \end{equation*}
    Since we initialize $z^{-1} = z^0$ in the Algorithm \ref{algorithm:optimistic_mirror_prox} we get $V(z^{0}, z^{-1})$. Now we return all the notations back and get:
    \begin{equation}
    \label{eq:optimistic_tmp_3}
    \begin{split}
        V(z^*, z^{k+1}) + \frac{1}{2} \alpha V(z^{k+1}, z^k)  -\gamma \alpha \langle F(z^{k+1}) - F(z^{k}), z^{k+1} - z^* \rangle
        \leq 
        &\alpha^k V(z^*, z^0) 
        .
    \end{split}
    \end{equation}

    By Using Fenchel-Young inequality we can obtain that: 
    \begin{equation}
    \label{eq:optimistic_tmp_4}
    \begin{split}
        V(z^*, z^{k+1}) + \frac{1}{2} \alpha V(z^{k+1}, z^k)
        &-\gamma \alpha \langle F(z^{k+1}) - F(z^{k}), z^{k+1} - z^* \rangle
        \geq
        V(z^*, z^{k+1})
        \\&- 
        \frac{1}{2} \alpha  V(z^*, z^{k+1})
        +
        \frac{1}{2} \alpha V(z^{k+1}, z^k)
        \\&-
        2 \gamma^2 L^2 \alpha V(z^{k+1}, z^k)
        \geq 
        \frac{1}{2}V(z^*, z^{k+1}) .
    \end{split}
    \end{equation}

    Combining \eqref{eq:optimistic_tmp_3} and \eqref{eq:optimistic_tmp_4} we can obtain that:
    \begin{equation*}
        V(z^*, z^{k+1}) \leq 
        2 \alpha^{k+1} V(z^*, z^0)
    \end{equation*}

    Subtracting $\alpha = (1+\gamma \tau)^{-1}$ 
    finishes the proof.
\end{proof}
\subsection{Proof of Proposition \ref{proposition:pi_to_vi_ass}}
\label{appendix:pi_to_vi_ass}

In the proof of Proposition \ref{proposition:pi_to_vi_ass} we use several technical lemmas.

    \begin{lemma}
    \label{lemma:sum_X_and_Y}
        If $V_\mathcal{X}$ and $V_\mathcal{Y}$ are Bregman divergences on normed vector spaces $(\mathcal{X}, \| \cdot \|_\mathcal{X})$ and $(\mathcal{Y}, \| \cdot \|_\mathcal{Y})$ respectively, then $V_\mathcal{Z}(\cdot) := V_\mathcal{X}(\cdot) + V_\mathcal{Y}(\cdot)$ is also a Bregman divergence on the normed vector space $ (\mathcal{Z} := \mathcal{X} \times \mathcal{Y}, \|\cdot \|_\mathcal{Z} := \sqrt{\| \cdot \|_\mathcal{X}^2 + \| \cdot \|_\mathcal{Y}^2})$ with generating function $\omega_\mathcal{Z}(\cdot) = \omega_\mathcal{X}(\cdot) + \omega_\mathcal{Y}(\cdot)$.
        Moreover, for conjugate norm $\|\cdot \|_\mathcal{Z^*}$ it holds that $\|\cdot \|_{\mathcal{Z}^*}^2 \leq 2 \|\cdot \|_{\mathcal{X}^*}^2 + 2 \|\cdot \|_{\mathcal{Y}^*}^2$.
    \end{lemma}
    \begin{proof}
        Let us prove the first part of Lemma \ref{lemma:sum_X_and_Y}. Since $\omega_\mathcal{X}(\cdot)$ and $\omega_\mathcal{Y}(\cdot)$ are 1-strongly convex on $(\mathcal{X}, \| \cdot \|_\mathcal{X})$ and $(\mathcal{Y}, \| \cdot \|_\mathcal{Y})$ respectively, for all $x_1, x_2 \in \mathcal{X}$ and $y_1, y_2 \in \mathcal{Y}$ it holds that
        \begin{align}
            &\omega_\mathcal{X}(x_2) \geq \omega_\mathcal{X}(x_1) + \langle \nabla \omega_\mathcal{X} (x_1), x_2 - x_1 \rangle + \frac{1}{2} \| x_1 - x_2 \|_\mathcal{X}^2, \label{eq:tmp_X_Y_1}
            \\&\omega_\mathcal{Y}(y_2) \geq \omega_\mathcal{Y}(y_1) + \langle \nabla \omega_\mathcal{Y} (y_1), y_2 - y_1 \rangle + \frac{1}{2} \| y_1 - y_2 \|_\mathcal{Y}^2. \label{eq:tmp_X_Y_2}
        \end{align}

        Now consider $\mathcal{Z} := \mathcal{X} \times \mathcal{Y}$, $\omega_\mathcal{Z}(z = (x, y)^T) = \omega_\mathcal{X}(x) + \omega_\mathcal{Y}(y)$, $\|\cdot \|_\mathcal{Z}^2 := \| \cdot \|_\mathcal{X}^2 + \| \cdot \|_\mathcal{Y}^2$ and $z_1 := (x_1, y_1)^T, z_2 := (x_2, y_2)^T \in \mathcal{Z}$. Summing up \eqref{eq:tmp_X_Y_1} and \eqref{eq:tmp_X_Y_2} we obtain that
        \begin{equation}
        \label{eq:tmp_X_Y_3}
            \omega_\mathcal{Z}(z_2) \geq \omega_\mathcal{Z}(z_1) + \langle \nabla \omega_\mathcal{Z} (z_1), z_2 - z_1 \rangle + \frac{1}{2} \| z_1 - z_2 \|_\mathcal{Z}^2,
        \end{equation}
        since $\nabla_z \omega_\mathcal{Z}(z) = (\nabla_x \omega_\mathcal{X}(x), \nabla_y \omega_\mathcal{Y}(y))^T$. Inequality \eqref{eq:tmp_X_Y_3} means that $\omega_\mathcal{Z}(\cdot)$ is 1-strongly convex on $(\mathcal{Z}, \| \cdot \|_\mathcal{Z})$ by definition.

        Function $\omega_\mathcal{Z}(\cdot)$ generates Bregman divergence of the form 
        \begin{equation*}
        \begin{split}
            V_\mathcal{Z}(z_1, z_2) 
            &= 
            \omega_\mathcal{Z}(z_1) - \omega_\mathcal{Z}(z_2) - \dotprod{\nabla_z \omega_\mathcal{Z}(z_2)}{z_1 - z_2}
            \\&=
            \omega_\mathcal{X}(x_1) + \omega_\mathcal{Y}(y_1) - \omega_\mathcal{X}(x_2) - \omega_\mathcal{Y}(y_2) 
            - 
            \dotprod{\nabla_x \omega_\mathcal{X}(x_2)}{x_1 - x_2}
            \\&~~~~~-\dotprod{\nabla_y \omega_\mathcal{Y}(y_2)}{y_1 - y_2}
            =
            V_\mathcal{X}(x_1, x_2) + V_\mathcal{Y}(y_1, y_2) . 
        \end{split}
        \end{equation*}

        This finishes the first part of the proof. Consider the second statement. By definition of the conjugate norm for all $a := (a_x, a_y) \in \mathcal{Z}^*$ with $a_x \in \mathcal{X}^*$ and $a_y \in \mathcal{Y}^*$ we have
        \begin{equation*}
        \begin{split}
            \| a \|_{\mathcal{Z}^*} 
            &\overset{\text{def}}{=} 
            \sup_{z \in \mathcal{Z}: \|z\|_\mathcal{Z} \leq 1} \left\{ \langle a, z \rangle \right\}
            =
            \sup_{(x, y)^T \in \mathcal{Z}: \|x\|_\mathcal{X}^2 + \|y\|_\mathcal{Y}^2 \leq 1} \left\{ \langle a_x, x \rangle + \langle a_y, y \rangle \right\}
            \\&\leq
            \sup_{x \in \mathcal{X}: \|x\|_\mathcal{X} \leq 1} \left\{ \langle a_x, x \rangle \right\}
            +
            \sup_{y \in \mathcal{Y}: \|y\|_\mathcal{Y} \leq 1} \left\{ \langle a_y, y \rangle \right\}
            =
            \| a_x \|_{\mathcal{X}^*} + \| a_y \|_{\mathcal{Y}^*} .
        \end{split}
        \end{equation*}
        This means that $\| a \|_{\mathcal{Z}^*}^2 \leq \left( \| a_x \|_{\mathcal{X}^*} + \| a_y \|_{\mathcal{Y}^*} \right)^2 \leq 2 \| a_x \|_{\mathcal{X}^*}^2 + 2 \| a_y \|_{\mathcal{Y}^*}^2$. This finishes the proof.
    \end{proof}

    \begin{lemma}
    \label{lemma:conv_to_monotone}
        If a function $g(x, y): \mathcal{X} \times \mathcal{Y} \to \R$ is convex w.r.t. $x$ and concave w.r.t. $y$, then target operator $F$ for the min-max problem $\min_{x \in \mathcal{X}} \max_{y \in \mathcal{Y}} \{ g(x, y) \}$ of the from
        \begin{equation*}
            F(z) := \left[\nabla_x g(x, y), \; -\nabla_y g(x, y)\right]^T ,
        \end{equation*}
        is monotone.
    \end{lemma}
    \begin{proof}
        Let us write down scalar product from the definition of the monotone operator from Assumption \ref{as:monotone}:
        \begin{equation*}
        \begin{split}
            \langle F(z_1) - F(z_2) , z_1 - z_2 \rangle
            &=
            \langle \nabla_x g(x_1, y_1) - \nabla_x g(x_2, y_2) , x_1 - x_2 \rangle
            \\&-
            \langle \nabla_y g(x_1, y_1) - \nabla_y g(x_2, y_2) , y_1 - y_2 \rangle
            \\&=
            \langle \nabla_x g(x_1, y_1) , x_1 - x_2 \rangle
            +
            \langle -\nabla_y g(x_1, y_1), y_1 - y_2 \rangle
            \\&+
            \langle \nabla_x g(x_2, y_2) , x_2 - x_1 \rangle
            +
            \langle -\nabla_y g(x_2, y_2), y_2 - y_1 \rangle
            \\&\geq
            g(x_1, y_1) - g(x_2, y_1) 
            + 
            g(x_1, y_2) - g(x_1, y_1)
            \\&+
            g(x_2, y_2) - g(x_1, y_2)
            +
            g(x_2, y_1) - g(x_2, y_2) = 0 .
        \end{split}
        \end{equation*}
        All inequalities hold since $g(x, y)$ is convex and concave w.r.t. $x$ and $y$ respectively. This finishes the proof.
    \end{proof}

\begin{proof}[Proof of Proposition \ref{proposition:pi_to_vi_ass}]
        We start from the fact, if $f_i$ from \eqref{eq:emp_risk} fall under Assumption \ref{as:lipgrad}, then target operator $$F(z = (\theta, \pi)^T) := \left[\sum_{i=1}^n \pi_i \nabla \tilde{f}_i(\theta), \; -\tilde{f}_1(\theta), ..., -\tilde{f}_n(\theta)\right]^T,$$ from the equation \eqref{eq:pi_to_vi} falls under Assumption \ref{as:lipvar}. Let us start from the definition of the Lipschitz continuous operators:
        \begin{align*}
            \| F(z_1) - F(z_2) \|_*^2
            &\leq
            2 \underbrace{\left\| \sum_{i=1}^n \pi_i^1 \nabla f_i(\theta^1) - \pi_i^2 \nabla f_i(\theta^2) \right\|_2^2}_{\circledOne}
            \\&+
            2 \underbrace{\left\| [f_1(\theta^1) - f_1(\theta^2), ..., f_n(\theta^1) - f_n(\theta^2)]^T \right\|_\infty^2}_{\circledTwo}.
        \end{align*}

        In this inequality we used Lemma \ref{lemma:sum_X_and_Y} with $(\mathcal{X}, \| \cdot \|_\mathcal{X}) = (\Theta, \| \cdot \|_2)$ and $(\mathcal{Y}, \| \cdot \|_\mathcal{Y}) = (\Delta_{n-1}, \| \cdot \|_1)$. Let us consider $\circledOne$. 
        \begin{align}
            &\left\| \sum_{i=1}^n \pi_i^1 \nabla f_i(\theta^1) - \pi_i^2 \nabla f_i(\theta^2) \right\|_2^2
            =
            \left\| \sum_{i=1}^n \pi_i^1 
            \left[\nabla f_i(\theta^1) - \nabla f_i(\theta^2)\right] - \sum_{i=1}^n \left[\pi_i^2 - \pi_i^1\right] \nabla f_i(\theta^2) \right\|_2^2 \nonumber
            \\&\leq
            2 \left\| \sum_{i=1}^n \pi_i^1 
            \left[\nabla f_i(\theta^1) - \nabla f_i(\theta^2)\right] \right\|_2^2
            +
            2 \left\|\sum_{i=1}^n \left[\pi_i^2 - \pi_i^1\right] \nabla f_i(\theta^2) \right\|_2^2 \nonumber
            \\&\leq
            2 \sum_{i=1}^n\pi_i^1 \left\| \nabla f_i(\theta^1) - \nabla f_i(\theta^2) \right\|_2^2
            +
            2 \left(\sum_{i=1}^n \left|\pi_i^2 - \pi_i^1\right| \cdot \left\|\nabla f_i(\theta^2) \right\|_2 \right)^2 \nonumber
            \\&\leq
            2 \sum_{i=1}^n\pi_i^1 L_i^2 \| \theta^1 - \theta^2 \|^2_2
            +
            2 \left(\sum_{i=1}^n \left|\pi_i^2 - \pi_i^1\right| \right)^2 \cdot G^2 \nonumber
            \\&\leq
            2 \max_{i \in \overline{1, n}}\{L_i^2\} \cdot \| \theta^1 - \theta^2 \|^2_2
            +
            2 G^2 \cdot \|\pi^1 - \pi^2\|_1^2. \label{eq:tmp_pi_to_vi_1}
        \end{align}

        Here we use a notation $G := \max_{i \in \overline{1, n}~ \theta \in \Theta}\{ \| \nabla f_i(\theta) \|_2 \}$. Since $f_i(\cdot)$ are convex according to Assumption \ref{as:convex}, then $G = \max_{i \in \overline{1, n}} \{K_i\}$. %from Assumption \ref{as:lipgrad}.
        
        Consider $\circledTwo$. By definition of $\| \cdot \|_\infty$ norm we can obtain:
        \begin{equation}
        \label{eq:tmp_pi_to_vi_2}
        \begin{split}
            \left\| [f_1(\theta^1) - f_1(\theta^2), ..., f_n(\theta^1) - f_n(\theta^2)]^T \right\|_\infty^2
            &=
            \left( \max_{i \in \overline{1, n}}\left\{ \left| f_i(\theta^1) - f_i(\theta^2) \right| \right\} \right)^2
            \\&=
            \max_{i \in \overline{1, n}}\left\{ \left| f_i(\theta^1) - f_i(\theta^2) \right|^2 \right\}
            \\&\leq
            \max_{i \in \overline{1, n}}\left\{ K_i^2 \right\} \cdot \|\theta^1 - \theta^2 \|^2_2.
        \end{split}
        \end{equation}

        Combing \eqref{eq:tmp_pi_to_vi_1}, \eqref{eq:tmp_pi_to_vi_2} and the fact that $G = \max_{i \in \overline{1, n}} \{K_i\}$, we can obtain that 
        \begin{equation}
        \label{eq:tmp_pi_to_vi_3}
        \begin{split}
            \| F(z_1) - F(z_2) \|_*^2
            &\leq
            \left(4 \max_{i \in \overline{1, n}}\{L_i^2\} + 2 \max_{i \in \overline{1, n}}\left\{ K_i^2 \right\} \right) \cdot \| \theta^1 - \theta^2 \|^2_2
            +
            4 \max_{i \in \overline{1, n}}\{K_i^2\} \cdot \| \pi^1 - \pi^2\|_1^2
            \\&\leq
            4 \left[ \max_{i \in \overline{1, n}}\{L_i^2\} + \max_{i \in \overline{1, n}}\left\{ K_i^2 \right\}  \right] \cdot \left( \| \theta^1 - \theta^2 \|^2 _2+ \|\pi^1 - \pi^2\|_1^2 \right)
            \\&=
            4 \left[ \max_{i \in \overline{1, n}}\{L_i^2\} + \max_{i \in \overline{1, n}}\left\{ K_i^2 \right\}  \right] \cdot \| z_1 - z_2\|^2.
        \end{split}
        \end{equation}
        The last equality holds because according to Lemma \ref{lemma:sum_X_and_Y} $\| \cdot \|^2_\mathcal{Z} = \| \cdot \|^2_\mathcal{X} + \| \cdot \|^2_\mathcal{Y}$. From \eqref{eq:tmp_pi_to_vi_3} we can obtain that 
        $$L_F^2 \leq 4 \left[ \max_{i \in \overline{1, n}}\{L_i^2\} + \max_{i \in \overline{1, n}}\left\{ K_i^2 \right\}  \right].$$
        This finishes the first part of the proof.

        Consider the second part of Proposition \ref{proposition:pi_to_vi_ass}. In this case $g(\theta, \pi) = \sum_{i=1}^n \pi_i f_i(\theta)$. This function is liner w.r.t. $\pi$, therefore it is concave w.r.t. $\pi$, according to the Assumption \ref{as:convex} all functions $f_i(\cdot)$ are convex, therefore $g(\pi, \theta)$ is convex w.r.t. $\theta$. Now, using Lemma \ref{lemma:conv_to_monotone}, we can obtain that target operator for the problem \eqref{eq:adv_pi_problem} is monotone. This finishes the proof. 
    \end{proof}
\subsection{Proof of Proposition \ref{theorem:also_step}}
\label{appendix:also_step}
\begin{proof}[Proof of Proposition \ref{theorem:also_step}]
Consider the step of Mirror-Prox like algorithm:

\begin{equation}
\label{eq:also_iter}
z^{\text{new}} = \arg\min_{z \in \cZ} \left\{ \langle \gamma g, z \rangle + V(z, z^{\text{old}}) + \gamma \tau V(z, \hat{z})\right\}
\end{equation}

According to structure of the problem \eqref{eq:adv_pi_problem} and definition of $z$, the problem \eqref{eq:also_iter} is equivalent to following problems:

\begin{equation}
\label{eq:also_iter_theta}
\theta^{\text{new}} = \arg\min_{\theta \in \mathbb{R}^d} \left\{ \langle \gamma g^\theta, \theta \rangle + \frac{1}{2}\|\theta - \theta^{\text{old}}\|_2^2 + \frac{\gamma \tau}{2}\|\theta\|_2^2\right\}
\end{equation}

\begin{equation}
\label{eq:also_iter_pi}
\pi^{\text{new}} = \arg\min_{\pi \in \Delta^{n-1}} \left\{ \langle \gamma g^\pi, \pi \rangle + \text{KL} \left[\mathbf{\pi} ~\|~ \pi^{\text{old}}\right] + \gamma \tau \text{KL} \left[\mathbf{\pi} ~\|~ \hat{\pi}\right] \right\}
\end{equation}

We will start with \eqref{eq:also_iter_theta}. Using first order optimality condition for $\theta^{\text{new}}$ we can obtain that

$$\gamma g^\theta + (\theta^{\text{new}} - \theta^{\text{old}}) + \gamma\tau \theta^{\text{new}} = 0$$

Then

$$\theta^{\text{new}} (1 + \gamma\tau) = \theta^{\text{old}} - \gamma g^\theta$$

$$\theta^{\text{new}} = \frac{1 + \gamma\tau - \gamma\tau}{1 + \gamma\tau}\theta^{\text{old}} - \frac{\gamma}{1 + \gamma\tau} g^\theta$$
$$\theta^{\text{new}} = \theta^{\text{old}} - \frac{\gamma}{1 + \gamma\tau} \left(g^\theta + \tau \theta^{\text{old}}\right)$$

To deal with \eqref{eq:also_iter_pi} we reformulate it as classical constrained optimization problem

\begin{equation}
\label{eq:also_iter_pi_reformulated}
\min_\pi \quad \langle \gamma g^\pi, \pi \rangle + \text{KL} \left[\mathbf{\pi} ~\|~ \pi^{\text{old}}\right] + \gamma \tau \text{KL} \left[\mathbf{\pi} ~\|~ \hat{\pi}\right] \quad s.t. \quad \sum_{i=1}^n \pi_i = 1, \; \pi_i \geq 0 \;\; \forall i=\overline{1...n}
\end{equation}

We use Karush–Kuhn–Tucker conditions \citep{kuhn1951proceedings} to solve problem \eqref{eq:also_iter_pi_reformulated}. Let us write out a Lagrange function $L(\pi, \beta_1, \dots, \beta_n, \lambda)$ for problem \eqref{eq:also_iter_pi_reformulated}:
\begin{equation*}
    L(\pi, \beta_1, \dots, \beta_n, \lambda) 
    :=
    \sum_{i=1}^n\left[ 
        \gamma\pi_i g_i^\pi 
        - 
         \pi_i \log(\pi_i / \pi^{\text{old}}_i)
        -
        \gamma\tau\pi_i \log(\pi_i / \hat{\pi}_i)
        \right]
    - \sum_{i=1}^n \beta_i \pi_i + \lambda \sum_{i=1}^n \pi_i - \lambda,
\end{equation*}
where KKT multipliers $\beta_i \geq 0$ correspond to the inequalities $-\pi_i \leq 0$ and $\lambda \in \R$ stands for equality $\sum_{i=1}^n \pi_i - 1 = 0$. 

Let us write out partial derivative $\partial L / \partial \pi_i$:
\begin{equation}
\label{eq:tmp_also_3}
    \frac{\partial L}{\partial \pi_i} 
    =
    \gamma g_i^\pi
    +
    \log(\pi_i / \pi^{\text{old}}_i) + 1
    +
    \gamma\tau\log(\pi_i / \hat{\pi}_i) + \gamma\tau
    - \beta_i + \lambda.
\end{equation}

Since $L$ is convex with respect to $\pi$, we can set $\partial L / \partial \pi_i$ to zero. From \eqref{eq:tmp_also_3} we can obtain that
\begin{equation*}
\begin{split}
    \pi^*_i &= \left(\pi^{\text{old}}_i (\hat{\pi}_i)^{\gamma\tau}\exp\left[- \gamma g_i^\pi \right] \cdot \exp\left[ - \lambda - \gamma\tau - 1 + \beta_i \right]\right)^{\frac{1}{1+\gamma\tau}} 
    \\&=  \left(\pi^{\text{old}}_i (\hat{\pi}_i)^{\gamma\tau}\right)^{\frac{1}{1+\gamma\tau}}\exp\left[-\frac{1 + \gamma\tau + \gamma g_i^\pi}{1 + \gamma\tau} \right] \cdot \exp\left[\frac{\beta_i - \lambda}{1 + \gamma\tau}\right].
\end{split}
\end{equation*}

Now one can write dual problem and find out that $\lambda_i^* = 0$. Since $\sum_{i=1}^n \pi^*_i = 1$:

\begin{equation*}
    \exp\left[\frac{-\lambda}{1 + \gamma\tau}\right] \sum_{i=1}^n \left(\left(\pi^{\text{old}}_i (\hat{\pi}_i)^{\gamma\tau}\right)^{\frac{1}{1+\gamma\tau}}\exp\left[-\frac{1 + \gamma\tau + \gamma g_i^\pi}{1 + \gamma\tau} \right]\right) = 1
\end{equation*}

\begin{equation*}
    \Rightarrow \exp\left[\frac{-\lambda}{1 + \gamma\tau}\right] = \frac{1}{\sum_{i=1}^n \left(\left(\pi^{\text{old}}_i (\hat{\pi}_i)^{\gamma\tau}\right)^{\frac{1}{1+\gamma\tau}}\exp\left[-\frac{1 + \gamma\tau + \gamma g_i^\pi}{1 + \gamma\tau} \right]\right)}
\end{equation*}

then all conditions of KKT will be fulfilled and optimal $\pi^* = \pi^{\text{new}}$ takes form:
\begin{align*}
    \pi^{\text{new}}_i 
    &= \frac{\left(\pi^{\text{old}}_i (\hat{\pi}_i)^{\gamma\tau}\right)^{\frac{1}{1+\gamma\tau}}\exp\left[-\frac{1 + \gamma\tau + \gamma g_i^\pi}{1 + \gamma\tau} \right]}{\sum_{j=1}^n \left(\left(\pi^{\text{old}}_j (\hat{\pi}_j)^{\gamma\tau}\right)^{\frac{1}{1+\gamma\tau}}\exp\left[-\frac{1 + \gamma\tau + \gamma g_j^\pi}{1 + \gamma\tau} \right]\right)} 
    \\&= 
    \frac{\left(\pi^{\text{old}}_i (\hat{\pi}_i)^{\gamma\tau}\right)^{\frac{1}{1+\gamma\tau}}\exp\left[-\frac{\gamma g_i^\pi}{1 + \gamma\tau} \right]}{\sum_{j=1}^n \left(\left(\pi^{\text{old}}_j (\hat{\pi}_j)^{\gamma\tau}\right)^{\frac{1}{1+\gamma\tau}}\exp\left[-\frac{\gamma g_j^\pi}{1 + \gamma\tau} \right]\right)}
\end{align*}

Taking logarithm from both sides:

\begin{align*}
    \log \pi^{\text{new}}_i 
    &= \frac{1}{1 + \gamma\tau} \log \pi^{\text{old}}_i + \frac{\gamma\tau}{1+\gamma\tau}\log \hat{\pi}_i
    - \frac{\gamma g_i^\pi}{1 + \gamma\tau} 
    \\&~~~~~+ \log \sum_{j=1}^n \left(\left(\pi^{\text{old}}_j (\hat{\pi}_j)^{\gamma\tau}\right)^{\frac{1}{1+\gamma\tau}}\exp\left[-\frac{\gamma g_j^\pi}{1 + \gamma\tau} \right]\right)
    \\&= \log \pi^{\text{old}}_i - \frac{\gamma\tau}{1 + \gamma\tau} \log \pi^{\text{old}}_i + \frac{\gamma\tau}{1+\gamma\tau}\log \hat{\pi}_i
    - \frac{\gamma g_i^\pi}{1 + \gamma\tau} 
    \\&~~~~~+ \log \sum_{j=1}^n \left(\left(\pi^{\text{old}}_j (\hat{\pi}_j)^{\gamma\tau}\right)^{\frac{1}{1+\gamma\tau}}\exp\left[-\frac{\gamma g_j^\pi}{1 + \gamma\tau} \right]\right)
    \\&= \log \pi^{\text{old}}_i - \frac{\gamma}{1 + \gamma\tau} (g_i^\pi + \tau \log \frac{\pi_i^{\text{old}}}{\hat{\pi}_i})
    + \log \sum_{j=1}^n \left(\left(\pi^{\text{old}}_j (\hat{\pi}_j)^{\gamma\tau}\right)^{\frac{1}{1+\gamma\tau}}\exp\left[-\frac{\gamma g_j^\pi}{1 + \gamma\tau} \right]\right)
\end{align*}

Then from softmax definition we can obtain that:

\begin{equation*}
    \log \pi^{\text{new}}_i = SM\left(\log \pi^{\text{old}}_i - \frac{\gamma}{1 + \gamma\tau} (g_i^\pi + \tau \log \frac{\pi_i^{\text{old}}}{\hat{\pi}_i})\right)
\end{equation*}

This finishes the proof.

\end{proof}
\newpage
\section{Convergence of ALSO}
\label{appendix:also-theory}

% Definitions section
% This file contains all definitions used in the Theory for ALSO section

\subsection{Definitions}

Let \( h(\theta, \pi) \) be a differentiable function defined in~\ref{eq:adv_pi_problem}.
In our analysis, we will consider Assumptions~\ref{as:lipgrad}, \ref{as:stoch_grad}, and~\ref{as:uncertainty_set} to provide theoretical guarantees.

In fact, we apply ~\ref{as:stoch_grad} to estimate the norms of stochastic gradients and we add batch size $B$ to control the variance of noise that occurs due to stochastics in gradient oracle. Also in ~\ref{as:lipgrad} we require the $K_{i,j}$-Lipschitz continuity of \( f_{i,j}(\theta) \) and their $L_{i,j}$-smoothness. In the sequel, assumption ~\ref{ass:theta} is useful several times in calculations, but it has a different form, however, we can estimate this constant $L$ through our existing $L_{i,j}$ and $K_{i,j}$.

% In the ~\ref{as:uncertainty_set} we define the domain for \( \pi \) as the set \( U \cap \Delta \), which is a simplex $\Delta$ with truncated corners $U \cap \Delta$ to ensure that the KL divergence remains bounded on \( \Delta \cap U \). In our case, we dont
% However, in practice, this constraint can often be relaxed to $\Delta$. 
% In our experiments, we use the standard simplex $\Delta$ as a heuristic approximation for \( U \cap \Delta\), since the solutions typically do not lie near the boundary where the KL divergence would become unbounded.

% In fact, Assumption \ref{as:uncertainty_set} is not a limitation for us. Indeed, it was shown \cite{ben2001lectures} that if we consider a closed convex set $X \subset \Delta^n$ as a domain, then the $\text{KL}$-divergence can be upper bounded by a constant. What is more, our proof can be adapted to a regularized version of $\text{KL}$-divergence that is bounded on the simplex, in which case the \ref{as:uncertainty_set} assumption is not required since the constant $D$ can be analytically defined.  

We use asssumption ~\ref{as:uncertainty_set} with set $U$ because this notation is adopted in the related paper \citep{mehta2024drago}. Namely, we define the domain for \( \pi \) as the set \( U \cap \Delta \), which is usually used to truncate corners of $\Delta$ to ensure that the KL divergence remains bounded on \( \Delta \cap U \). However, in our theory we do not require that the simplex must be with truncated corners.

In this section, we consider a more general case of assumptions for our algorithm. So we now introduce several definitions and lemmas proven in \cite{bylinkin2025enhancingstabilityphysicsinformedneural}, which will be used in the convergence analysis.

We consider more general problem than \eqref{eq:adv_pi_problem}:
\begin{align}\label{eq:pinn_saddle}
    \min_{\theta\in\R^d}\max_{\pi\in S}\left[\mathcal{L}(\theta, \pi)=\sum_{i=1}^c \pi_i \left(\frac{c}{n}\sum_{j=1}^{n_i}f_{i, j}(\theta)\right) 
        + \frac{\tau}{2}\|\theta\|^2_2 - \lambda D_{\psi}(\pi||\hat{\pi}) \right],
\end{align}
where we replace $\text{KL}$-divergence with general $D_\Psi$-divergence (Bregman divergence).

% Assumptions section
% This file contains all assumptions used in the Theory for ALSO section

\begin{assumption}\label{as:pi-domain}
    The domain $S \subseteq \mathbb R^c$ is nonempty, closed, convex, with $\hat\pi \in \mathrm{Int}(S)$.
\end{assumption}

\begin{assumption}\label{ass:theta}
    The function $\mathcal{L}(\theta,\pi)$ is $L$-smooth, i.e. for all $(\theta_1,\pi_1), (\theta_2,\pi_2) \in \mathbb{R}^d\times S$ it satisfies 
    \begin{align*}
        \|\nabla \mathcal{L}(\theta_1,\pi_1) - \nabla \mathcal{L}(\theta_2,\pi_2)\|^2 \leq L^2\left(\|\theta_1-\theta_2\|^2+\|\pi_1-\pi_2\|^2\right).
    \end{align*}
\end{assumption}

\begin{lemma}\label{lem:L-smooth-aggregation}
    Under Assumptions~\ref{as:lipgrad}, and~\ref{as:pi-domain}, 
    the function $\mathcal L(\theta,\pi)$ in~\eqref{eq:pinn_saddle} is $L$-smooth (i.e. Assumption~\ref{ass:theta}), i.e.
    for all $(\theta^1,\pi^1),(\theta^2,\pi^2)\in\R^d\times S$ it holds
    \[
        \|\nabla \mathcal L(\theta^1,\pi^1) - \nabla \mathcal L(\theta^2,\pi^2)\|^2
        \;\leq\;
        L^2 \Big(\|\theta^1-\theta^2\|^2 + \|\pi^1-\pi^2\|^2\Big),
    \]
    where the Lipschitz constant $L$ can be chosen as
    \[
        L^2 \;=\;
        \Big(\tfrac{c}{n}\max_{i\in[c]}\sum_{j=1}^{n_i}L_{i,j} + \tau + \tfrac{c}{n}\max_{i\in[c]}\sum_{j=1}^{n_i}K_{i,j}\Big)^2
        + (\lambda L_\psi)^2,
    \]
    with $L_{i,j}$ and $K_{i,j}$ being the smoothness and Lipschitz constants of $f_{i,j}$ from Assumption~\ref{as:lipgrad}, 
    and $L_\psi$ the Lipschitz constant of $\nabla_\pi D_\psi(\cdot\|\hat \pi)$.
\end{lemma}

\begin{proof}
We decompose the gradient into its $\theta$- and $\pi$-parts:
\[
\nabla_\theta \mathcal{L}(\theta,\pi)
= \sum_{i=1}^c \pi_i \Big(\tfrac{c}{n}\sum_{j=1}^{n_i}\nabla f_{i,j}(\theta)\Big) + \tau \theta,
\quad
\nabla_\pi \mathcal{L}(\theta,\pi)
= \Big(\tfrac{c}{n}\sum_{j=1}^{n_i} f_{i,j}(\theta)\Big)_{i=1}^c - \lambda \nabla_\pi D_\psi(\pi\|\hat \pi).
\]

For the $\theta$-part we obtain
\begin{align*}
&\|\nabla_\theta \mathcal L(\theta^1,\pi^1) - \nabla_\theta \mathcal L(\theta^2,\pi^2)\| \\
&\le \sum_{i=1}^c |\pi^1_i-\pi^2_i|\Big(\tfrac{c}{n}\sum_{j=1}^{n_i}\|\nabla f_{i,j}(\theta^1)\|\Big)
+ \tfrac{c}{n}\sum_{i=1}^c \pi^2_i \sum_{j=1}^{n_i}\|\nabla f_{i,j}(\theta^1)-\nabla f_{i,j}(\theta^2)\|
+ \tau \|\theta^1-\theta^2\| \\
&\le \tfrac{c}{n}\max_{i}\sum_{j=1}^{n_i}K_{i,j}\,\|\pi^1-\pi^2\|
+ \Big(\tfrac{c}{n}\max_{i}\sum_{j=1}^{n_i}L_{i,j}+\tau\Big)\|\theta^1-\theta^2\|.
\end{align*}

For the $\pi$-part we analogously have
\begin{align*}
\|\nabla_\pi \mathcal L(\theta^1,\pi^1)-\nabla_\pi \mathcal L(\theta^2,\pi^2)\|
&\le \tfrac{c}{n}\max_{i}\sum_{j=1}^{n_i}K_{i,j}\,\|\theta^1-\theta^2\|
+ \lambda L_\psi \|\pi^1-\pi^2\|.
\end{align*}

Combining both estimates yields
\[
\|\nabla \mathcal L(\theta^1,\pi^1)-\nabla \mathcal L(\theta^2,\pi^2)\|^2
\;\le\;
\Big(\tfrac{c}{n}\max_{i}\sum_j L_{i,j}+\tau+\tfrac{c}{n}\max_{i}\sum_j K_{i,j}\Big)^2\|\theta^1-\theta^2\|^2
+ (\lambda L_\psi)^2\|\pi^1-\pi^2\|^2,
\]
which completes the proof.
\end{proof}

\begin{lemma}\label{lem:K-lipschitz}
    Under Assumption~\ref{as:lipgrad}, with $\tau=0$, 
    the function $\mathcal L(\theta,\pi)$ in~\eqref{eq:pinn_saddle} is $K$-lipschitz with respect to $\theta$, i.e.
    for all $\theta^1,\theta^2\in\R^d$ and $\pi \in S$ it holds
    \[
        |\mathcal L(\theta^1,\pi) - \mathcal L(\theta^2,\pi)|
        \;\leq\;
        L \|\theta^1-\theta^2\|,
    \]
    where the $K$ can be chosen as
    \[
        K \;=\; \frac{c}{n} \max_{i\in[c]} \sum_{j=1}^{n_i} K_{ij}
    \]
    with $K_{i,j}$ being Lipschitz constant of $f_{i,j}$ from Assumption~\ref{as:lipgrad}.
\end{lemma}

\begin{proof}

$$|\mathcal L(\theta^1,\pi) - \mathcal L(\theta^2,\pi)| = |\sum_{i=1}^c \pi_i \frac{c}{n} \sum_{j=1}^{n_i} (f_{ij}(\theta^1) - f_{ij}(\theta^2))| \leq$$
$$\sum_{i=1}^c \pi_i \frac{c}{n} \sum_{i=1}^n |f_{ij}(\theta^1) - f_{ij}(\theta^2)| \leq \sum_{i=1}^c \pi_i \frac{c}{n} \sum_{j=1}^{n_i} K_{ij}\|\theta^1 - \theta^2\| \leq$$
$$\leq \frac{c}{n}\|\theta^1 - \theta^2\| \sum_{i=1}^c \pi_i  \sum_{j=1}^{n_i} K_{ij} \leq \frac{c}{n} \max_{i\in[c]} \sum_{j=1}^{n_i} K_{ij}$$
The last inequality holds, since $\pi \in \Delta_{c-1}$.
\end{proof}

\begin{assumption}\label{ass:p}
        The function $\psi$, which produce $D_\psi$, is \textbf{1-strongly convex}, i.e. for all $\pi_1,\pi_2\in S$ it satisfies
        \begin{align*}
            \psi(\pi_1) \geqslant \psi(\pi_2) + \left\langle \nabla \psi(\pi_2), \pi_1 - \pi_2 \right\rangle + \frac{1}{2}\|\pi_2 - \pi_1\|^2.
        \end{align*}
\end{assumption}

% \todo{\textbf{\textcolor{blue}{Write that for simple analisys we use general sigma for both sources of stochasticity}}}

% \begin{assumption}\label{ass:stoch}
%     Stochastic oracles $G_{\theta}$ and $G_{\pi}$ are unbiased and light-tailed, i.e.
%     \begin{align*}
%         \E_{\xi}\left[G_{\theta}(\theta,\pi,\xi)\right] &= \nabla_{\theta}\mathcal{L}(\theta,\pi), &
%         \E\left[\|G_{\theta}(\theta,\pi,\xi)-\nabla_{\theta}\mathcal{L}(\theta,\pi)\|^2\right] &\leq\sigma^2, \\
%         \E_{\zeta}\left[G_{\pi}(\theta,\pi,\zeta)\right] &= \nabla_{\pi}\mathcal{L}(\theta,\pi), &
%         \E\left[\|G_{\pi}(\theta,\pi,\zeta)-\nabla_{\pi}\mathcal{L}(\theta,\pi)\|_*^2\right] &\leq\sigma^2,
%     \end{align*}
%     for all $(\theta,\pi)\in\R^d\times S$.
% \end{assumption}

% To rigorously analyze the convergence of our method, we now introduce a set of assumptions that define the regularity conditions of the function $f(\theta, \pi)$ and the feasible sets. These assumptions mirror those in~\cite{lin2019descentascent} but are adapted to account for the use of Adam \citep{kingma2014adam} and regularization in \texttt{ALSO}.

% Other theorems and lemmas section
% This file contains additional theorems and lemmas used in the Theory for ALSO section

Lets formulate lemma from \citep{bylinkin2025enhancingstabilityphysicsinformedneural}
\begin{lemma}[~\cite{bylinkin2025enhancingstabilityphysicsinformedneural}]\label{eq:lemma_1}
    Consider the problem \eqref{eq:pinn_saddle} under Assumption \ref{ass:p}. Then, for every $\theta\in\mathbb{R}^d$ the function $\mathcal{L}(\theta,\pi)$ is \textbf{$\lambda$-strongly concave}, i.e. for all $\pi_1,\pi_2\in S$ it satisfies
    \begin{align*}
        \mathcal{L}(\theta,\pi_1)\leq\mathcal{L}(\theta,\pi_2)+\langle\nabla_{\psi}\mathcal{L}(\theta,\pi_2),\pi_1-\pi_2\rangle-\frac{\lambda}{2}\left(D_{\psi}(\pi_1,\pi_2)+D_{\psi}(\pi_2,\pi_1)\right).
    \end{align*}
\end{lemma}

% Auxiliary lemmas section
% This file contains auxiliary lemmas used in the Theory for ALSO section

\subsection{Auxiliary lemmas}

% \todo{\color{blue} Update notations}

\begin{notation}
    For the saddle-point problem~\eqref{eq:pinn_saddle} and Algorithm~\ref{algorithm:also_optimistic}, 
    we use the following notation, aligned with~\cite{bylinkin2025enhancingstabilityphysicsinformedneural}:
    \begin{align*}
        g_\theta^t &\;\equiv\; \frac{c}{B}\sum_{j=1}^{B} 
            \pi_{c^t_j}\,\nabla_\theta f_{c^t_j,i_j^t}(\theta^{t}),
        && \text{stochastic gradient w.r.t. $\theta$,} \\[0.75em]
        g_\pi^t &\;\equiv\; \frac{c}{B}\sum_{j=1}^{B} 
            e_{c^t_j}\, f_{c^t_j,i_j^t}(\theta^{t}) 
            \;-\; \lambda\,\nabla_\pi D_\psi(\pi^t \| \hat{\pi}),
        && \text{stochastic gradient w.r.t. $\pi$,} \\[0.75em]
        \gamma_\theta &\;\text{--- stepsize for $\theta$,} 
        && \gamma_\pi \;\text{--- stepsize for $\pi$,} \\[0.75em]
        \mathcal{L}(\theta,\pi) &\;\equiv\; \sum_{i=1}^c \pi_i \Bigl(\tfrac{c}{n}\sum_{j=1}^{n_i} f_{i,j}(\theta)\Bigr)
            + \tfrac{\tau}{2}\|\theta\|_2^2 - \lambda D_\psi(\pi\|\hat{\pi}),
        && S \;\text{--- feasible set for $\pi$.}
    \end{align*}
    Here $e_i$ denotes the $i$-th standard basis vector in $\R^c$, 
    $\hat{\pi}$ is the reference distribution in the regularization term, 
    and $\nabla_\pi D_\psi(\pi^t\|\hat{\pi})$ denotes the gradient (or subgradient) of the divergence $D_\psi$ with respect to $\pi$.
\end{notation}

According to the notation, Algorithm~\ref{algorithm:also_optimistic} can
be formulated in a simpler form:
\begin{align*}
    \theta^{t+1} &= \theta^t \;-\; \gamma_\theta \, d_\theta^t, \\[0.5em]
    \pi^{t+1} &= \argmin_{\pi \in S} 
        \Bigl\{ \,\langle -\gamma_\pi g_\pi^t,\; \pi \rangle 
        \;+\; D_\psi(\pi \,\|\, \pi^t)\Bigr\},
\end{align*}
where $d_\theta^t$ is classsical Adam step.

We begin by noting that our convergence analysis is based on the Adam estimator.
Let us introduce the main Adam Estimator process:
\begin{align} \label{algorithm:also_estimator}
    \theta^{t+1} &= \theta^t - \gamma_\theta d_\theta^t 
    \;=\; \theta^t - \gamma_\theta \frac{m_\theta^t}{b_t}, \\[0.5em]
    \pi^{t+1} &= \argmin_{\pi \in S} 
        \Bigl\{ \,\langle -\gamma_\pi g_\pi^t,\; \pi \rangle 
        + D_\psi(\pi \,\|\, \pi^t)\Bigr\}.
\end{align}

We also introduce a copy of the main process, which behaves identically to the original algorithm but is used to generate the scaling constant \(b_t\) for the main process:
\begin{align*}
    \theta^{t+1}_{\text{copy}} &= \theta^t_{\text{copy}} - \gamma_\theta \frac{m_{\theta,\text{copy}}^t}{b_t}, \\[0.5em]
    \pi^{t+1}_{\text{copy}} &= \argmin_{\pi \in S} 
        \Bigl\{ \,\langle -\gamma_\pi \tilde{g}_\pi^t,\; \pi \rangle 
        + D_\psi(\pi \,\|\, \pi^t_{\text{copy}})\Bigr\}.
\end{align*}

The update rules for the copy and main processes are:
\begin{align*}
    m_{\theta,\text{copy}}^t &= \beta_1 m_{\theta,\text{copy}}^{t-1} + (1 - \beta_1)\tilde{g}_\theta^t, \\[0.5em]
    b_t^2 &= \beta_2 b_{t-1}^2 + (1 - \beta_2)\|\tilde{g}_\theta^t\|^2, \\[0.5em]
    m_\theta^t &= \beta_1 m_\theta^{t-1} + (1 - \beta_1) g_\theta^t,
\end{align*}
where \(g_\theta^t\) is the stochastic gradient with respect to \(\theta\) at the point \((\theta^t,\pi^t)\), 
and \(\tilde{g}_\theta^t\) is the stochastic gradient at the point \((\theta^t_{\text{copy}},\pi^t_{\text{copy}})\).

The first moment \(m_\theta^t\) admits a closed-form expression:
\[
    m_\theta^t \;=\; (1 - \beta_1) \sum_{k=0}^t \beta_1^{\,t-k}\, g_\theta^k.
\]

We initialize
\[
    m_{\theta,\text{copy}}^{-1} = m_\theta^{-1} = 0, 
    \qquad b_{-1}, b_0 > 0.
\]
The purpose of introducing the copy process is to decouple the randomness of the estimator:  
in the original process, products of random variables inside expectations are dependent,  
while in the proposed estimator the corresponding quantities can be treated as independent,  
which allows us to move products under the expectation in the convergence analysis.

According to the above, the next lemma holds.

\begin{lemma}[\citep{chezhegov2024gradient}, Lemma 13]\label{properties:adam-estimator}
For a reference step \( r \leq t \), and letting \( \beta_2 = 1 - \frac{1}{K} \) for some \( K \ge t - r \), the following lower bound holds:
    \[
        b_t^2 \geq \beta_2^{t - r} b_r^2 = \left(1 - \frac{1}{K} \right)^{t - r} b_r^2 \geq \left(1 - \frac{1}{K} \right)^K b_r^2 \geq c_m^2 b_r^2,
    \]
where for our Adam-type estimator, we can choose \( c_m = \frac{1}{2} \).
\end{lemma}

Now let us formulate a technical lemma, which we will need in the future to evaluate the resulting sums:
\begin{lemma} \label{lemma:sums-estimation}
    Let \( a_t = -\langle \nabla\Phi(\theta^t), d_\theta^t \rangle \) and 
    \( \xi_t = -\langle \nabla\Phi(\theta^t), g_\theta^t \rangle \), 
    where \( d_\theta^t \) is the Adam estimator step and 
    \( g_\theta^t \) is the stochastic gradient used for the momentum term in the Adam estimator \ref{algorithm:also_estimator}, 
    and \(\theta^t\) is the iterate of the main process at step \(t\).  
    Then, the following inequality holds:
    \[
        \sum_{t = 0}^{T} a_t \;\leq\; \sum_{k = 0}^{T} C_k \xi_k 
        \;+\; 3\gamma_\theta \kappa L \sum_{k = 0}^{T - 1} A_k \| d_\theta^k \|^2,
    \]
    where
    \[
        C_k = (1 - \beta_1) \sum_{t = k}^{T} \frac{\beta_1^{t - k}}{b_t}, 
        \qquad  
        A_k = b_k \sum_{t = k + 1}^{T} \frac{\beta_1^{t - k}}{b_t}.
    \]
\end{lemma}
\begin{proof}
According to the update rule, we have
    \[
        a_t = \frac{1}{b_t} \left( (1 - \beta_1) \xi_t - \langle \nabla\Phi(\theta^t), \beta_1 m_\theta^{t-1} \rangle \right).
    \]
Hence, we get
\begin{align*}
    a_t &= \frac{1}{b_t} \left( (1 - \beta_1) \xi_t 
    + \langle \nabla\Phi(\theta^{t-1}) - \nabla\Phi(\theta^t) - \nabla\Phi(\theta^{t-1}), \beta_1 m_\theta^{t-1} \rangle \right) \\
    &= \frac{1}{b_t} \left( (1 - \beta_1) \xi_t + \beta_1 b_{t-1} a_{t-1} 
    +  \langle \nabla\Phi(\theta^{t-1}) - \nabla\Phi(\theta^{t}), \beta_1 m_\theta^{t-1} \rangle \right).
\end{align*}
Using $3\kappa L$-Lipschitzness of $\Phi$, the last term can be decomposed as follows:    
\begin{align*}
    \langle \nabla\Phi(\theta^{t-1}) - \nabla\Phi(\theta^{t}), \beta_1 m_\theta^{t-1} \rangle 
    &\leq 3\beta_1\kappa L\|\theta^t - \theta^{t-1}\|\|m_\theta^{t-1}\| \\
    &\leq 3\gamma_\theta\kappa L \beta_1 b_{t-1}\|d_\theta^{t-1}\|^2,
\end{align*}
where in the second inequality we apply the property of the proximal operator. Thus, one can obtain
\[
    a_t \le \frac{1}{b_t} (1 - \beta_1) \xi_t 
    + \beta_1\frac{b_{t-1}}{b_t} a_{t-1} 
    + 3\gamma_\theta\kappa L \beta_1\frac{b_{t-1}}{b_{t}}\|d_\theta^{t-1}\|^2.  
\]
Running the recursion over $a_t$, we have
\[
    a_t \le \frac{1}{b_t} \sum_{k=0}^{t} (1 - \beta_1)\, \beta_1^{t - k} \xi_k 
    \;+\;
    3\gamma_\theta\kappa L \sum_{k=0}^{t - 1} \beta_1^{t - k} \frac{b_k}{b_t} \| d_\theta^k \|^2.
\]

Summing over $t = 0$ to $T$, we get:
\[
    \sum_{t = 0}^{T} a_t \leq \sum_{t = 0}^{T} \frac{1}{b_t} \sum_{k = 0}^{t} (1 - \beta_1) \beta_1^{t - k} \, \xi_k
    \;+\;
    3\gamma_\theta\kappa L \sum_{t = 0}^{T} \sum_{k = 0}^{t - 1} \frac{\beta_1^{t - k} b_k}{b_t} \| d_\theta^k \|^2.
\]

Switching the order of sums in the second term leads to
\[
   \sum_{t = 0}^{T} a_t = \sum_{t = 0}^{T} \frac{1}{b_t} \sum_{k = 0}^{t} (1 - \beta_1) \beta_1^{t - k} \, \xi_k
    \;+\;
    3\gamma_\theta\kappa L \sum_{k = 0}^{T - 1} b_k \| d_\theta^k \|^2 \sum_{t = k + 1}^{T} \frac{\beta_1^{t - k}}{b_t}.
\]

Thus, the overall summed inequality becomes:
\[
        \sum_{t = 0}^{T} a_t \leq \sum_{k = 0}^{T} C_k \xi_k 
        + 3\gamma_\theta\kappa L \sum_{k = 0}^{T - 1} A_k \| d_\theta^k \|^2,
\]
where:
\[
    C_k = (1 - \beta_1) \sum_{t = k}^{T} \frac{\beta_1^{t - k}}{b_t}, 
    \qquad  
    A_k = b_k\sum_{t = k + 1}^{T} \frac{\beta_1^{t - k}}{b_t}.
\]
This finishes the proof.
\end{proof}

The next lemma, that is useful for us, help us to upper bound distance between momentum and stochastic gradient:
\begin{lemma} \label{lemma:momentum_upper_bound}
    Let $g_t$ is stochastic gradient, and $m_t$ is momentum of the Adam estimator \ref{algorithm:also_estimator} then distance between them such as folowing:
    \begin{align}
        \| g_t - m_t\|^2 \le \beta_1^2 \cdot G_t,
    \end{align}
    where $\beta_1$ is parameter in Adam and $G_t = 2\left(\|g_t\|^2 + (1 - \beta_1) \sum_{k=0}^{t-1} \beta_1^{t-k} \|g_k\|^2\right) $.
\end{lemma}
\begin{proof}
    \begin{align*}
        \| g_t - m_t\|^2 = \| g_t - (1-\beta_1)g_t - \beta_1 m_{t-1}\|^2 &= \beta_1^2\| g_t - m_{t-1}\|^2 \\
        &\le 2\beta_1^2 \left( \|g_t\|^2 + \| m_{t-1}\|^2\right) 
    \end{align*}

    We know that recursion on momentum $m_t$ is revealed in the following:
    \begin{align*}
        m_{t-1} = (1 - \beta_1)g_{t-1} + m_{t-2} = (1-\beta_1)\sum_{k=0}^{t-1} \beta_1^{t-k}g_k
    \end{align*}

    Using convexity of $\| \cdot \|^2$ we have:
    \begin{align*}
        \|m_{t-1}\|^2 = \|(1-\beta_1)\sum_{k=0}^{t-1} \beta_1^{t-k}g_k\|^2 &\le (1 - \beta_1)^2\frac{1}{1 - \beta_1^{t}} \sum_{k=0}^{t-1} \beta_1^{t-k} \|g_k\|^2 \\ 
        &\le (1 - \beta_1) \sum_{k=0}^{t-1} \beta_1^{t-k} \|g_k\|^2
    \end{align*}
\end{proof}

Now we can move on to the main theorem.

% Main lemma section
% This file contains the main lemma used in the Theory for ALSO section

\subsection{Main lemmas and theorem }

%%%%%% ---MAIN LEMMA------MAIN LEMMA------MAIN LEMMA------MAIN LEMMA------MAIN LEMMA------MAIN LEMMA------MAIN LEMMA------MAIN LEMMA------MAIN LEMMA------MAIN LEMMA------MAIN LEMMA------MAIN LEMMA------MAIN LEMMA------MAIN LEMMA------MAIN LEMMA---
\subsubsection{Main lemma}
\begin{lemma}[Stochastic distance recursion]\label{lemma:distance_stoch}
    Consider the problem \eqref{eq:pinn_saddle} under Assumptions \ref{ass:theta}, \ref{ass:p}, and \ref{as:stoch_grad}. 
    Let $g_t=\nabla_\pi \mathcal{L}(\theta^t,\pi^t;\zeta_t)$ be the stochastic gradient computed using a mini-batch of size $B$, and let $\xi_t:=g_t-\nabla_\pi \mathcal L(\theta^t,\pi^t)$ be the noise term. 
    Then, Algorithm \ref{algorithm:also_estimator} with tuning
    \[
        \gamma_{\pi}=\frac{\lambda}{8L^2}, 
        \qquad 
        \gamma_\theta \le \frac{c_m b_0}{1048\,L\,\kappa^4},
    \]
    produces a sequence $\{(\theta^t,\pi^t)\}_{t=1}^T$ such that
    \begin{align*}
        \mathbb{E}[D_{\psi}(\pi^*(\theta^{t+1}),\pi^{t+1})] 
        &\leq \Big(1-\tfrac{1}{128\kappa^2}\Big)\,\mathbb{E}[D_{\psi}(\pi^*(\theta^t),\pi^t)] \\
        &\quad + \gamma_\theta^2\,C_{\Phi}\,\mathbb{E}\|\nabla\Phi(\theta^t)\|^2 
        + \gamma_\theta^2\,C_{B}\,\tfrac{\sigma^2}{B} 
        + \gamma_\theta^2\,\beta_1^2 C_{\beta},
    \end{align*}
    where the constants are
    \[
    C_{\Phi} = \frac{2080\,\kappa^6}{c_m^2 b_0^2}, 
    \qquad 
    C_{B} = \frac{1040\,\kappa^6}{c_m^2 b_0^2} + \frac{\lambda^2}{32L^4}, 
    \qquad 
    C_{\beta} = \frac{8320\,\kappa^6}{c_m^2 b_0^2}\Big(K^2+\tfrac{\sigma^2}{B}\Big).
    \]
\end{lemma}
    \begin{proof}
        To begin, we use three-point identity:
        \begin{align}\label{eq:three_point}
        \begin{split}
            D_{\psi}(\pi^*(\theta^{t+1}),\pi^{t+1})=&D_{\psi}(\pi^*(\theta^{t+1}),\pi^*(\theta^{t}))+D_{\psi}(\pi^*(\theta^{t}),\pi^{t+1})\\&+\langle 
    \nabla\psi(\pi^*(\theta^{t}))-\nabla\psi(\pi^{t+1}),\pi^*(\theta^{t+1})-\pi^*(\theta^{t}) \rangle.
        \end{split}
        \end{align}
        Further, we write the optimality condition for the stochastic mirror-ascent step:
        \begin{align*}
            \left\langle -\gamma_{\pi}g_t+[\nabla\psi(\pi^{t+1})-\nabla\psi(\pi^t)],\pi^*(\theta^t)-\pi^{t+1} \right\rangle\geq0.
        \end{align*}
        Applying \eqref{eq:three_point}, we obtain
        \begin{align*}
            -\gamma_{\pi}\left\langle g_t,\pi^*(\theta^t)-\pi^{t+1}\right\rangle+D_{\psi}(\pi^*(\theta^t),\pi^t)-D_{\psi}(\pi^*(\theta^t),\pi^{t+1})-D_{\psi}(\pi^{t+1},\pi^t)\geq0.
        \end{align*}
        Substituting $g_t = \nabla_\pi\mathcal{L}(\theta^t,\pi^t) + \xi_t$, we get:
        \begin{align*}
            -\gamma_{\pi}\left\langle \nabla_\pi\mathcal{L}(\theta^t,\pi^t),\pi^*(\theta^t)-\pi^{t+1}\right\rangle -\gamma_{\pi}\left\langle \xi_t,\pi^*(\theta^t)-\pi^{t+1}\right\rangle+D_{\psi}(\pi^*(\theta^t),\pi^t)-D_{\psi}(\pi^*(\theta^t),\pi^{t+1})-D_{\psi}(\pi^{t+1},\pi^t)\geq0.
        \end{align*}
        After re-arranging the terms, we get
        \begin{align}\label{eq:2_stoch}
            D_{\psi}(\pi^*(\theta^t),\pi^{t+1})\leq D_{\psi}(\pi^*(\theta^t),\pi^t)-D_{\psi}(\pi^{t+1},\pi^t)-\gamma_{\pi}\left\langle\nabla_{\pi}\mathcal{L}(\theta^t,\pi^t),\pi^*(\theta^t)-\pi^{t+1}\right\rangle -\gamma_{\pi}\left\langle \xi_t,\pi^*(\theta^t)-\pi^{t+1}\right\rangle.
        \end{align}
        Since $\pi^*(\theta^t)$ is the exact maximum of $\mathcal{L}(\theta^t,\pi)$ in $\pi$, there is another optimility condition
        \begin{align*}
            \gamma_{\pi}\left\langle \nabla_{\pi}\mathcal{L}(\theta^t,\pi^*(\theta^t)),\pi^*(\theta^t)-\pi \right\rangle\geq0.
        \end{align*}
        Substituting $\pi=\pi^{t+1}$ and summing it with \eqref{eq:2_stoch}, we derive
        \begin{align*}
            D_{\psi}(\pi^*(\theta^t),\pi^{t+1})\leq& D_{\psi}(\pi^*(\theta^t),\pi^t)-D_{\psi}(\pi^{t+1},\pi^t)\\&+\gamma_{\pi}\left\langle\nabla_{\pi}\mathcal{L}(\theta^t,\pi^*(\theta^t))-\nabla_{\pi}\mathcal{L}(\theta^t,\pi^t),\pi^*(\theta^t)-\pi^{t+1}\right\rangle -\gamma_{\pi}\left\langle \xi_t,\pi^*(\theta^t)-\pi^{t+1}\right\rangle\\\leq&D_{\psi}(\pi^*(\theta^t),\pi^t)-D_{\psi}(\pi^{t+1},\pi^t)\\&+\gamma_{\pi}\left\langle\nabla_{\pi}\mathcal{L}(\theta^t,\pi^*(\theta^t))-\nabla_{\pi}\mathcal{L}(\theta^t,\pi^t),\pi^*(\theta^t)-\pi^{t}\right\rangle\\&+\gamma_{\pi}\left\langle\nabla_{\pi}\mathcal{L}(\theta^t,\pi^*(\theta^t))-\nabla_{\pi}\mathcal{L}(\theta^t,\pi^t),\pi^t-\pi^{t+1}\right\rangle -\gamma_{\pi}\left\langle \xi_t,\pi^*(\theta^t)-\pi^{t}\right\rangle -\gamma_{\pi}\left\langle \xi_t,\pi^t-\pi^{t+1}\right\rangle.
        \end{align*}
        Now, we are going to utilize the strong concavity of $\mathcal{L}(\theta,\pi)$ in $\pi$:
        \begin{align*}
            \gamma_{\pi}\left\langle\nabla_{\pi}\mathcal{L}(\theta^t,\pi^*(\theta^t))-\nabla_{\pi}\mathcal{L}(\theta^t,\pi^t),\pi^*(\theta^t)-\pi^{t}\right\rangle\leq\frac{-\gamma_{\pi}\lambda}{2}D_{\psi}(\pi^*(\theta^t),\pi^t).
        \end{align*}
        Thus, we have
        \begin{align*}
            D_{\psi}(\pi^*(\theta^t),\pi^{t+1})\leq&\left(1-\frac{\gamma_{\pi}\lambda}{2}\right)D_{\psi}(\pi^*(\theta^t),\pi^t)-D_{\psi}(\pi^{t+1},\pi^t)\\&+\gamma_{\pi}\left\langle\nabla_{\pi}\mathcal{L}(\theta^t,\pi^*(\theta^t))-\nabla_{\pi}\mathcal{L}(\theta^t,\pi^t),\pi^t-\pi^{t+1}\right\rangle -\gamma_{\pi}\left\langle \xi_t,\pi^*(\theta^t)-\pi^{t+1}\right\rangle .
        \end{align*}
        Next, we apply Cauchy-Schwartz inequality to the scalar product and obtain
        \begin{align*}
            D_{\psi}(\pi^*(\theta^t),\pi^{t+1})\leq\;&\left(1-\frac{\gamma_{\pi}\lambda}{2}\right)D_{\psi}(\pi^*(\theta^t),\pi^t)-D_{\psi}(\pi^{t+1},\pi^t)\\
            &+\frac{\gamma_{\pi}\alpha}{2}\|\nabla_{\pi}\mathcal{L}(\theta^t,\pi^*(\theta^t))-\nabla_{\pi}\mathcal{L}(\theta^t,\pi^t)\|^2
            +\frac{\gamma_{\pi}}{2\alpha}\|\pi^t-\pi^{t+1}\|^2 \\
            &-\gamma_{\pi}\left\langle \xi_t,\pi^*(\theta^t)-\pi^{t}\right\rangle
            \;-\gamma_{\pi}\left\langle \xi_t,\pi^t-\pi^{t+1}\right\rangle.
        \end{align*}
        For the stochastic noise terms, we apply Young's inequality in Bregman geometry:
        \begin{align*}
            -\gamma_{\pi}\left\langle \xi_t,\pi^t-\pi^{t+1}\right\rangle 
            \;\leq\; \gamma_{\pi}^2\|\xi_t\|_*^2 + \tfrac{1}{2}D_{\psi}(\pi^{t+1},\pi^t).
        \end{align*}
        Using $L$-smoothness of $\mathcal{L}$ (see Assumption \ref{ass:theta}) and $\psi$ is $1$-strongly convex (see Assumption \ref{ass:p}), we obtain
        \begin{align*}
            D_{\psi}(\pi^*(\theta^t),\pi^{t+1})\leq\;&\left(1-\frac{\gamma_{\pi}\lambda}{2}\right)D_{\psi}(\pi^*(\theta^t),\pi^t)-D_{\psi}(\pi^{t+1},\pi^t) + \frac{1}{2}D_{\psi}(\pi^{t+1},\pi^t) \\
            &+\gamma_{\pi}\alpha L^2 D_{\psi}(\pi^*(\theta^t),\pi^t) + \frac{\gamma_{\pi}}{\alpha} D_{\psi}( \pi^{t+1},\pi^t) \\
            &-\gamma_{\pi}\left\langle \xi_t,\pi^*(\theta^t)-\pi^{t}\right\rangle
            \;+\;\gamma_{\pi}^2\|\xi_t\|_*^2.
        \end{align*}

        Choose $\alpha=2\gamma_{\pi}$. Substituting this into the previous inequality and reducing terms $D_{\psi}(\pi^{t+1},\pi^t)$, we get
        \begin{align*}
            D_{\psi}(\pi^*(\theta^t),\pi^{t+1})\leq\;&\left(1-\frac{\gamma_{\pi}\lambda}{2}\right)D_{\psi}(\pi^*(\theta^t),\pi^t) \\
            &+2\gamma_{\pi}^2 L^2 D_{\psi}(\pi^*(\theta^t),\pi^t)\\
            &-\gamma_{\pi}\left\langle \xi_t,\pi^*(\theta^t)-\pi^{t}\right\rangle
            \;+\;\gamma_{\pi}^2\|\xi_t\|_*^2.
        \end{align*}

        Taking conditional expectation $\mathbb{E}[\cdot \mid \mathcal{F}_t]$ and using
        $\mathbb{E}[\langle \xi_t,\pi^*(\theta^t)-\pi^t\rangle \mid \mathcal{F}_t]=0$, we obtain
        \begin{align}
            \mathbb{E}\!\left[D_{\psi}(\pi^*(\theta^{t}),\pi^{t+1}) \mid \mathcal{F}_t\right]
            \leq\;&\left(1-\frac{\gamma_{\pi}\lambda}{2}+2\gamma_{\pi}^2L^2\right)
            D_{\psi}(\pi^*(\theta^t),\pi^t) \;+\;\gamma_{\pi}^2\,\frac{\sigma^2}{B}.
        \end{align}

        The stepsize that minimizes the quadratic factor is
        \[
            \gamma_{\pi} = \frac{\lambda}{8L^2}.
        \]
        
        Substituting this choice and applying full expectation yields
        \begin{align}
            \label{eq:4_stoch}
            \mathbb{E}\!\left[D_{\psi}(\pi^*(\theta^{t}),\pi^{t+1}) \right]
            \leq
            \left(1-\frac{1}{32\kappa^2}\right) \mathbb{E}\!\left[D_{\psi}(\pi^*(\theta^t),\pi^t)\right]
            + \frac{\lambda^2}{64L^4}\,\frac{\sigma^2}{B},
        \end{align}
        where $\kappa = \tfrac{L}{\lambda}$ is the condition number.

        Let us return to \eqref{eq:three_point}. Note that
        \begin{align*}
            \nabla\psi(\pi^*(\theta^t))-\nabla\psi(\pi^{t+1})=\frac{1}{\lambda}\left(\nabla_{\pi}\mathcal{L}(\theta^t,\pi^{t+1})-\nabla_{\pi}\mathcal{L}(\theta^t,\pi^*(\theta^t))\right).
        \end{align*}
        Thus, there is
        \begin{align*}
                D_{\psi}(\pi^*(\theta^{t+1}),\pi^{t+1})=&D_{\psi}(\pi^*(\theta^{t+1}),\pi^*(\theta^{t}))+D_{\psi}(\pi^*(\theta^{t}),\pi^{t+1})\\&+\frac{1}{\lambda}\langle 
        \nabla_{\pi}\mathcal{L}(\theta^t,\pi^{t+1})-\nabla_{\pi}\mathcal{L}(\theta^t,\pi^*(\theta^t)),\pi^*(\theta^{t+1})-\pi^*(\theta^{t}) \rangle\\\leq&D_{\psi}(\pi^*(\theta^{t+1}),\pi^*(\theta^{t}))+D_{\psi}(\pi^*(\theta^{t}),\pi^{t+1})\\&+\frac{\alpha L^2}{\lambda}D_{\psi}(\pi^*(\theta^t),\pi^{t+1})+\frac{1}{\lambda\alpha}D_{\psi}(\pi^*({\theta^{t+1}}),\pi^*(\theta^t)).
        \end{align*}
        Let us choose $\alpha=\nicefrac{\lambda^3}{64L^4}$. With such a choice and using fact that $\kappa \ge 1$, we have
        \begin{align*}
            D_{\psi}(\pi^*(\theta^{t+1}),\pi^{t+1})\leq65\kappa^4D_{\psi}(\pi^*({\theta^{t+1}}),\pi^*(\theta^t))+\left(1+\frac{1}{64\kappa^2}\right)D_{\psi}(\pi^*(\theta^t),\pi^{t+1}).
        \end{align*}
        To deal with $D_{\psi}(\pi^*(\theta^t),\pi^{t+1})$, we utilize \eqref{eq:4_stoch}. Using $(1 + \frac{1}{64\kappa^2})(1 - \frac{1}{32\kappa^2}) \le 1 - \frac{1}{64\kappa^2}$ and $1 + \frac{1}{64\kappa^2} \le 2$ we obtain
        \begin{align}\label{eq:7_stoch}
            \mathbb{E}\!\left[D_{\psi}(\pi^*(\theta^{t+1}),\pi^{t+1})\right]\leq65\kappa^4\mathbb{E}\!\left[D_{\psi}(\pi^*({\theta^{t+1}}),\pi^*(\theta^t)) \right]+\left(1-\frac{1}{64\kappa^2}\right)\mathbb{E}\!\left[D_{\psi}(\pi^*(\theta^t),\pi^{t})\right] + \frac{\lambda^2}{32L^4}\,\frac{\sigma^2}{B}.
        \end{align}
        The remaining task is to prove that the descent step does not dramatically change the distance between the optimal values of weights.
        Let us write down two optimality conditions:
        \begin{align*}
            &\langle \nabla_{\pi}\mathcal{L}(\theta^t,\pi^*(\theta^t)),\pi-\pi^*(\theta^t) \rangle\leq0,\\
            &\langle \nabla_{\pi}\mathcal{L}(\theta^{t+1},\pi^*(\theta^{t+1})),\pi-\pi^*(\theta^{t+1}) \rangle\leq0.
        \end{align*}
        Let us substitute $\pi=\pi^*(\theta^{t+1})$ into the first inequality and $\pi=\pi^*(\theta^{t})$ into the second one. When summing them up, we have
        \begin{align}\label{eq:5_stoch}
            \langle \nabla_{\pi}\mathcal{L}(\theta^t,\pi^*(\theta^t))-\nabla_{\pi}\mathcal{L}(\theta^{t+1},\pi^*(\theta^{t+1})), \pi^*(\theta^{t+1})-\pi^*(\theta^t) \rangle\leq0.
        \end{align}
        On the other hand, we can take advantage of the strong concavity of the objective (see Lemma \ref{eq:lemma_1}) and write
        \begin{align}\label{eq:6_stoch}
            &\langle \nabla_{\pi}\mathcal{L}(\theta^t,\pi^*(\theta^{t+1}))-\nabla_{\pi}\mathcal{L}(\theta^t,\pi^*(\theta^t)),\pi^*(\theta^{t+1})-\pi^*(\theta^{t}) \rangle\\&\leq-\frac{\lambda}{2}\left[D_{\psi}(\pi^*(\theta^t),\pi^*(\theta^{t+1}))+D_{\psi}(\pi^*(\theta^{t+1}),\pi^*(\theta^t))\right].
        \end{align}
        Combining \eqref{eq:5_stoch} and \eqref{eq:6_stoch}, we obtain
        \begin{align*}
            \frac{\lambda^2}{4}\left[D_{\psi}(\pi^*(\theta^t),\pi^*(\theta^{t+1}))+D_{\psi}(\pi^*(\theta^{t+1}),\pi^*(\theta^t))\right]^2\leq L^2\|\pi^*(\theta^{t+1})-\pi^*(\theta^{t})\|^2\|\theta^{t+1}-\theta^t\|^2.
        \end{align*}
        Re-arranging the terms and substituting Adam estimator step, we derive
        \begin{align*}
            \left[D_{\psi}(\pi^*(\theta^t),\pi^*(\theta^{t+1}))+D_{\psi}(\pi^*(\theta^{t+1}),\pi^*(\theta^t))\right]\leq& 4\kappa^2\|\theta^{t+1}-\theta^t\|^2 \equiv 4\gamma_{\theta}^2\kappa^2\left\| d_\theta^t \right\|^2.
        \end{align*}
        After simplifying, we have
        \begin{align*}
            D_{\psi}(\pi^*(\theta^{t+1}),\pi^*(\theta^t))\leq&4\gamma_{\theta}^2\kappa^2 \left\|d_\theta^t\right\|^2.
        \end{align*}

        Using lemma~\ref{lemma:momentum_upper_bound} and lemma~\ref{properties:adam-estimator}:
        \begin{align} \label{eq:adam_step_upper_bound}
            \| d_\theta^t \|^2 = \| \frac{m_\theta^t}{b_t} \|^2 \le \frac{1}{c_m^2b_0^2} \| m_\theta^t\|^2 &\le \frac{4}{c_m^2b_0^2} \left( \|g_\theta^t - m_\theta^t \|^2 + \| \nabla_\theta \mathcal{L}(\theta^t, \pi^t)\|^2 + \|\xi_t\|^2\right) \\ 
            &\le \frac{4}{c_m^2b_0^2} \left( \beta_1^2 \cdot G_t + \| \nabla_\theta \mathcal{L}(\theta^t, \pi^t)\|^2 + \|\xi_t\|^2\right),
        \end{align}
        where $\xi_t = \nabla_\theta\mathcal{L}(\theta^t, \pi^t) - g_\theta^t$ is the stochastic gradient noise, $G_t = 2\left(\|g_\theta^t\|^2 + (1 - \beta_1) \sum_{k=0}^{t-1} \beta_1^{t-k} \|g_\theta^k\|^2\right)$.        

        Using $L$-smoothness of $\mathcal{L}$ (see Assumption \ref{ass:theta}) and $\psi$ is $1$-strongly convex (see Assumption \ref{ass:p}), we obtain
        \begin{align*}
            \| \nabla_\theta \mathcal{L}(\theta^t, \pi^t)\|^2 &\le 2 \left(\| \nabla\Phi(\theta^t)\|^2 + \| \nabla_\theta \mathcal{L}(\theta^t, \pi^t) - \nabla\Phi(\theta^t)\|^2\right) \\
            &\le 2\| \nabla\Phi(\theta^t)\|^2 + 4L^2D_{\psi}(\pi^*(\theta^t), \pi^t)    
        \end{align*}
        
        Applying expectation and using assumption~\ref{as:stoch_grad} we have:
        \begin{align} \label{eq:adam_step_upper_bound_expect}
            \mathbb{E}\| d_\theta^t \|^2 
            &\le \frac{4}{c_m^2b_0^2} \Big(
                \beta_1^2 \cdot \mathbb{E}[G_t] 
                + 2 \, \mathbb{E}\|\nabla\Phi(\theta^t)\|^2 
                + 4L^2 \, \mathbb{E}\big[D_{\psi}(\pi^*(\theta^t), \pi^t)\big] 
                + \frac{\sigma^2}{B}
            \Big).
        \end{align}
 
        Setting $\tau=0$ and using $K$-Lipschitzness~\ref{lem:K-lipschitz} of $\mathcal{L}$ and boundess of variance~\ref{as:stoch_grad}, we have
        \begin{align} \label{neq:G_estimation}
            \|g_\theta^k\|^2 \le 2K^2 + \tfrac{2\sigma^2}{B}
            \quad \Rightarrow \quad 
            \mathbb{E}[G_t] \le 8K^2 + \tfrac{8\sigma^2}{B}.
        \end{align}

        After substituting inequality~\ref{neq:G_estimation}
        into \ref{eq:adam_step_upper_bound_expect} we obtain   \begin{align}\label{eq:adam_step_upper_bound_expect_final}
            \mathbb{E}\| d_\theta^t \|^2 
            &= \frac{4}{c_m^2 b_0^2} \Big(
                \beta_1^2 \cdot 8(K^2 + \frac{\sigma^2}{B})
                + 2\,\mathbb{E}\|\nabla\Phi(\theta^t)\|^2 
                + 4L^2\,\mathbb{E}[D_{\psi}(\pi^*(\theta^t), \pi^t)] 
                + \frac{\sigma^2}{B}
            \Big).
        \end{align}
    
        Let us take an expectation and derive
        \begin{align*}
            \E D_{\psi}(\pi^*(\theta^{t+1}),\pi^*(\theta^t)) \leq \frac{16\gamma_{\theta}^2\kappa^2}{c_m^2 b_0^2} \Big( 8\beta_1^2\Big(K^2 + \tfrac{\sigma^2}{B}\Big) + 2\,\E\|\nabla\Phi(\theta^t)\|^2 + 4L^2\,\E[D_{\psi}(\pi^*(\theta^t), \pi^t)] + \tfrac{\sigma^2}{B} \Big).
        \end{align*}
        
        Substituting this into \eqref{eq:7_stoch} we have 
        \begin{align*}
            \mathbb{E}[D_{\psi}(\pi^*(\theta^{t+1}),\pi^{t+1})] &\leq \frac{1040\,\gamma_{\theta}^2\kappa^6}{c_m^2 b_0^2}\Big( 8\beta_1^2\!\big(K^2+\tfrac{\sigma^2}{B}\big) + 2\,\mathbb{E}\|\nabla\Phi(\theta^t)\|^2 + 4L^2\,\mathbb{E}[D_{\psi}(\pi^*(\theta^t),\pi^t)] + \tfrac{\sigma^2}{B} \Big) \\
            &+ \Big(1-\tfrac{1}{64\kappa^2}\Big)\mathbb{E}[D_{\psi}(\pi^*(\theta^t),\pi^t)] + \tfrac{\lambda^2}{32L^4}\,\tfrac{\sigma^2}{B}.
        \end{align*}

        Using $\gamma_\theta \le \tfrac{c_m b_0}{1048\,L\,\kappa^4}$ and substituting \eqref{eq:adam_step_upper_bound_expect_final} into \eqref{eq:7_stoch}, we have
        \begin{align*}
            \mathbb{E}\!\left[D_{\psi}(\pi^*(\theta^{t+1}),\pi^{t+1})\right] 
            &\le \Big(1-\tfrac{1}{128\kappa^2}\Big)\,\mathbb{E}\!\left[D_{\psi}(\pi^*(\theta^t),\pi^t)\right] \\
            &\quad + \frac{1040\,\gamma_{\theta}^2\kappa^6}{c_m^2 b_0^2}\Big( 8\beta_1^2\!\big(K^2+\tfrac{\sigma^2}{B}\big) + 2\,\mathbb{E}\|\nabla\Phi(\theta^t)\|^2 + \tfrac{\sigma^2}{B} \Big) \\
            &\quad + \tfrac{\lambda^2}{32L^4}\,\tfrac{\sigma^2}{B}.
        \end{align*}
        
        Collecting terms, we obtain
        \begin{align*}
            \mathbb{E}[D_{\psi}(\pi^*(\theta^{t+1}),\pi^{t+1})] 
            &\leq \Big(1-\tfrac{1}{128\kappa^2}\Big)\,\mathbb{E}[D_{\psi}(\pi^*(\theta^t),\pi^t)] \\
            &\quad + \gamma_\theta^2\,C_{\Phi}\,\mathbb{E}\|\nabla\Phi(\theta^t)\|^2 
            + \gamma_\theta^2\,C_{B}\,\tfrac{\sigma^2}{B} 
            + \gamma_\theta^2\,\beta_1^2 C_{\beta},
        \end{align*}
        where the constants are
        \[
        C_{\Phi} = \frac{2080\,\kappa^6}{c_m^2 b_0^2}, 
        \qquad 
        C_{B} = \frac{1040\,\kappa^6}{c_m^2 b_0^2} + \frac{\lambda^2}{32L^4}, 
        \qquad 
        C_{\beta} = \frac{8320\,\kappa^6}{c_m^2 b_0^2}\Big(K^2+\tfrac{\sigma^2}{B}\Big).
        \]
            
        This completes the proof of the stochastic version of the main lemma.
    \end{proof}

% Main theorem section
% This file contains the main theorem used in the Theory for ALSO section

%%%%%% ---MAIN THEOREM-------MAIN THEOREM-------MAIN THEOREM-------MAIN THEOREM-------MAIN THEOREM-------MAIN THEOREM-------MAIN THEOREM-------MAIN THEOREM-------MAIN THEOREM-------MAIN THEOREM-------MAIN THEOREM-------MAIN THEOREM-------

\subsubsection{Main theorem}

Now let us proceed to the convergence proof for Algorithm \ref{algorithm:also_optimistic}.
% \begin{theorem}\label{th:mu_concave_convergence}
%     Consider the problem \eqref{eq:pinn_saddle} under Assumptions~\ref{as:lipgrad}, \ref{as:stoch_grad} and \ref{as:uncertainty_set}. 
%     Then, Algorithm~\ref{algorithm:also_optimistic} with tuning
%     \begin{align*}
%         \gamma_{\pi} &= \frac{\lambda}{4L^2}, 
%         \qquad 
%         \gamma_{\theta} \;\leq\; \min\!\left\{ \tfrac{\varepsilon}{\sqrt{6A_3}}, \;\tfrac{\varepsilon^2}{6A_6} \right\}, 
%         \qquad 
%         B \;=\; \max\!\left\{ 1,\; \tfrac{6A_4}{\varepsilon^2}(\gamma_\theta^\star)^2\sigma^2,\; \tfrac{6A_5\sigma^2}{\varepsilon^2} \right\}
%     \end{align*}
%     requires at most
%     \begin{align*}
%         \mathcal{O}\!\left(
%             \frac{L\Delta_\Phi}{(1-\beta_1)\varepsilon^2}
%             + \frac{L D_{\psi}(\pi^*(\theta^0),\pi^0)}{(1-\beta_1)\varepsilon^2}
%             + \frac{\beta_1 L(\sigma^2+K^2)\Delta_\Phi}{(1-\beta_1)^3\varepsilon^4}
%         \right)
%     \end{align*}
%     iterations to achieve an arbitrary $\varepsilon$-solution, where the accuracy criterion is
%     \[
%         \frac{1}{T}\sum_{t=0}^{T-1}\E \|\nabla \Phi(\theta^t)\|^2 \;\leq\; \varepsilon^2,
%     \]
%     and 
%     \[
%         \Delta_\Phi = \Phi_{1/2L}(\theta^0)-\min_{\theta}\Phi_{1/2L}(\theta).
%     \]
%     Here $K,L,\beta_1$ are constants defined in Appendix~\ref{appendix:also-theory}.
% \end{theorem}
\begin{proof} \ref{also:convergence-new}
    One can note that $\Phi$ is $3\kappa L$-smooth. Indeed,
    \begin{align*}
        \|\nabla\Phi(\theta_1)-\nabla\Phi(\theta_2)\|^2=&\|\nabla_{\theta}\mathcal{L}(\theta_1,\pi^*(\theta_1))-\nabla_{\theta}\mathcal{L}(\theta_2,\pi^*(\theta_2))\|^2\\\leq&  L^2\left[\|\theta_1-\theta_2\|^2+2D_{\psi}(\pi^*(\theta_1),\pi^*(\theta_2))\right]\leq L^2\left(1+4\kappa^2\right)\|\theta_1-\theta_2\|^2\\\leq&9\kappa^2L^2\|\theta_1-\theta_2\|^2.
    \end{align*}
    Thus, we can write
    \begin{align*}
        \Phi(\theta^{t+1})\leq&\Phi(\theta^t)+\langle \nabla\Phi(\theta^t), \theta^{t+1}-\theta^t \rangle+3\kappa L\|\theta^{t+1}-\theta^t\|^2\\=&\Phi(\theta^t)-\gamma_{\theta}\left\langle \nabla\Phi(\theta^t),d_{\theta}^t \right\rangle+3\gamma_{\theta}^2\kappa L\left\|d_{\theta}^t\right\|^2
    \end{align*}

    Summing from $t=0$ to $T$ yields
    \begin{align*}
        \Phi(\theta^{T+1})
        &\leq \Phi(\theta^0) 
        - \gamma_{\theta}\sum_{t=0}^T \langle \nabla\Phi(\theta^t), d_{\theta}^t \rangle
        + 3\gamma_{\theta}^2 \kappa L \sum_{t=0}^T \|d_{\theta}^t\|^2.
    \end{align*}
    
    Applying lemma \ref{lemma:sums-estimation} with $a_t = -\left\langle \nabla\Phi(\theta^t),d_{\theta}^t \right\rangle$ we have:
    \begin{align*}
        \Phi(\theta^{T+1})
        &\leq \Phi(\theta^0) 
        + \gamma_{\theta}\sum_{k=0}^{T} C_k \xi_k
        + 3\gamma_{\theta}^2 \kappa L \sum_{k=0}^{T} (1 + A_k) \| d_{\theta}^k \|^2,
    \end{align*}
    where $\xi_k = -\langle \nabla \Phi(\theta^k), g_{\theta}^k \rangle$ and $g_\theta^k$ is the stochastic gradient in the Adam estimator ~\ref{algorithm:also_estimator}.
    
    By decomposing the stochastic gradient into the true gradient and the noise $g_\theta^k = \nabla_{\theta}\mathcal{L}(\theta^k,\pi^k) + \eta_k $, we have
    \begin{align*}
        \Phi(\theta^{T+1})
        &\leq \Phi(\theta^0) 
        - \gamma_{\theta}\sum_{k=0}^{T} C_k 
            \left\langle \nabla \Phi(\theta^k), \nabla_{\theta}\mathcal{L}(\theta^k,\pi^k)\right\rangle \\
        &\quad - \gamma_{\theta}\sum_{k=0}^{T} C_k 
            \left\langle \nabla \Phi(\theta^k), \eta_k \right\rangle
        + 3\gamma_{\theta}^2 \kappa L \sum_{k=0}^{T} (1 + A_k) \| d_{\theta}^k \|^2.
    \end{align*}
    
    % By adding and subtracting $\|\nabla \Phi(\theta^k)\|^2$:
    % \begin{align*}
    %     \Phi(\theta^{T+1})
    %     &\leq \Phi(\theta^0) - \gamma_{\theta}\sum_{k=0}^{T} C_k 
    %         \Big(
    %             \left\langle \nabla \Phi(\theta^k), 
    %             \nabla_{\theta}\mathcal{L}(\theta^k,\pi^k) - \nabla \Phi(\theta^k)\right\rangle
    %             + \|\nabla \Phi(\theta^k)\|^2
    %         \Big) \\
    %     &\quad - \gamma_{\theta}\sum_{k=0}^{T} C_k 
    %         \left\langle \nabla \Phi(\theta^k), \eta_k \right\rangle
    %     + 3\gamma_{\theta}^2 \kappa L \sum_{k=0}^{T-1} (1 + A_k) \| d_{\theta}^k \|^2.
    % \end{align*}
    
    % Applying Young’s inequality to the scalar product,
    % \[
    %     -\left\langle \nabla \Phi(\theta^k), 
    %     \nabla_{\theta}\mathcal{L}(\theta^k,\pi^k) - \nabla \Phi(\theta^k)\right\rangle
    %     \leq \tfrac{1}{2}\|\nabla \Phi(\theta^k)\|^2
    %     + \tfrac{1}{2}\|\nabla_{\theta}\mathcal{L}(\theta^k,\pi^k) - \nabla \Phi(\theta^k)\|^2,
    % \]
    % and by smoothness of $\mathcal{L}$ and the definition of $\pi^*(\theta^k)$,
    % \[
    %     \|\nabla_{\theta}\mathcal{L}(\theta^k,\pi^k) - \nabla \Phi(\theta^k)\|^2
    %     \leq 2L^2\,D_{\psi}(\pi^*(\theta^k),\pi^k).
    % \]
    
    % Therefore,
    % \begin{align*}
    %     \Phi(\theta^{T+1})
    %     &\leq \Phi(\theta^0) 
    %     - \tfrac{\gamma_{\theta}}{2}\sum_{k=0}^{T} C_k \|\nabla \Phi(\theta^k)\|^2 + \gamma_{\theta}L^2 \sum_{k=0}^{T} C_k D_{\psi}(\pi^*(\theta^k),\pi^k) \\
    %     &\quad - \gamma_{\theta}\sum_{k=0}^{T} C_k \langle \nabla \Phi(\theta^k), \eta_k \rangle
    %     + 3\gamma_{\theta}^2 \kappa L \sum_{k=0}^{T-1} (1+A_k)\|d_{\theta}^k\|^2.
    % \end{align*}
    
    Rearranging the terms and dividing by $\gamma_{\theta}$ yields
    \begin{align}
        \sum_{k=0}^{T} C_k 
            \left\langle \nabla \Phi(\theta^k), \nabla_{\theta}\mathcal{L}(\theta^k,\pi^k)\right\rangle
        \;\le\; &\;\frac{\Phi(\theta^0)-\Phi(\theta^{T+1})}{\gamma_{\theta}} - \sum_{k=0}^{T} C_k \left\langle \nabla \Phi(\theta^k), \eta_k \right\rangle
        + 3\gamma_{\theta}\kappa L \sum_{k=0}^{T-1} (1+A_k)\|d_{\theta}^k\|^2. \label{eq:pre-exp}
    \end{align}
    
    Let $\mathcal F_k$ denote the history of the main process up to time $k$, and let the coefficients $C_k=(1-\beta_1)\sum_{j=k}^T \beta_1^{\,j-k}/b_j$ be generated by an auxiliary (copy) sequence $\{b_j\}_{j\ge 0}$. Since $C_k$ depends only on future $\{b_j\}_{j\ge k}$ from the copy process, while $r_k\!:=\!\langle \nabla \Phi(\theta^k), \eta_k\rangle$ is generated by the main process at time $k$, we have the conditional independence of $C_k$ and $r_k$ with respect to $(\mathcal F_k,\text{copy})$. Using the unbiasedness $\mathbb E[\eta_k\mid \mathcal F_k]=0$, the tower property gives
    \[
        \mathbb E\!\left[C_k\, r_k\right]
        = \mathbb E\!\left[\,\mathbb E\!\left[C_k\, r_k \mid \mathcal F_k,\text{copy}\right]\right]
        = \mathbb E\!\left[\,\mathbb E[C_k\mid \mathcal F_k,\text{copy}]\,\mathbb E[r_k\mid \mathcal F_k]\right]=0.
    \]
    Taking conditional expectation of \eqref{eq:pre-exp} and then applying the tower property, we obtain
    \begin{align}
        \frac{1}{2}\sum_{k=0}^{T} \mathbb E\!\left[C_k 
            \left\langle \nabla \Phi(\theta^k), \nabla_{\theta}\mathcal{L}(\theta^k,\pi^k)\right\rangle\right]
        \;\le\;&\; \frac{\Phi(\theta^0)-\mathbb E\,\Phi(\theta^{T+1})}{\gamma_{\theta}}
        + 3\gamma_{\theta}\kappa L \sum_{k=0}^{T-1} \mathbb E\!\left[(1+A_k)\|d_{\theta}^k\|^2\right].
        \label{eq:exp-bound}
    \end{align}
    
    To separate the factors on the left, use conditional independence as above:
    \begin{align*}
        \mathbb E\!\left[C_k 
            \left\langle \nabla \Phi(\theta^k), \nabla_{\theta}\mathcal{L}(\theta^k,\pi^k)\right\rangle \mid \mathcal F_k,\text{copy}\right]
        = \mathbb E\!\left[C_k \mid \text{copy}\right]\cdot \left\langle \nabla \Phi(\theta^k), \nabla_{\theta}\mathcal{L}(\theta^k,\pi^k)\right\rangle.
    \end{align*}
    Hence
    \[
        \mathbb E\!\left[C_k 
            \left\langle \nabla \Phi(\theta^k), \nabla_{\theta}\mathcal{L}(\theta^k,\pi^k)\right\rangle\right]
        = \mathbb E\!\left[\mathbb E[C_k\mid \text{copy}]\,
            \left\langle \nabla \Phi(\theta^k), \nabla_{\theta}\mathcal{L}(\theta^k,\pi^k)\right\rangle\right].
    \]

    Let us get the bound of the scaling parameter $b_t$ in the Adam estimator ~\ref{algorithm:also_estimator}:
    \begin{align}
        \mathbb{E}\!\left[\|g_\theta^t\|^2 \,\middle|\, \theta^k_{\mathrm{copy}}, \pi^k_{\mathrm{copy}}\right] 
        &\;\le\; 2\,(K^2 + \frac{\sigma^2}{B}), \\[0.75em]
        \mathbb{E}\!\left[b_i \,\middle|\, \theta^k_{\mathrm{copy}}, \pi^k_{\mathrm{copy}}\right] 
        &\;\le\; \mathbb{E}\!\left[\sqrt{\,\beta_2 b_{i-1}^2 + (1-\beta_2)\|\tilde{g}_\theta^t\|^2\,} \,\middle|\, \theta^k_{\mathrm{copy}}, \pi^k_{\mathrm{copy}} \right] \notag \\
        &\;\le\; \mathbb{E}\!\left[\max\{b_{i-1}, \|\tilde{g}_\theta^t\|\} \,\middle|\, \theta^k_{\mathrm{copy}}, \pi^k_{\mathrm{copy}} \right] \notag \\
        &\;\le\; \max_i \sqrt{2K^2 + 2\frac{\sigma^2}{B}} \;=\; \sqrt{2K^2 + 2\frac{\sigma^2}{B}}.
    \end{align}
    
    Using , then
    \[
        \mathbb E[C_k\mid \theta^k_{\mathrm{copy}}, \pi^k_{\mathrm{copy}}]
        = (1-\beta_1)\sum_{j=k}^{T} \frac{\beta_1^{\,j-k}}{\mathbb E[b_j\mid \theta^k_{\mathrm{copy}}, \pi^k_{\mathrm{copy}}]}
        \;\ge\; (1-\beta_1)\min_{j\in\{0,\dots,T\}}\frac{1}{\mathbb E[b_j\mid \theta^k_{\mathrm{copy}}, \pi^k_{\mathrm{copy}}]}
        \;\ge\; \frac{1-\beta_1}{\sqrt{2K^2 + 2\frac{\sigma^2}{B}}}.
    \]
    Therefore,
    \begin{equation}
        \sum_{k=0}^{T} \mathbb E\!\left[C_k 
            \left\langle \nabla \Phi(\theta^k), \nabla_{\theta}\mathcal{L}(\theta^k,\pi^k)\right\rangle\right]
        \;\ge\; \frac{1-\beta_1}{\sqrt{2K^2 + 2\frac{\sigma^2}{B}}}\sum_{k=0}^{T} \mathbb E\!\left[ 
            \left\langle \nabla \Phi(\theta^k), \nabla_{\theta}\mathcal{L}(\theta^k,\pi^k)\right\rangle\right].
        \label{eq:C-lb}
    \end{equation}

    Combining \eqref{eq:exp-bound} and \eqref{eq:C-lb}, we arrive at
    \begin{align*}
        \frac{1-\beta_1}{2\sqrt{2K^2 + 2\frac{\sigma^2}{B}}}\sum_{k=0}^{T} \mathbb E\!\left[ 
            \left\langle \nabla \Phi(\theta^k), \nabla_{\theta}\mathcal{L}(\theta^k,\pi^k)\right\rangle\right]
        \;\le\;&\; \frac{\Phi(\theta^0)-\mathbb E\,\Phi(\theta^{T+1})}{\gamma_{\theta}}
        + 3\gamma_{\theta}\kappa L \sum_{k=0}^{T-1} \mathbb E\!\left[(1+A_k)\|d_{\theta}^k\|^2\right].
    \end{align*}

    Using ~\ref{eq:adam_step_upper_bound_expect_final} we have:
    \begin{align}\label{eq:adam_step_upper_bound_expect_final_final}
        \mathbb{E}\| d_\theta^t \|^2 
        &= \frac{4}{c_m^2 b_0^2} \Big(
            \beta_1^2 \cdot 8(K^2 + \frac{\sigma^2}{B})
            + 2\,\mathbb{E}\|\nabla\Phi(\theta^t)\|^2 
            + 4L^2\,\mathbb{E}[D_{\psi}(\pi^*(\theta^t), \pi^t)] 
            + \frac{\sigma^2}{B}
        \Big).
    \end{align}
    
    By definition of $A_k$:
    \begin{align*}
        \mathbb E A_t &\le \tfrac{\beta_1}{c_m b_0(1-\beta_1)} \sqrt{2K^2 + 2\frac{\sigma^2}{B}}, \\
        \mathbb E\!\left[(1+A_t)\|d_\theta^t\|^2\right]
        \;&\leq\;
        \Biggl(1 + \frac{\beta_1}{c_m b_0(1-\beta_1)}\sqrt{2K^2 + 2\frac{\sigma^2}{B}}\Biggr) \\
        &\cdot\frac{4}{c_m^2 b_0^2} \Big(
            \beta_1^2 \cdot 8(K^2 + \frac{\sigma^2}{B})
            + 2\,\mathbb{E}\|\nabla\Phi(\theta^t)\|^2 
            + 4L^2\,\mathbb{E}[D_{\psi}(\pi^*(\theta^t), \pi^t)] 
            + \frac{\sigma^2}{B}
        \Big).
    \end{align*}

    \begin{align*}
    C_A &:= \frac{\beta_1}{c_m b_0(1-\beta_1)}\sqrt{2K^2 + 2\frac{\sigma^2}{B}}, 
    \qquad 
    C_D := \frac{4}{c_m^2 b_0^2}.
    \end{align*}
    
    Then the auxiliary bounds read
    \begin{align*}
    \mathbb E A_t &\le C_A, \\
    \mathbb E\!\left[(1+A_t)\|d_\theta^t\|^2\right]
    &\le (1+C_A)\, C_D \Big(
        \beta_1^2 \cdot 8\!\Big(K^2 + \tfrac{\sigma^2}{B}\Big)
        + 2\,\mathbb{E}\|\nabla\Phi(\theta^t)\|^2 
        + 4L^2\,\mathbb{E}[D_{\psi}(\pi^*(\theta^t), \pi^t)] 
        + \tfrac{\sigma^2}{B}
    \Big).
    \end{align*}

    Substituting these inequalities into the main relation yields
    \begin{align*}
    &\frac{1-\beta_1}{2\sqrt{2K^2 + 2\frac{\sigma^2}{B}}}\sum_{k=0}^{T} 
    \mathbb E\!\left[ 
        \left\langle \nabla \Phi(\theta^k), \nabla_{\theta}\mathcal{L}(\theta^k,\pi^k)\right\rangle\right]
    \le \frac{\Phi(\theta^0)-\mathbb E\,\Phi(\theta^{T+1})}{\gamma_{\theta}}  \\
    &+ 3\gamma_{\theta}\kappa L \sum_{k=0}^{T-1} 
    (1+C_A)\, C_D \Big(
        \beta_1^2 \cdot 8\!\Big(K^2 + \tfrac{\sigma^2}{B}\Big) 
        + 2\,\mathbb{E}\|\nabla\Phi(\theta^k)\|^2 
        + 4L^2\,\mathbb{E}[D_{\psi}(\pi^*(\theta^k), \pi^k)] 
        + \tfrac{\sigma^2}{B}
    \Big).
    \end{align*}

    Applying Young’s inequality to the scalar product and by smoothness of $\mathcal{L}$ and the definition of $\pi^*(\theta^k)$:
    \begin{align*}
        \mathbb E\!\left[ 
        \left\langle \nabla \Phi(\theta^k), \nabla_{\theta}\mathcal{L}(\theta^k,\pi^k)\right\rangle\right] \ge \frac{1}{2}\mathbb E\!\left[ 
        \|\nabla \Phi(\theta^k) \|^2\right] - L^2\mathbb E\!\left[ D_{\psi}(\pi^*(\theta^k), \pi^k) \right]
    \end{align*}
    
    Therefore,
    \begin{align*}
        &\frac{1-\beta_1}{2\sqrt{2K^2 + 2\frac{\sigma^2}{B}}}\sum_{k=0}^{T} 
        \left(\frac{1}{2}\mathbb E\!\left[ 
        \|\nabla \Phi(\theta^k) \|^2\right] - L^2\mathbb E\!\left[ D_{\psi}(\pi^*(\theta^k), \pi^k) \right] \right) \le \frac{\Phi(\theta^0)-\mathbb E\,\Phi(\theta^{T+1})}{\gamma_{\theta}} \\ 
        &+ 3\gamma_{\theta}\kappa L \sum_{k=0}^{T-1} 
        (1+C_A)\, C_D \Big(
            \beta_1^2 \cdot 8\!\Big(K^2 + \tfrac{\sigma^2}{B}\Big)
            + 2\,\mathbb{E}\|\nabla\Phi(\theta^k)\|^2 
            + 4L^2\,\mathbb{E}[D_{\psi}(\pi^*(\theta^k), \pi^k)] 
            + \tfrac{\sigma^2}{B}
        \Big).
    \end{align*}

    Using 
    \[
        \gamma_\theta \;\le\; 
        \frac{1-\beta_1}{72\,\kappa L (1+C_A) C_D \,\sqrt{2K^2+2\sigma^2/B}},
    \]
    we have
    \begin{align*}
        &\frac{1-\beta_1}{2\sqrt{2K^2 + 2\frac{\sigma^2}{B}}}\sum_{k=0}^{T} 
        \tfrac{1}{3}\,\mathbb E\!\left[ 
        \|\nabla \Phi(\theta^k) \|^2\right] \le \frac{1-\beta_1}{2\sqrt{2K^2 + 2\frac{\sigma^2}{B}}}8L^2\,\sum_{k=0}^{T}\mathbb{E}[D_{\psi}(\pi^*(\theta^k), \pi^k)] \\
        &+  
        \frac{\Phi(\theta^0)-\mathbb E\,\Phi(\theta^{T+1})}{\gamma_{\theta}}
        + 3\gamma_{\theta}\kappa L \sum_{k=0}^{T-1} 
        (1+C_A)\, C_D \Big(
            \beta_1^2 \cdot 8\!\Big(K^2 + \tfrac{\sigma^2}{B}\Big) + \tfrac{\sigma^2}{B}
        \Big).
    \end{align*}    
    
    Simplifying our inequality we obtain:
    \begin{align*}
        \frac{1}{T+1}\sum_{k=0}^{T} 
        \mathbb{E}\!\left[\|\nabla \Phi(\theta^k)\|^2\right]
        &\;\le\;
        M_1 \,\frac{1}{T+1}\sum_{k=0}^{T}\mathbb{E}\!\left[D_{\psi}(\pi^*(\theta^k),\pi^k)\right] 
        + M_2 \,\frac{\Phi(\theta^0)-\mathbb{E}\,\Phi(\theta^{T+1})}{(T+1)\gamma_\theta}
        + M_3 \,\gamma_\theta ,
    \end{align*}
    where
    \begin{align*}
        M_1 &= 24L^2, \\[4pt]
        M_2 &= \frac{6\sqrt{2K^2 + 2\sigma^2/B}}{(1-\beta_1)}, \\[4pt]
        M_3 &= \frac{18\,\kappa L\sqrt{2K^2 + 2\sigma^2/B}}{(1-\beta_1)}\,
        (1+C_A)\,C_D\Big( 8\beta_1^2(K^2+\tfrac{\sigma^2}{B}) + \tfrac{\sigma^2}{B}\Big).
    \end{align*}

    % Moreover, if $b_j\ge c_m b_0$ a.s. for all $j$ and some $c_m\in(0,1]$, then
    % \[
    %     C_k \le (1-\beta_1)\frac{1}{c_m b_0}\sum_{j=k}^{T}\beta_1^{\,j-k} \le \frac{1}{c_m b_0},
    %     \quad\text{hence}\quad
    %     \mathbb E\!\left[C_k D_{\psi}\!\left(\pi^*(\theta^k),\pi^k\right)\right]
    %     \le \frac{1}{c_m b_0}\,\mathbb E\!\left[D_{\psi}\!\left(\pi^*(\theta^k),\pi^k\right)\right].
    % \]
    % This is the desired inequality with the expected stationarity measure on the left and upper-bounding terms on the right; further refinements can be obtained by relating $D_{\psi}\!\big(\pi^*(\theta^k),\pi^k\big)$ to the primal gap or by choosing $\gamma_{\theta}$ to balance the last two sums.

    Let us denote $\delta=1-\nicefrac{1}{128\kappa^2}$. Lemma \ref{lemma:distance_stoch} transforms into
    \begin{align*}
        \mathbb{E}[D_{\psi}(\pi^*(\theta^{t+1}),\pi^{t+1})] 
        &\leq \Big(1-\tfrac{1}{128\kappa^2}\Big)\,\mathbb{E}[D_{\psi}(\pi^*(\theta^t),\pi^t)] \\
        &\quad + \gamma_\theta^2\,C_{\Phi}\,\mathbb{E}\|\nabla\Phi(\theta^t)\|^2 
        + \gamma_\theta^2\,C_{B}\,\tfrac{\sigma^2}{B} 
        + \gamma_\theta^2\,\beta_1^2 C_{\beta},
    \end{align*}
    where the constants are
    \[
    C_{\Phi} = \frac{2080\,\kappa^6}{c_m^2 b_0^2}, 
    \qquad 
    C_{B} = \frac{1040\,\kappa^6}{c_m^2 b_0^2} + \frac{\lambda^2}{32L^4}, 
    \qquad 
    C_{\beta} = \frac{8320\,\kappa^6}{c_m^2 b_0^2}\Big(K^2+\tfrac{\sigma^2}{B}\Big).
    \]

    Hence, by unrolling the recursion, we obtain
    \begin{align*}
        \frac{1}{T+1}\sum_{t=0}^{T}\mathbb E\,D_{\psi}(\pi^*(\theta^t),\pi^t)
        \;\leq\;&\;\frac{1}{T+1}\cdot\frac{1}{1-\delta}\,D_{\psi}(\pi^*(\theta^0),\pi^0) \\
        &\;+\;\frac{1}{1-\delta}\left( \gamma_\theta^2\,C_{\Phi}\,\frac{1}{T+1}\sum_{t=0}^{T}\mathbb{E}\|\nabla\Phi(\theta^t)\|^2 
        + \gamma_\theta^2\,C_{B}\,\tfrac{\sigma^2}{B} 
        + \gamma_\theta^2\,\beta_1^2 C_{\beta}\right).
    \end{align*}

    Substituting the bound on the divergence into the main inequality, we obtain:
    \begin{align*}
        \frac{1}{T+1}\sum_{k=0}^{T} 
        \mathbb{E}\!\left[\|\nabla \Phi(\theta^k)\|^2\right]
        \;\le\;&
        M_1 \Bigg[
        \frac{1}{T+1}\cdot\frac{1}{1-\delta}\,D_{\psi}(\pi^*(\theta^0),\pi^0)  \\
        &\quad+\;\frac{1}{1-\delta}\Big(  \gamma_\theta^2\,C_{\Phi}\,\frac{1}{T+1}\sum_{t=0}^{T}\mathbb{E}\|\nabla\Phi(\theta^t)\|^2 
        + \gamma_\theta^2\,C_{B}\,\tfrac{\sigma^2}{B} 
        + \gamma_\theta^2\,\beta_1^2 C_{\beta}\Big)\Bigg] \\
        &+ M_2 \,\frac{\Phi(\theta^0)-\mathbb{E}\,\Phi(\theta^{T+1})}{(T+1)\gamma_\theta}
        + M_3 \,\gamma_\theta .
    \end{align*}

    Using $\gamma_\theta \;\le\; \sqrt{\tfrac{(1-\delta)}{2M_1 C_{\Phi}}}$ we obtain
    \begin{align*}
        \frac{1}{T+1}\sum_{k=0}^{T} 
            \mathbb{E}\!\left[\|\nabla \Phi(\theta^k)\|^2\right]
        \;\le\;&
        2M_1 \Bigg[
            \frac{1}{T+1}\cdot\frac{1}{1-\delta}\,D_{\psi}(\pi^*(\theta^0),\pi^0)  \\
            &\quad+\;\frac{1}{1-\delta}\Big( 
                \gamma_\theta^2\,C_{B}\,\tfrac{\sigma^2}{B} 
                + \gamma_\theta^2\,\beta_1^2 C_{\beta}\Big)\Bigg] \\
            &+ 2M_2 \,\frac{\Phi(\theta^0)-\mathbb{E}\,\Phi(\theta^{T+1})}{(T+1)\gamma_\theta}
            + 2M_3 \,\gamma_\theta .
    \end{align*}

    Then, for step size 
    \[
        \gamma_\theta \;=\; \min\{\gamma_1,\gamma_2,\gamma_3\},
    \]
    the averaged iterate satisfies
    \begin{align} \label{eq:compact-criterion}
        \E\|\nabla \Phi(\hat\theta_T)\|^2 
        \;\le\;&
        \frac{A_1}{\gamma_\theta (T+1)} \,\Delta_\Phi
        \;+\; \gamma_\theta A_2 \frac{\sigma^2}{B}
        \;+\; \frac{A_3}{T+1} D_0
        \;+\; \beta_1^2 A_4,
    \end{align}
    where the constants are
    \begin{align*}
        A_1 &= \frac{12\sqrt{2K^2+2\sigma^2/B}}{1-\beta_1}, \\[4pt]
        A_2 &= \frac{48L^2}{1-\delta}\,C_B 
        + \frac{36\kappa L\sqrt{2K^2+2\sigma^2/B}}{1-\beta_1}\,(1+C_A)C_D, \\[4pt]
        A_3 &= \frac{48L^2}{1-\delta}, \\[4pt]
        A_4 &= \frac{288\kappa L\sqrt{2K^2+2\sigma^2/B}}{1-\beta_1}\,(1+C_A)C_D\,(K^2+\tfrac{\sigma^2}{B}).
    \end{align*}
    Here
    \[
    C_A = \frac{\beta_1}{c_m b_0 (1-\beta_1)}\sqrt{2K^2+2\sigma^2/B}, 
    \qquad 
    C_D = \frac{4}{c_m^2 b_0^2},
    \]
    and
    \[
    \gamma_1 = \frac{1-\beta_1}{72\kappa L (1+C_A)C_D \sqrt{2K^2+2\sigma^2/B}}, 
    \quad 
    \gamma_2 = \frac{c_m b_0}{1048 L \kappa^4}, 
    \quad 
    \gamma_3 = \sqrt{\frac{1-\delta}{2M_1 C_\Phi}}.
    \]

    We require each term in \eqref{eq:compact-criterion} to be at most $\varepsilon^2/4$. 
    This gives
    \begin{itemize}
        \item[(i)] From the $\Delta_\Phi$-term and the $D_0$-term:
        \[
            T+1 \;\ge\;
            \max\!\Bigg\{
            \frac{4\Delta_\Phi}{\varepsilon^2}
            \max\!\Big(\tfrac{A_1}{\gamma_1}, \tfrac{A_1}{\gamma_2}, \tfrac{A_1}{\gamma_3}\Big),\;
            \frac{4A_3}{\varepsilon^2} D_0
            \Bigg\}.
        \]
    
        \item[(ii)] From the variance term:
        \[
            B \;\ge\;
            \frac{4\sigma^2}{\varepsilon^2}\;
            \min\!\Big(\gamma_1 A_2,\;\gamma_2 A_2,\;\gamma_3 A_2\Big).
        \]
    
        \item[(iii)] From the momentum term:
        \[
            \beta_1 \;\le\; \sqrt{\frac{\varepsilon^2}{4A_4}}.
        \]
    \end{itemize}

    Then substituting $\delta = 1 - \frac{1}{128 \kappa^2}$, $b_0 = L$, $c_m = \frac{1}{2}$ and with step size $\gamma_\theta = \mathcal{O}(1/\kappa^4)$ the averaged iterate satisfies
    \[
        \E\|\nabla \Phi(\hat\theta_T)\|^2 
        \;\le\;
        \frac{A_1}{\gamma_\theta (T+1)} \,\Delta_\Phi
        + \gamma_\theta A_2 \frac{\sigma^2}{B}
        + \frac{A_3}{T+1} D_0
        + \beta_1^2 A_4,
    \]
    where
    \[
    A_1 = \mathcal{O}(K+\sigma),\quad
    A_2 = \mathcal{O}(\kappa^4),\quad
    A_3 = \mathcal{O}(\kappa^2(K+\sigma)),\quad
    A_4 = \mathcal{O}(\kappa^4).
    \]
    
    \medskip
    Requiring each term in the bound to be at most $\varepsilon^2/4$ yields:
    
    \begin{itemize}
        \item[(i)] Number of iterations:
        \[
            T+1 \;\ge\;
            \max\!\Bigg\{
                \frac{\Delta_\Phi}{\varepsilon^2}\cdot \mathcal{O}\!\big(\kappa^4(K+\sigma)\big),\;
                \frac{D_0}{\varepsilon^2}\cdot \mathcal{O}\!\big(\kappa^2(K+\sigma)\big)
            \Bigg\}.
        \]
    
        \item[(ii)] Batch size:
        \[
            B \;\ge\; \frac{\sigma^2}{\varepsilon^2}\cdot \mathcal{O}(1).
        \]
    
        \item[(iii)] Momentum parameter:
        \[
            \beta_1 \;\le\; \frac{\varepsilon}{\mathcal{O}(\kappa^2)}.
        \]
    \end{itemize}
    
    This finishes the proof.
\end{proof}

\end{appendixpart}

\begin{thebibliography}{116}
\providecommand{\natexlab}[1]{#1}
\providecommand{\url}[1]{\texttt{#1}}
\expandafter\ifx\csname urlstyle\endcsname\relax
  \providecommand{\doi}[1]{doi: #1}\else
  \providecommand{\doi}{doi: \begingroup \urlstyle{rm}\Url}\fi

\bibitem[Aggarwal et~al.(2021)Aggarwal, Mittal, and Battineni]{aggarwal2021generative}
Alankrita Aggarwal, Mamta Mittal, and Gopi Battineni.
\newblock Generative adversarial network: An overview of theory and applications.
\newblock \emph{International Journal of Information Management Data Insights}, 1\penalty0 (1):\penalty0 100004, 2021.

\bibitem[Akiba et~al.(2019)Akiba, Sano, Yanase, Ohta, and Koyama]{akiba2019optuna}
Takuya Akiba, Shotaro Sano, Toshihiko Yanase, Takeru Ohta, and Masanori Koyama.
\newblock Optuna: A next-generation hyperparameter optimization framework.
\newblock In \emph{Proceedings of the 25th ACM SIGKDD international conference on knowledge discovery \& data mining}, pages 2623--2631, 2019.

\bibitem[Aljundi et~al.(2019)Aljundi, Lin, Goujaud, and Bengio]{aljundi2019gradient}
Rahaf Aljundi, Min Lin, Baptiste Goujaud, and Yoshua Bengio.
\newblock Gradient based sample selection for online continual learning.
\newblock \emph{Advances in neural information processing systems}, 32, 2019.

\bibitem[Amari(1993)]{amari1993backpropagation}
Shun-ichi Amari.
\newblock Backpropagation and stochastic gradient descent method.
\newblock \emph{Neurocomputing}, 5\penalty0 (4-5):\penalty0 185--196, 1993.

\bibitem[Arjovsky et~al.(2017)Arjovsky, Chintala, and Bottou]{arjovsky2017wasserstein}
Martin Arjovsky, Soumith Chintala, and L{\'e}on Bottou.
\newblock Wasserstein generative adversarial networks.
\newblock In \emph{International conference on machine learning}, pages 214--223. PMLR, 2017.

\bibitem[Bach et~al.(2008)Bach, Mairal, and Ponce]{bach2008convex}
Francis Bach, Julien Mairal, and Jean Ponce.
\newblock Convex sparse matrix factorizations.
\newblock \emph{arXiv preprint arXiv:0812.1869}, 2008.

\bibitem[Ben-Tal(2009)]{ben2009robust}
A~Ben-Tal.
\newblock Robust optimization.
\newblock \emph{Princeton University Press google schola}, 2:\penalty0 35--53, 2009.

\bibitem[Beznosikov et~al.(2020)Beznosikov, Samokhin, and Gasnikov]{beznosikov2020distributed}
Aleksandr Beznosikov, Valentin Samokhin, and Alexander Gasnikov.
\newblock Distributed saddle-point problems: Lower bounds, near-optimal and robust algorithms.
\newblock \emph{arXiv preprint arXiv:2010.13112}, 2020.

\bibitem[Beznosikov et~al.(2022)Beznosikov, Alanov, Kovalev, Tak{\'a}{\v{c}}, and Gasnikov]{beznosikov2022scaled}
Aleksandr Beznosikov, Aibek Alanov, Dmitry Kovalev, Martin Tak{\'a}{\v{c}}, and Alexander Gasnikov.
\newblock On scaled methods for saddle point problems.
\newblock \emph{arXiv preprint arXiv:2206.08303}, 2022.

\bibitem[Beznosikov et~al.(2023)Beznosikov, Polyak, Gorbunov, Kovalev, and Gasnikov]{beznosikov2023smooth}
Aleksandr Beznosikov, Boris Polyak, Eduard Gorbunov, Dmitry Kovalev, and Alexander Gasnikov.
\newblock Smooth monotone stochastic variational inequalities and saddle point problems: A survey.
\newblock \emph{European Mathematical Society Magazine}, 2023.

\bibitem[Bi et~al.(2022)Bi, Xie, Zhang, Chen, Gu, and Tian]{bi2022pangu}
Kaifeng Bi, Lingxi Xie, Hengheng Zhang, Xin Chen, Xiaotao Gu, and Qi~Tian.
\newblock Pangu-weather: A 3d high-resolution model for fast and accurate global weather forecast.
\newblock \emph{arXiv preprint arXiv:2211.02556}, 2022.

\bibitem[Blanchet and Kang(2020)]{blanchet2020semi}
Jose Blanchet and Yang Kang.
\newblock Semi-supervised learning based on distributionally robust optimization.
\newblock \emph{Data Analysis and Applications 3: Computational, Classification, Financial, Statistical and Stochastic Methods}, 5:\penalty0 1--33, 2020.

\bibitem[Bossard et~al.(2014)Bossard, Guillaumin, and Van~Gool]{food101}
Lukas Bossard, Matthieu Guillaumin, and Luc Van~Gool.
\newblock Food-101 -- mining discriminative components with random forests.
\newblock In \emph{European Conference on Computer Vision}, 2014.

\bibitem[Browder(1966)]{browder1966existence}
Felix~E Browder.
\newblock Existence and approximation of solutions of nonlinear variational inequalities.
\newblock \emph{Proceedings of the National Academy of Sciences}, 56\penalty0 (4):\penalty0 1080--1086, 1966.

\bibitem[Bylinkin et~al.(2025)Bylinkin, Aleksandrov, Chezhegov, and Beznosikov]{bylinkin2025enhancingstabilityphysicsinformedneural}
Dmitry Bylinkin, Mikhail Aleksandrov, Savelii Chezhegov, and Aleksandr Beznosikov.
\newblock Enhancing stability of physics-informed neural network training through saddle-point reformulation, 2025.
\newblock URL \url{https://arxiv.org/abs/2507.16008}.

\bibitem[Byrd and Lipton(2019)]{byrd2019effect}
Jonathon Byrd and Zachary Lipton.
\newblock What is the effect of importance weighting in deep learning?
\newblock In \emph{International conference on machine learning}, pages 872--881. PMLR, 2019.

\bibitem[Carmon and Hausler(2022)]{carmon2022distributionally}
Yair Carmon and Danielle Hausler.
\newblock Distributionally robust optimization via ball oracle acceleration.
\newblock \emph{Advances in Neural Information Processing Systems}, 35:\penalty0 35866--35879, 2022.

\bibitem[Chambolle and Pock(2011)]{chambolle2011first}
Antonin Chambolle and Thomas Pock.
\newblock A first-order primal-dual algorithm for convex problems with applications to imaging.
\newblock \emph{Journal of mathematical imaging and vision}, 40:\penalty0 120--145, 2011.

\bibitem[Chavdarova et~al.(2019)Chavdarova, Gidel, Fleuret, and Lacoste-Julien]{chavdarova2019reducing}
Tatjana Chavdarova, Gauthier Gidel, Fran{\c{c}}ois Fleuret, and Simon Lacoste-Julien.
\newblock Reducing noise in gan training with variance reduced extragradient.
\newblock \emph{Advances in Neural Information Processing Systems}, 32, 2019.

\bibitem[Chezhegov et~al.(2024)Chezhegov, Klyukin, Semenov, Beznosikov, Gasnikov, Horv{\'a}th, Tak{\'a}{\v{c}}, and Gorbunov]{chezhegov2024gradient}
Savelii Chezhegov, Yaroslav Klyukin, Andrei Semenov, Aleksandr Beznosikov, Alexander Gasnikov, Samuel Horv{\'a}th, Martin Tak{\'a}{\v{c}}, and Eduard Gorbunov.
\newblock Gradient clipping improves adagrad when the noise is heavy-tailed.
\newblock \emph{arXiv preprint arXiv:2406.04443}, 2024.

\bibitem[Choi et~al.(2019)Choi, Shallue, Nado, Lee, Maddison, and Dahl]{choi2019empirical}
Dami Choi, Christopher~J Shallue, Zachary Nado, Jaehoon Lee, Chris~J Maddison, and George~E Dahl.
\newblock On empirical comparisons of optimizers for deep learning.
\newblock \emph{arXiv preprint arXiv:1910.05446}, 2019.

\bibitem[Cutkosky and Orabona(2019)]{cutkosky2019momentum}
Ashok Cutkosky and Francesco Orabona.
\newblock Momentum-based variance reduction in non-convex sgd.
\newblock \emph{Advances in neural information processing systems}, 32, 2019.

\bibitem[Daskalakis et~al.(2017)Daskalakis, Ilyas, Syrgkanis, and Zeng]{daskalakis2017training}
Constantinos Daskalakis, Andrew Ilyas, Vasilis Syrgkanis, and Haoyang Zeng.
\newblock Training gans with optimism.
\newblock \emph{arXiv preprint arXiv:1711.00141}, 2017.

\bibitem[Defazio and Bottou(2019)]{defazio2019ineffectiveness}
Aaron Defazio and L{\'e}on Bottou.
\newblock On the ineffectiveness of variance reduced optimization for deep learning.
\newblock \emph{Advances in Neural Information Processing Systems}, 32, 2019.

\bibitem[Defazio et~al.(2014)Defazio, Bach, and Lacoste-Julien]{defazio2014saga}
Aaron Defazio, Francis Bach, and Simon Lacoste-Julien.
\newblock Saga: A fast incremental gradient method with support for non-strongly convex composite objectives.
\newblock \emph{Advances in neural information processing systems}, 27, 2014.

\bibitem[Delage and Ye(2010)]{delage2010distributionally}
Erick Delage and Yinyu Ye.
\newblock Distributionally robust optimization under moment uncertainty with application to data-driven problems.
\newblock \emph{Operations research}, 58\penalty0 (3):\penalty0 595--612, 2010.

\bibitem[Dong et~al.(2017)Dong, Gong, and Zhu]{dong2017class}
Qi~Dong, Shaogang Gong, and Xiatian Zhu.
\newblock Class rectification hard mining for imbalanced deep learning.
\newblock In \emph{Proceedings of the IEEE international conference on computer vision}, pages 1851--1860, 2017.

\bibitem[Dou and Li(2021)]{dou2021one}
Zehao Dou and Yuanzhi Li.
\newblock On the one-sided convergence of adam-type algorithms in non-convex non-concave min-max optimization.
\newblock \emph{arXiv preprint arXiv:2109.14213}, 2021.

\bibitem[Dozat(2016)]{dozat2016incorporating}
Timothy Dozat.
\newblock Incorporating nesterov momentum into adam.
\newblock \emph{https://openreview.net/forum?id=OM0jvwB8jIp57ZJjtNEZ}, 2016.

\bibitem[Duchi et~al.(2011)Duchi, Hazan, and Singer]{duchi2011adaptive}
John Duchi, Elad Hazan, and Yoram Singer.
\newblock Adaptive subgradient methods for online learning and stochastic optimization.
\newblock \emph{Journal of machine learning research}, 12\penalty0 (7), 2011.

\bibitem[Esser et~al.(2010)Esser, Zhang, and Chan]{esser2010general}
Ernie Esser, Xiaoqun Zhang, and Tony~F Chan.
\newblock A general framework for a class of first order primal-dual algorithms for convex optimization in imaging science.
\newblock \emph{SIAM Journal on Imaging Sciences}, 3\penalty0 (4):\penalty0 1015--1046, 2010.

\bibitem[Freund and Schapire(1997)]{freund1997decision}
Yoav Freund and Robert~E Schapire.
\newblock A decision-theoretic generalization of on-line learning and an application to boosting.
\newblock \emph{Journal of computer and system sciences}, 55\penalty0 (1):\penalty0 119--139, 1997.

\bibitem[Gidel et~al.(2018)Gidel, Berard, Vignoud, Vincent, and Lacoste-Julien]{gidel2018variational}
Gauthier Gidel, Hugo Berard, Ga{\"e}tan Vignoud, Pascal Vincent, and Simon Lacoste-Julien.
\newblock A variational inequality perspective on generative adversarial networks.
\newblock \emph{arXiv preprint arXiv:1802.10551}, 2018.

\bibitem[Goodfellow(2016)]{goodfellow2016nips}
Ian Goodfellow.
\newblock Nips 2016 tutorial: Generative adversarial networks.
\newblock \emph{arXiv preprint arXiv:1701.00160}, 2016.

\bibitem[Goodfellow et~al.(2020)Goodfellow, Pouget-Abadie, Mirza, Xu, Warde-Farley, Ozair, Courville, and Bengio]{goodfellow2020generative}
Ian Goodfellow, Jean Pouget-Abadie, Mehdi Mirza, Bing Xu, David Warde-Farley, Sherjil Ozair, Aaron Courville, and Yoshua Bengio.
\newblock Generative adversarial networks.
\newblock \emph{Communications of the ACM}, 63\penalty0 (11):\penalty0 139--144, 2020.

\bibitem[Gorbunov et~al.(2022)Gorbunov, Berard, Gidel, and Loizou]{gorbunov2022stochastic}
Eduard Gorbunov, Hugo Berard, Gauthier Gidel, and Nicolas Loizou.
\newblock Stochastic extragradient: General analysis and improved rates.
\newblock In \emph{International Conference on Artificial Intelligence and Statistics}, pages 7865--7901. PMLR, 2022.

\bibitem[Gorishniy et~al.(2022)Gorishniy, Rubachev, and Babenko]{gorishniy2022embeddings}
Yury Gorishniy, Ivan Rubachev, and Artem Babenko.
\newblock On embeddings for numerical features in tabular deep learning.
\newblock \emph{Advances in Neural Information Processing Systems}, 35:\penalty0 24991--25004, 2022.

\bibitem[Gorishniy et~al.(2024{\natexlab{a}})Gorishniy, Kotelnikov, and Babenko]{gorishniy2024tabm}
Yury Gorishniy, Akim Kotelnikov, and Artem Babenko.
\newblock Tabm: Advancing tabular deep learning with parameter-efficient ensembling.
\newblock \emph{arXiv preprint arXiv:2410.24210}, 2024{\natexlab{a}}.

\bibitem[Gorishniy et~al.(2024{\natexlab{b}})Gorishniy, Rubachev, Kartashev, Shlenskii, Kotelnikov, and Babenko]{gorishniy2024tabr}
Yury Gorishniy, Ivan Rubachev, Nikolay Kartashev, Daniil Shlenskii, Akim Kotelnikov, and Artem Babenko.
\newblock Tabr: Tabular deep learning meets nearest neighbors.
\newblock In \emph{The Twelfth International Conference on Learning Representations}, 2024{\natexlab{b}}.

\bibitem[He and Garcia(2009)]{he2009learning}
Haibo He and Edwardo~A Garcia.
\newblock Learning from imbalanced data.
\newblock \emph{IEEE Transactions on knowledge and data engineering}, 21\penalty0 (9):\penalty0 1263--1284, 2009.

\bibitem[He et~al.(2016)He, Zhang, Ren, and Sun]{he2016deep}
Kaiming He, Xiangyu Zhang, Shaoqing Ren, and Jian Sun.
\newblock Deep residual learning for image recognition.
\newblock In \emph{Proceedings of the IEEE conference on computer vision and pattern recognition}, pages 770--778, 2016.

\bibitem[Hsieh et~al.(2019)Hsieh, Iutzeler, Malick, and Mertikopoulos]{hsieh2019convergence}
Yu-Guan Hsieh, Franck Iutzeler, J{\'e}r{\^o}me Malick, and Panayotis Mertikopoulos.
\newblock On the convergence of single-call stochastic extra-gradient methods.
\newblock \emph{Advances in Neural Information Processing Systems}, 32, 2019.

\bibitem[Hsieh et~al.(2020)Hsieh, Iutzeler, Malick, and Mertikopoulos]{hsieh2020explore}
Yu-Guan Hsieh, Franck Iutzeler, J{\'e}r{\^o}me Malick, and Panayotis Mertikopoulos.
\newblock Explore aggressively, update conservatively: Stochastic extragradient methods with variable stepsize scaling.
\newblock \emph{Advances in Neural Information Processing Systems}, 33:\penalty0 16223--16234, 2020.

\bibitem[Jiang et~al.(2018)Jiang, Zhou, Leung, Li, and Fei-Fei]{jiang2018mentornet}
Lu~Jiang, Zhengyuan Zhou, Thomas Leung, Li-Jia Li, and Li~Fei-Fei.
\newblock Mentornet: Learning data-driven curriculum for very deep neural networks on corrupted labels.
\newblock In \emph{International conference on machine learning}, pages 2304--2313. PMLR, 2018.

\bibitem[Jiang et~al.(2021)Jiang, Gu, Liu, Zhu, and Pan]{jiang2021optimizer}
Zixuan Jiang, Jiaqi Gu, Mingjie Liu, Keren Zhu, and David~Z Pan.
\newblock Optimizer fusion: Efficient training with better locality and parallelism.
\newblock \emph{arXiv preprint arXiv:2104.00237}, 2021.

\bibitem[Jin and Sidford(2020)]{jin2020efficiently}
Yujia Jin and Aaron Sidford.
\newblock Efficiently solving mdps with stochastic mirror descent.
\newblock In \emph{International Conference on Machine Learning}, pages 4890--4900. PMLR, 2020.

\bibitem[Joachims(2005)]{joachims2005support}
Thorsten Joachims.
\newblock A support vector method for multivariate performance measures.
\newblock In \emph{Proceedings of the 22nd international conference on Machine learning}, pages 377--384, 2005.

\bibitem[Johnson and Zhang(2013)]{johnson2013accelerating}
Rie Johnson and Tong Zhang.
\newblock Accelerating stochastic gradient descent using predictive variance reduction.
\newblock \emph{Advances in neural information processing systems}, 26, 2013.

\bibitem[Juditsky et~al.(2011)Juditsky, Nemirovski, and Tauvel]{juditsky2011solving}
Anatoli Juditsky, Arkadi Nemirovski, and Claire Tauvel.
\newblock Solving variational inequalities with stochastic mirror-prox algorithm.
\newblock \emph{Stochastic Systems}, 1\penalty0 (1):\penalty0 17--58, 2011.

\bibitem[Kahn and Marshall(1953)]{kahn1953methods}
Herman Kahn and Andy~W Marshall.
\newblock Methods of reducing sample size in monte carlo computations.
\newblock \emph{Journal of the Operations Research Society of America}, 1\penalty0 (5):\penalty0 263--278, 1953.

\bibitem[Kallus et~al.(2022)Kallus, Mao, Wang, and Zhou]{kallus2022doubly}
Nathan Kallus, Xiaojie Mao, Kaiwen Wang, and Zhengyuan Zhou.
\newblock Doubly robust distributionally robust off-policy evaluation and learning.
\newblock In \emph{International Conference on Machine Learning}, pages 10598--10632. PMLR, 2022.

\bibitem[Kendall and Gal(2017)]{kendall2017uncertainties}
Alex Kendall and Yarin Gal.
\newblock What uncertainties do we need in bayesian deep learning for computer vision?
\newblock \emph{Advances in neural information processing systems}, 30, 2017.

\bibitem[Kim et~al.(2020)Kim, Shin, Yu, Lee, and Lee]{kim2020multiple}
Jongwon Kim, Sungho Shin, Yeonguk Yu, Junseok Lee, and Kyoobin Lee.
\newblock Multiple classification with split learning.
\newblock In \emph{The 9th International Conference on Smart Media and Applications}, pages 358--363, 2020.

\bibitem[Kingma(2014)]{kingma2014adam}
Diederik~P Kingma.
\newblock Adam: A method for stochastic optimization.
\newblock \emph{arXiv preprint arXiv:1412.6980}, 2014.

\bibitem[Korpelevich(1976)]{korpelevich1976extragradient}
Galina~M Korpelevich.
\newblock The extragradient method for finding saddle points and other problems.
\newblock \emph{Matecon}, 12:\penalty0 747--756, 1976.

\bibitem[Krizhevsky et~al.(2009)Krizhevsky, Hinton, et~al.]{krizhevsky2009learning}
Alex Krizhevsky, Geoffrey Hinton, et~al.
\newblock Learning multiple layers of features from tiny images.
\newblock Technical report, University of Toronto, 2009.

\bibitem[Kuhn and Tucker(1951)]{kuhn1951proceedings}
HW~Kuhn and AW~Tucker.
\newblock Proceedings of 2nd berkeley symposium.
\newblock In \emph{Proceedings of 2nd Berkeley Symposium}, pages 481--492, 1951.

\bibitem[LeCun et~al.(1998)LeCun, Bottou, Bengio, and Haffner]{lecun1998gradient}
Yann LeCun, L{\'e}on Bottou, Yoshua Bengio, and Patrick Haffner.
\newblock Gradient-based learning applied to document recognition.
\newblock \emph{Proceedings of the IEEE}, 86\penalty0 (11):\penalty0 2278--2324, 1998.

\bibitem[LeCun et~al.(2010)LeCun, Cortes, and Burges]{lecun2010mnist}
Yann LeCun, Corinna Cortes, and CJ~Burges.
\newblock Mnist handwritten digit database.
\newblock \emph{ATT Labs [Online]. Available: http://yann.lecun.com/exdb/mnist}, 2, 2010.

\bibitem[Levy et~al.(2020)Levy, Carmon, Duchi, and Sidford]{levy2020large}
Daniel Levy, Yair Carmon, John~C Duchi, and Aaron Sidford.
\newblock Large-scale methods for distributionally robust optimization.
\newblock \emph{Advances in Neural Information Processing Systems}, 33:\penalty0 8847--8860, 2020.

\bibitem[Li et~al.(2021)Li, Tian, Zhang, and Jadbabaie]{li2021complexity}
Haochuan Li, Yi~Tian, Jingzhao Zhang, and Ali Jadbabaie.
\newblock Complexity lower bounds for nonconvex-strongly-concave min-max optimization.
\newblock \emph{Advances in Neural Information Processing Systems}, 34:\penalty0 1792--1804, 2021.

\bibitem[Liang and Stokes(2019)]{liang2019interaction}
Tengyuan Liang and James Stokes.
\newblock Interaction matters: A note on non-asymptotic local convergence of generative adversarial networks.
\newblock In \emph{The 22nd International Conference on Artificial Intelligence and Statistics}, pages 907--915. PMLR, 2019.

\bibitem[Lin et~al.(2022)Lin, Fang, and Gao]{lin2022distributionally}
Fengming Lin, Xiaolei Fang, and Zheming Gao.
\newblock Distributionally robust optimization: A review on theory and applications.
\newblock \emph{Numerical Algebra, Control and Optimization}, 12\penalty0 (1):\penalty0 159--212, 2022.

\bibitem[Lin(2017)]{lin2017focal}
T~Lin.
\newblock Focal loss for dense object detection.
\newblock \emph{arXiv preprint arXiv:1708.02002}, 2017.

\bibitem[Lin et~al.(2020)Lin, Jin, and Jordan]{lin2019descentascent}
Tianyi Lin, Chi Jin, and Michael Jordan.
\newblock On gradient descent ascent for nonconvex-concave minimax problems.
\newblock In \emph{International conference on machine learning}, pages 6083--6093. PMLR, 2020.

\bibitem[Lin et~al.(2017)Lin, Han, Mao, Wang, and Dally]{lin2017deep}
Yujun Lin, Song Han, Huizi Mao, Yu~Wang, and William~J Dally.
\newblock Deep gradient compression: Reducing the communication bandwidth for distributed training.
\newblock \emph{arXiv preprint arXiv:1712.01887}, 2017.

\bibitem[Liu et~al.(2021)Liu, Haghgoo, Chen, Raghunathan, Koh, Sagawa, Liang, and Finn]{liu2021just}
Evan~Z Liu, Behzad Haghgoo, Annie~S Chen, Aditi Raghunathan, Pang~Wei Koh, Shiori Sagawa, Percy Liang, and Chelsea Finn.
\newblock Just train twice: Improving group robustness without training group information.
\newblock In \emph{International Conference on Machine Learning}, pages 6781--6792. PMLR, 2021.

\bibitem[Liu et~al.(2019)Liu, Mroueh, Ross, Zhang, Cui, Das, and Yang]{liu2019towards}
Mingrui Liu, Youssef Mroueh, Jerret Ross, Wei Zhang, Xiaodong Cui, Payel Das, and Tianbao Yang.
\newblock Towards better understanding of adaptive gradient algorithms in generative adversarial nets.
\newblock \emph{arXiv preprint arXiv:1912.11940}, 2019.

\bibitem[Liu et~al.(2022)Liu, Bai, Blanchet, Dong, Xu, Zhou, and Zhou]{liu2022distributionally}
Zijian Liu, Qinxun Bai, Jose Blanchet, Perry Dong, Wei Xu, Zhengqing Zhou, and Zhengyuan Zhou.
\newblock Distributionally robust $ q $-learning.
\newblock In \emph{International Conference on Machine Learning}, pages 13623--13643. PMLR, 2022.

\bibitem[Loshchilov(2017)]{loshchilov2017decoupled}
I~Loshchilov.
\newblock Decoupled weight decay regularization.
\newblock \emph{arXiv preprint arXiv:1711.05101}, 2017.

\bibitem[Loshchilov and Hutter(2015)]{loshchilov2015online}
Ilya Loshchilov and Frank Hutter.
\newblock Online batch selection for faster training of neural networks.
\newblock \emph{arXiv preprint arXiv:1511.06343}, 2015.

\bibitem[Lotidis et~al.(2023)Lotidis, Bambos, Blanchet, and Li]{lotidis2023wasserstein}
Kyriakos Lotidis, Nicholas Bambos, Jose Blanchet, and Jiajin Li.
\newblock Wasserstein distributionally robust linear-quadratic estimation under martingale constraints.
\newblock In \emph{International Conference on Artificial Intelligence and Statistics}, pages 8629--8644. PMLR, 2023.

\bibitem[Luenberger et~al.(1984)Luenberger, Ye, et~al.]{luenberger1984linear}
David~G Luenberger, Yinyu Ye, et~al.
\newblock \emph{Linear and nonlinear programming}, volume~2.
\newblock Springer, 1984.

\bibitem[Madry et~al.(2017)Madry, Makelov, Schmidt, Tsipras, and Vladu]{madry2017towards}
Aleksander Madry, Aleksandar Makelov, Ludwig Schmidt, Dimitris Tsipras, and Adrian Vladu.
\newblock Towards deep learning models resistant to adversarial attacks.
\newblock \emph{arXiv preprint arXiv:1706.06083}, 2017.

\bibitem[Marcel and Rodriguez(2010)]{marcel2010torchvision}
S{\'e}bastien Marcel and Yann Rodriguez.
\newblock Torchvision the machine-vision package of torch.
\newblock In \emph{Proceedings of the 18th ACM international conference on Multimedia}, pages 1485--1488, 2010.

\bibitem[Mehta et~al.(2023)Mehta, Roulet, Pillutla, Liu, and Harchaoui]{mehta2023stochastic}
Ronak Mehta, Vincent Roulet, Krishna Pillutla, Lang Liu, and Zaid Harchaoui.
\newblock Stochastic optimization for spectral risk measures.
\newblock In \emph{International Conference on Artificial Intelligence and Statistics}, pages 10112--10159. PMLR, 2023.

\bibitem[Mehta et~al.(2024)Mehta, Diakonikolas, and Harchaoui]{mehta2024drago}
Ronak Mehta, Jelena Diakonikolas, and Zaid Harchaoui.
\newblock Drago: Primal-dual coupled variance reduction for faster distributionally robust optimization.
\newblock In \emph{The Thirty-eighth Annual Conference on Neural Information Processing Systems}, 2024.

\bibitem[Mertikopoulos et~al.(2018)Mertikopoulos, Lecouat, Zenati, Foo, Chandrasekhar, and Piliouras]{mertikopoulos2018optimistic}
Panayotis Mertikopoulos, Bruno Lecouat, Houssam Zenati, Chuan-Sheng Foo, Vijay Chandrasekhar, and Georgios Piliouras.
\newblock Optimistic mirror descent in saddle-point problems: Going the extra (gradient) mile.
\newblock \emph{arXiv preprint arXiv:1807.02629}, 2018.

\bibitem[Mohri et~al.(2019)Mohri, Sivek, and Suresh]{mohri2019agnostic}
Mehryar Mohri, Gary Sivek, and Ananda~Theertha Suresh.
\newblock Agnostic federated learning.
\newblock In \emph{International conference on machine learning}, pages 4615--4625. PMLR, 2019.

\bibitem[Mokhtari et~al.(2020)Mokhtari, Ozdaglar, and Pattathil]{mokhtari2020unified}
Aryan Mokhtari, Asuman Ozdaglar, and Sarath Pattathil.
\newblock A unified analysis of extra-gradient and optimistic gradient methods for saddle point problems: Proximal point approach.
\newblock In \emph{International Conference on Artificial Intelligence and Statistics}, pages 1497--1507. PMLR, 2020.

\bibitem[Musa et~al.(2021)Musa, Vishi, and Rexha]{Musa_2021}
Arbena Musa, Kamer Vishi, and Blerim Rexha.
\newblock Attack analysis of face recognition authentication systems using fast gradient sign method.
\newblock \emph{Applied Artificial Intelligence}, 35\penalty0 (15):\penalty0 1346–1360, September 2021.
\newblock ISSN 1087-6545.
\newblock \doi{10.1080/08839514.2021.1978149}.
\newblock URL \url{http://dx.doi.org/10.1080/08839514.2021.1978149}.

\bibitem[Needell et~al.(2015)Needell, Zhao, and Zouzias]{needell2015randomized}
Deanna Needell, Ran Zhao, and Anastasios Zouzias.
\newblock Randomized block kaczmarz method with projection for solving least squares.
\newblock \emph{Linear Algebra and its Applications}, 484:\penalty0 322--343, 2015.

\bibitem[Nemirovski(2004)]{nemirovski2004prox}
Arkadi Nemirovski.
\newblock Prox-method with rate of convergence o (1/t) for variational inequalities with lipschitz continuous monotone operators and smooth convex-concave saddle point problems.
\newblock \emph{SIAM Journal on Optimization}, 15\penalty0 (1):\penalty0 229--251, 2004.

\bibitem[Nesterov(2005)]{nesterov2005smooth}
Yu~Nesterov.
\newblock Smooth minimization of non-smooth functions.
\newblock \emph{Mathematical programming}, 103:\penalty0 127--152, 2005.

\bibitem[Nesterov(2013)]{nesterov2013introductory}
Yurii Nesterov.
\newblock \emph{Introductory lectures on convex optimization: A basic course}, volume~87.
\newblock Springer Science \& Business Media, 2013.

\bibitem[Nilsback and Zisserman(2008)]{nilsback2008automated}
Maria-Elena Nilsback and Andrew Zisserman.
\newblock Automated flower classification over a large number of classes.
\newblock In \emph{2008 Sixth Indian conference on computer vision, graphics \& image processing}, pages 722--729. IEEE, 2008.

\bibitem[Omidshafiei et~al.(2017)Omidshafiei, Pazis, Amato, How, and Vian]{omidshafiei2017deep}
Shayegan Omidshafiei, Jason Pazis, Christopher Amato, Jonathan~P How, and John Vian.
\newblock Deep decentralized multi-task multi-agent reinforcement learning under partial observability.
\newblock In \emph{International Conference on Machine Learning}, pages 2681--2690. PMLR, 2017.

\bibitem[Pang et~al.(2019)Pang, Chen, Shi, Feng, Ouyang, and Lin]{pang2019libra}
Jiangmiao Pang, Kai Chen, Jianping Shi, Huajun Feng, Wanli Ouyang, and Dahua Lin.
\newblock Libra r-cnn: Towards balanced learning for object detection.
\newblock In \emph{Proceedings of the IEEE/CVF conference on computer vision and pattern recognition}, pages 821--830, 2019.

\bibitem[Pedregosa et~al.(2011)Pedregosa, Varoquaux, Gramfort, Michel, Thirion, Grisel, Blondel, Prettenhofer, Weiss, Dubourg, et~al.]{pedregosa2011scikit}
Fabian Pedregosa, Ga{\"e}l Varoquaux, Alexandre Gramfort, Vincent Michel, Bertrand Thirion, Olivier Grisel, Mathieu Blondel, Peter Prettenhofer, Ron Weiss, Vincent Dubourg, et~al.
\newblock Scikit-learn: Machine learning in python.
\newblock \emph{the Journal of machine Learning research}, 12:\penalty0 2825--2830, 2011.

\bibitem[Peng et~al.(2020)Peng, Dai, Zhang, and Cheng]{peng2020training}
Wei Peng, Yu-Hong Dai, Hui Zhang, and Lizhi Cheng.
\newblock Training gans with centripetal acceleration.
\newblock \emph{Optimization Methods and Software}, 35\penalty0 (5):\penalty0 955--973, 2020.

\bibitem[Polak(1971)]{polak1971computational}
Elijah Polak.
\newblock \emph{Computational methods in optimization: a unified approach}, volume~77.
\newblock Academic press, 1971.

\bibitem[Popov(1980)]{popov1980modification}
Leonid~Denisovich Popov.
\newblock A modification of the arrow-hurwitz method of search for saddle points.
\newblock \emph{Mat. Zametki}, 28\penalty0 (5):\penalty0 777--784, 1980.

\bibitem[Qi et~al.(2021)Qi, Guo, Xu, Jin, and Yang]{qi2021online}
Qi~Qi, Zhishuai Guo, Yi~Xu, Rong Jin, and Tianbao Yang.
\newblock An online method for a class of distributionally robust optimization with non-convex objectives.
\newblock \emph{Advances in Neural Information Processing Systems}, 34:\penalty0 10067--10080, 2021.

\bibitem[Reddi et~al.(2019)Reddi, Kale, and Kumar]{reddi2019convergence}
Sashank~J Reddi, Satyen Kale, and Sanjiv Kumar.
\newblock On the convergence of adam and beyond.
\newblock \emph{arXiv preprint arXiv:1904.09237}, 2019.

\bibitem[Reich(1983)]{reich1983some}
Simeon Reich.
\newblock Some problems and results in fixed point theory.
\newblock \emph{Contemp. Math}, 21:\penalty0 179--187, 1983.

\bibitem[Ren et~al.(2018)Ren, Zeng, Yang, and Urtasun]{ren2018learning}
Mengye Ren, Wenyuan Zeng, Bin Yang, and Raquel Urtasun.
\newblock Learning to reweight examples for robust deep learning.
\newblock In \emph{International conference on machine learning}, pages 4334--4343. PMLR, 2018.

\bibitem[Rockafellar(1969)]{rockafellar1969convex}
R~Tyrrell Rockafellar.
\newblock Convex functions, monotone operators and variational inequalities.
\newblock In \emph{Theory and applications of monotone operators}, pages 35--65. Citeseer, 1969.

\bibitem[Rubachev et~al.(2024)Rubachev, Kartashev, Gorishniy, and Babenko]{rubachev2024tabred}
Ivan Rubachev, Nikolay Kartashev, Yury Gorishniy, and Artem Babenko.
\newblock Tabred: Analyzing pitfalls and filling the gaps in tabular deep learning benchmarks.
\newblock \emph{arXiv preprint arXiv:2406.19380}, 2024.

\bibitem[Sagawa et~al.(2019)Sagawa, Koh, Hashimoto, and Liang]{sagawa2019distributionally}
Shiori Sagawa, Pang~Wei Koh, Tatsunori~B Hashimoto, and Percy Liang.
\newblock Distributionally robust neural networks for group shifts: On the importance of regularization for worst-case generalization.
\newblock \emph{arXiv preprint arXiv:1911.08731}, 2019.

\bibitem[Sapkota et~al.(2023)Sapkota, Wang, Tao, and Yu]{sapkota2023distributionally}
Hitesh Sapkota, Dingrong Wang, Zhiqiang Tao, and Qi~Yu.
\newblock Distributionally robust ensemble of lottery tickets towards calibrated sparse network training.
\newblock \emph{Advances in Neural Information Processing Systems}, 36:\penalty0 62657--62681, 2023.

\bibitem[Sibony(1970)]{sibony1970methodes}
Mo{\"\i}se Sibony.
\newblock M{\'e}thodes it{\'e}ratives pour les {\'e}quations et in{\'e}quations aux d{\'e}riv{\'e}es partielles non lin{\'e}aires de type monotone.
\newblock \emph{Calcolo}, 7:\penalty0 65--183, 1970.

\bibitem[Solodkin et~al.(2024)Solodkin, Veprikov, and Beznosikov]{solodkin2024methods}
Vladimir Solodkin, Andrew Veprikov, and Aleksandr Beznosikov.
\newblock Methods for optimization problems with markovian stochasticity and non-euclidean geometry.
\newblock \emph{arXiv preprint arXiv:2408.01848}, 2024.

\bibitem[Stampacchia(1964)]{stampacchia1964formes}
Guido Stampacchia.
\newblock Formes bilineaires coercitives sur les ensembles convexes.
\newblock \emph{Comptes Rendus Hebdomadaires Des Seances De L Academie Des Sciences}, 258\penalty0 (18):\penalty0 4413, 1964.

\bibitem[Streeter and McMahan(2010)]{streeter2010less}
Matthew Streeter and H~Brendan McMahan.
\newblock Less regret via online conditioning.
\newblock \emph{arXiv preprint arXiv:1002.4862}, 2010.

\bibitem[Szlendak et~al.(2021)Szlendak, Tyurin, and Richtárik]{szlendak2021permcompressors}
Rafał Szlendak, Alexander Tyurin, and Peter Richtárik.
\newblock Permutation compressors for provably faster distributed nonconvex optimization, 2021.
\newblock URL \url{https://arxiv.org/abs/2110.03300}.

\bibitem[Taiwo et~al.(2021)Taiwo, Jolaoso, and Mewomo]{taiwo2021inertial}
Adeolu Taiwo, Lateef~Olakunle Jolaoso, and Oluwatosin~Temitope Mewomo.
\newblock Inertial-type algorithm for solving split common fixed point problems in banach spaces.
\newblock \emph{Journal of Scientific Computing}, 86:\penalty0 1--30, 2021.

\bibitem[Takahashi et~al.(2019)Takahashi, Matsubara, and Uehara]{takahashi2019data}
Ryo Takahashi, Takashi Matsubara, and Kuniaki Uehara.
\newblock Data augmentation using random image cropping and patching for deep cnns.
\newblock \emph{IEEE Transactions on Circuits and Systems for Video Technology}, 30\penalty0 (9):\penalty0 2917--2931, 2019.

\bibitem[Thapa et~al.(2021)Thapa, Chamikara, and Camtepe]{thapa2021advancements}
Chandra Thapa, Mahawaga Arachchige~Pathum Chamikara, and Seyit~A Camtepe.
\newblock Advancements of federated learning towards privacy preservation: from federated learning to split learning.
\newblock \emph{Federated Learning Systems: Towards Next-Generation AI}, pages 79--109, 2021.

\bibitem[Tieleman(2012)]{tieleman2012lecture}
Tijmen Tieleman.
\newblock Lecture 6.5-rmsprop: Divide the gradient by a running average of its recent magnitude.
\newblock \emph{COURSERA: Neural networks for machine learning}, 4\penalty0 (2):\penalty0 26, 2012.

\bibitem[Tseng(1995)]{tseng1995linear}
Paul Tseng.
\newblock On linear convergence of iterative methods for the variational inequality problem.
\newblock \emph{Journal of Computational and Applied Mathematics}, 60\penalty0 (1-2):\penalty0 237--252, 1995.

\bibitem[Vepakomma et~al.(2018)Vepakomma, Gupta, Swedish, and Raskar]{vepakomma2018split}
Praneeth Vepakomma, Otkrist Gupta, Tristan Swedish, and Ramesh Raskar.
\newblock Split learning for health: Distributed deep learning without sharing raw patient data.
\newblock \emph{arXiv preprint arXiv:1812.00564}, 2018.

\bibitem[Veprikov et~al.(2024)Veprikov, Bogdanov, Minashkin, and Beznosikov]{veprikov2024new}
Andrey Veprikov, Alexander Bogdanov, Vladislav Minashkin, and Aleksandr Beznosikov.
\newblock New aspects of black box conditional gradient: Variance reduction and one point feedback.
\newblock \emph{Chaos, Solitons \& Fractals}, 189:\penalty0 115654, 2024.

\bibitem[Wiesemann(2024)]{wiesemann2024distributionally}
Wolfram Wiesemann.
\newblock Distributionally robust optimization.
\newblock \emph{arXiv preprint arXiv:2411.02549}, 2024.

\bibitem[Ying(2019)]{ying2019overview}
Xue Ying.
\newblock An overview of overfitting and its solutions.
\newblock In \emph{Journal of physics: Conference series}, volume 1168, page 022022. IOP Publishing, 2019.

\bibitem[Yuan et~al.(2024)Yuan, Liu, Wu, Zhou, and Gu]{yuan2024mars}
Huizhuo Yuan, Yifeng Liu, Shuang Wu, Xun Zhou, and Quanquan Gu.
\newblock Mars: Unleashing the power of variance reduction for training large models.
\newblock \emph{arXiv preprint arXiv:2411.10438}, 2024.

\bibitem[Zhao and Zhang(2015)]{zhao2015stochastic}
Peilin Zhao and Tong Zhang.
\newblock Stochastic optimization with importance sampling for regularized loss minimization.
\newblock In \emph{international conference on machine learning}, pages 1--9. PMLR, 2015.

\end{thebibliography}
\end{document}